\documentclass[11pt]{article}
\usepackage{xcolor}
\usepackage{ulem}
\usepackage{booktabs}
\usepackage{pgfplots}

\usepackage{graphicx,wrapfig,lipsum}
\usepackage{float}    
\usepackage{verbatim} 
\usepackage{amsmath, amssymb, amsthm, mathtools}  
\usepackage{subfig}   
\usepackage[colorlinks=true,citecolor=blue,linkcolor=blue,urlcolor=blue]{hyperref}
\usepackage{bookmark}
\usepackage{fullpage}
\usepackage{enumerate}
\usepackage{paralist}
\usepackage{xspace}
\usepackage{multirow}
\usepackage{caption}
\usepackage{bbm}
\usepackage{bm}
\usepackage{url}
\usepackage{lipsum}
\usepackage{algorithm}
\usepackage{algorithmic}
\usepackage{color}
\usepackage{microtype}
\usepackage[numbers]{natbib}
\usepackage[section]{placeins}
\usepackage{lscape}
\usepackage{multirow}
\usepackage{todonotes}

\newcommand{\stoptocwriting}{%
  \addtocontents{toc}{\protect\setcounter{tocdepth}{-5}}}

\newcommand{\vertiii}[1]{{\left\vert\kern-0.25ex\left\vert\kern-0.25ex\left\vert #1
		\right\vert\kern-0.25ex\right\vert\kern-0.25ex\right\vert}}

\makeatletter

\newcommand{\vect}[1]{\ensuremath{\mathbf{#1}}}

\newcommand{\R}{\vect{R}}

\newcommand{\E}{\mathbb{E}}

\newcommand{\sign}{\mathrm{Sign}}
\DeclareMathOperator{\Cov}{Cov}

\DeclareMathOperator{\supp}{supp}

\newcommand{\err}{\epsilon}

\newcommand{\rank}{\mathrm{rank}}

\newcommand{\inner}[2]{\left\langle #1, #2 \right\rangle}

\newcommand{\norm}[1]{\left\lVert#1\right\rVert}

{\medskip}

{\hspace*{\fill}$\Box$\par\vspace{4mm}}
{\hspace*{\fill}$\Box$\par}

\newcommand{\cA}{\mathcal{A}}
\newcommand{\cB}{\mathcal{B}}
\newcommand{\cC}{\mathcal{C}}
\newcommand{\cD}{\mathcal{D}}
\newcommand{\cE}{\mathcal{E}}
\newcommand{\cF}{\mathcal{F}}

\newcommand{\cI}{\mathcal{I}}

\newcommand{\cL}{\mathcal{L}}
\newcommand{\cM}{\mathcal{M}}
\newcommand{\cN}{\mathcal{N}}
\newcommand{\cO}{\mathcal{O}}

\newcommand{\cS}{\mathcal{S}}
\newcommand{\cT}{\mathcal{T}}

\newcommand{\cV}{\mathcal{V}}
\newcommand{\cW}{\mathcal{W}}
\newcommand{\cX}{\mathcal{X}}

\newcommand{\bE}{\mathbb{E}}

\newcommand{\bP}{\mathbb{P}}

\newcommand{\bR}{\mathbb{R}}
\newcommand{\bS}{\mathbb{S}}
\newcommand{\bT}{\mathbb{T}}

\newtheorem{theorem}{Theorem}
\newtheorem{definition}{Definition}
\newtheorem{lemma}{Lemma}

\newtheorem{remark}{Remark}

\newtheorem{proposition}{Proposition}

\newtheorem{assumption}{Assumption}

\begin{document}

\title{Blessing of Nonconvexity in Deep Linear Models: Depth Flattens the Optimization Landscape Around the True Solution}
\author{Jianhao Ma$^*$ and Salar Fattahi$^+$\vspace{2mm}\\
    Department of Industrial and Operations Engineering\\
    University of Michigan, Ann Arbor\\
    $^*$\href{mailto:jianhao@umich.edu}{jianhao@umich.edu}, $^+$\href{mailto:fattahi@umich.edu}{fattahi@umich.edu}
}

\maketitle

\stoptocwriting

\begin{abstract}
    This work characterizes the effect of depth on the optimization landscape of linear regression, showing that, despite their nonconvexity, deeper models have more desirable optimization landscape.
    We consider a robust and over-parameterized setting, where a subset of measurements are grossly corrupted with noise and the true linear model is captured via an $N$-layer linear neural network. On the negative side, we show that this problem \textit{does not} have a benign landscape: given any $N\geq 1$, with constant probability, there exists a solution corresponding to the ground truth that is neither local nor global minimum. However, on the positive side, we prove that, for any $N$-layer model with $N\geq 2$, a simple sub-gradient method becomes oblivious to such ``problematic'' solutions; instead, it converges to a balanced solution that is not only close to the ground truth but also enjoys a flat local landscape, thereby eschewing the need for ``early stopping''. Lastly, we empirically verify that the desirable optimization landscape of deeper models extends to other robust learning tasks, including deep matrix recovery and deep ReLU networks with $\ell_1$-loss.
\end{abstract}

\section{Introduction}
\label{sec::introduction}

Supported by the empirical success of deep models in contemporary learning tasks, it is by now a conventional wisdom that ``deeper models generalize better''~\cite{he2016deep, novak2018sensitivity, bommasani2021opportunities}. Indeed, the flurry of recent attempts towards demystifying this phenomenon is a testament to the amount of research it has spawned: from simple linear regression to more complex and nonlinear models, it is shown that deeper models benefit from a range of desirable statistical properties, such as \textit{depth separation}~\cite{safran2021optimization, eldan2016power, telgarsky2015representation, telgarsky2016benefits}, \textit{benign overfitting}~\cite{bartlett2020benign}, and \textit{hierarchical learning}~\cite{allen2020backward}, to name a few.

Despite the great promise of deeper models---both theoretically and empirically---the effect of depth on their optimization landscape has remained elusive to this day. A recent body of work attempts to characterize the effect of depth on the loss function through the notion of \textit{benign landscape}. Roughly speaking, an optimization problem has a benign landscape if it is devoid of spurious local minima, and its true solutions---i.e., solutions corresponding to the ground truth---coincide with global minima. It has been shown that 2-layer~\cite{bhojanapalli2016global} and multi-layer~\cite{kawaguchi2016deep} linear neural networks with nearly-noiseless data have benign landscape. However, the notion of benign landscape is significantly stronger than what is needed in practice. For instance, the existence of spurious local minima may not pose any issue if an optimization algorithm can avoid them efficiently. Another line of research focuses on characterizing the solution trajectory of different local-search algorithms, showing that they enjoy an \textit{implicit bias} that steers them away from undesirable solutions~\cite{li2021implicit, vaskevicius2019implicit,chou2021more, gunasekar2018implicit, arora2019implicit, gissin2019implicit}. However, such guarantees only apply to specific trajectories of an algorithm, thereby falling short of any meaningful characterization of the optimization landscape around those trajectories.

\begin{figure*}
    \begin{centering}
        \subfloat[1-layer]{
            {\includegraphics[width=0.32\linewidth]{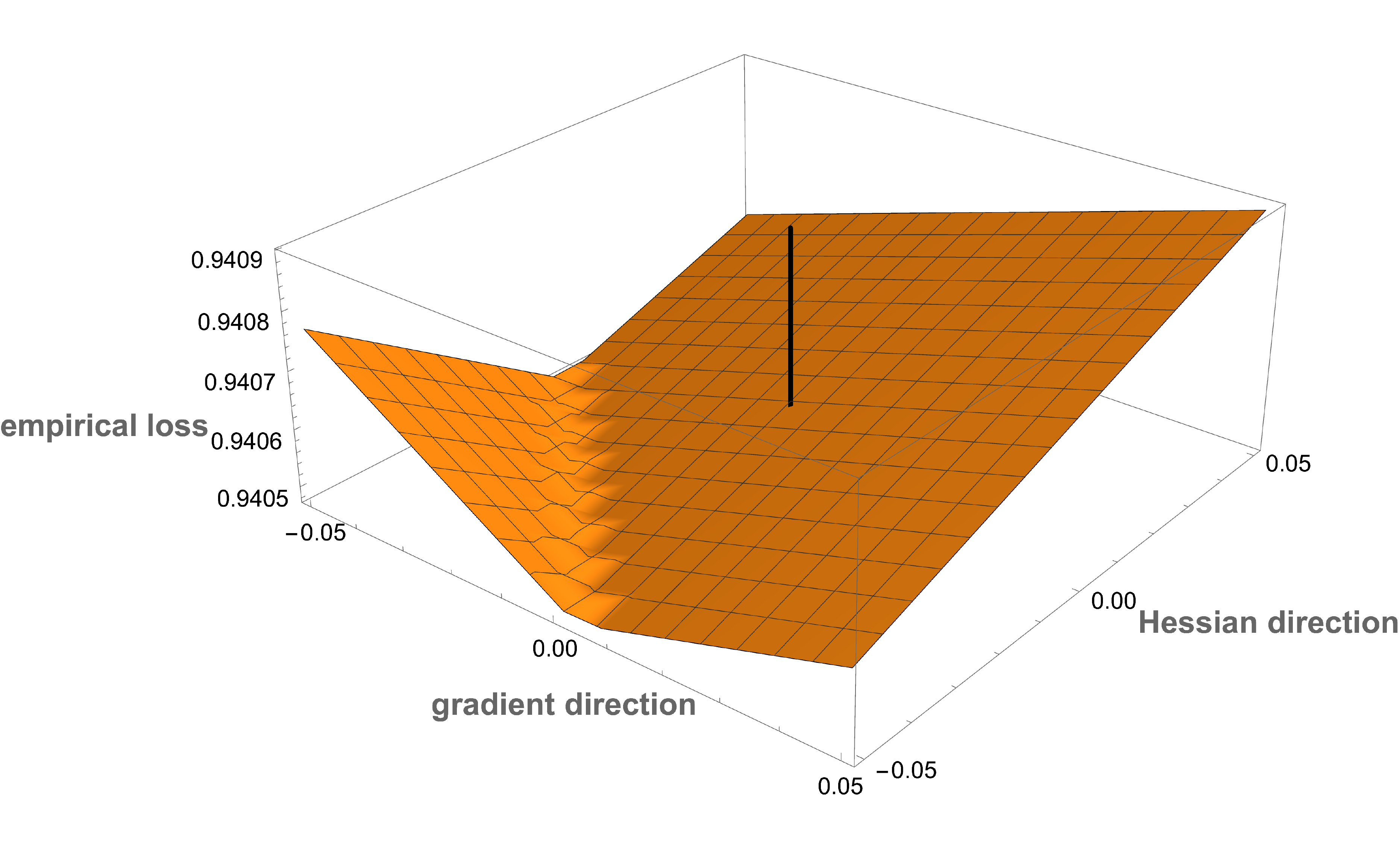}}\label{fig::1-layer-landscape}}
        \subfloat[2-layer]{
            {\includegraphics[width=0.32\linewidth]{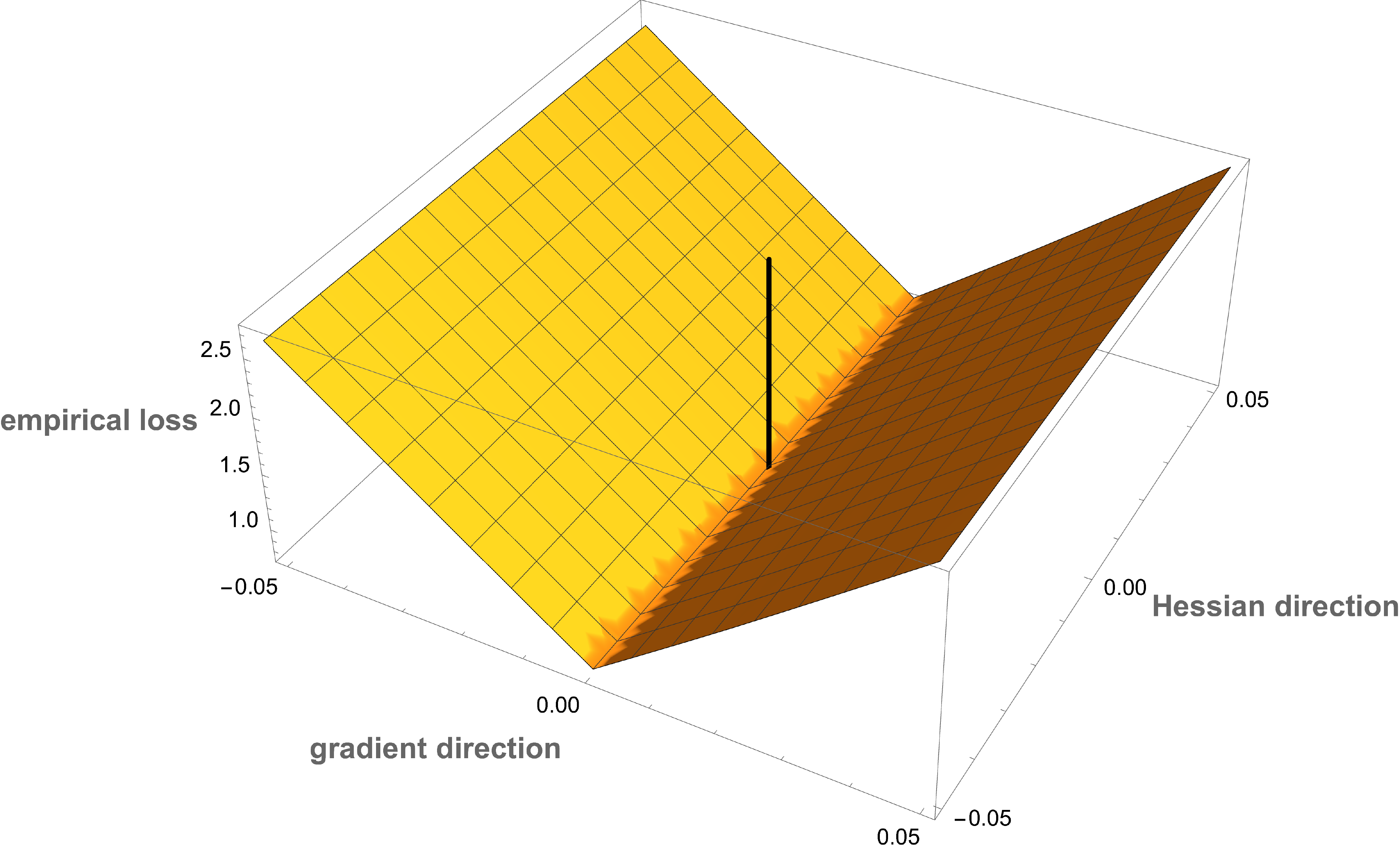}}\label{fig::2-layer-landscape}}
        \subfloat[3-layer]{
            {\includegraphics[width=0.32\linewidth]{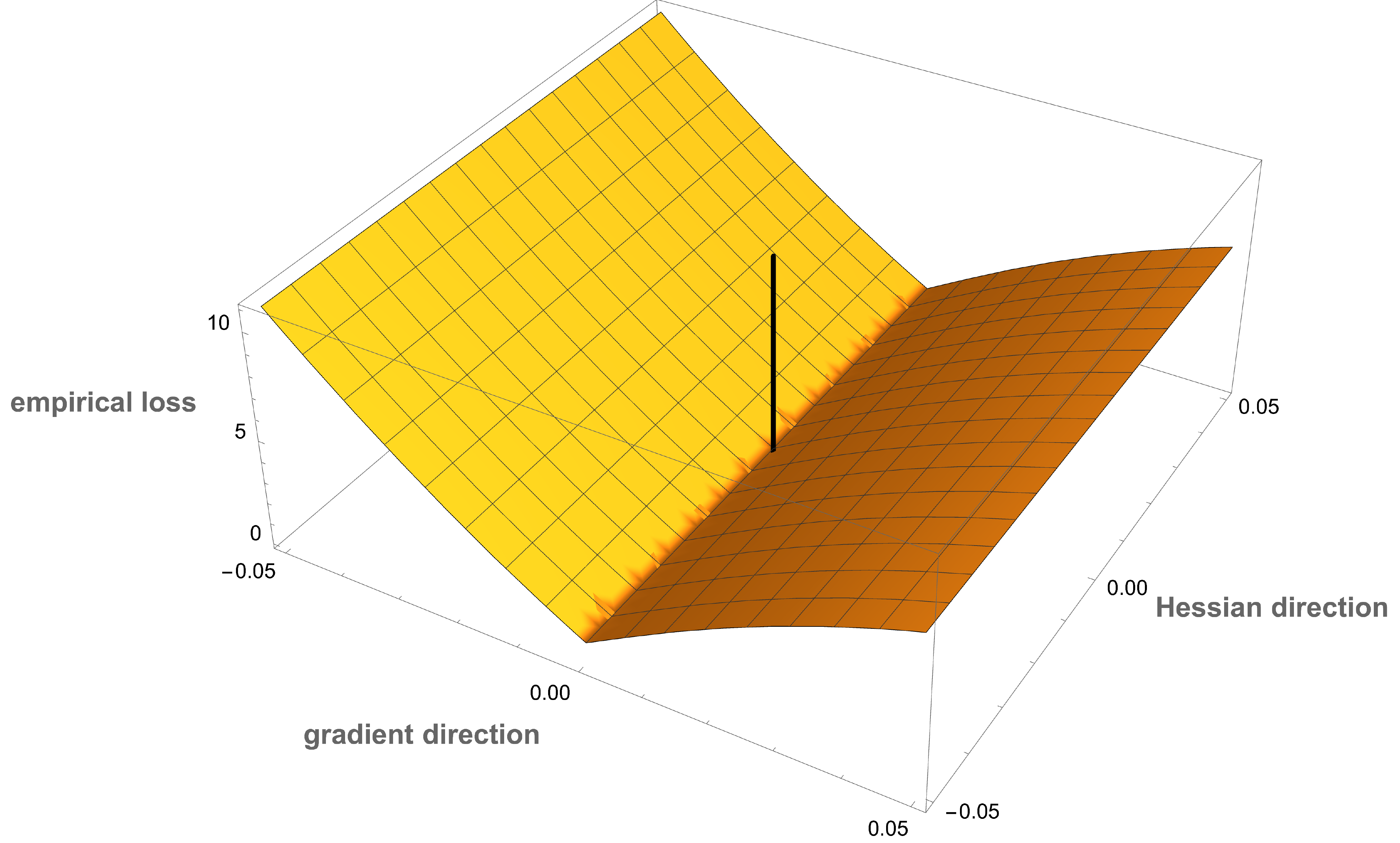}}\label{fig::3-layer-landscape}}\\
        \subfloat[1-layer]{
            {\includegraphics[width=0.32\linewidth]{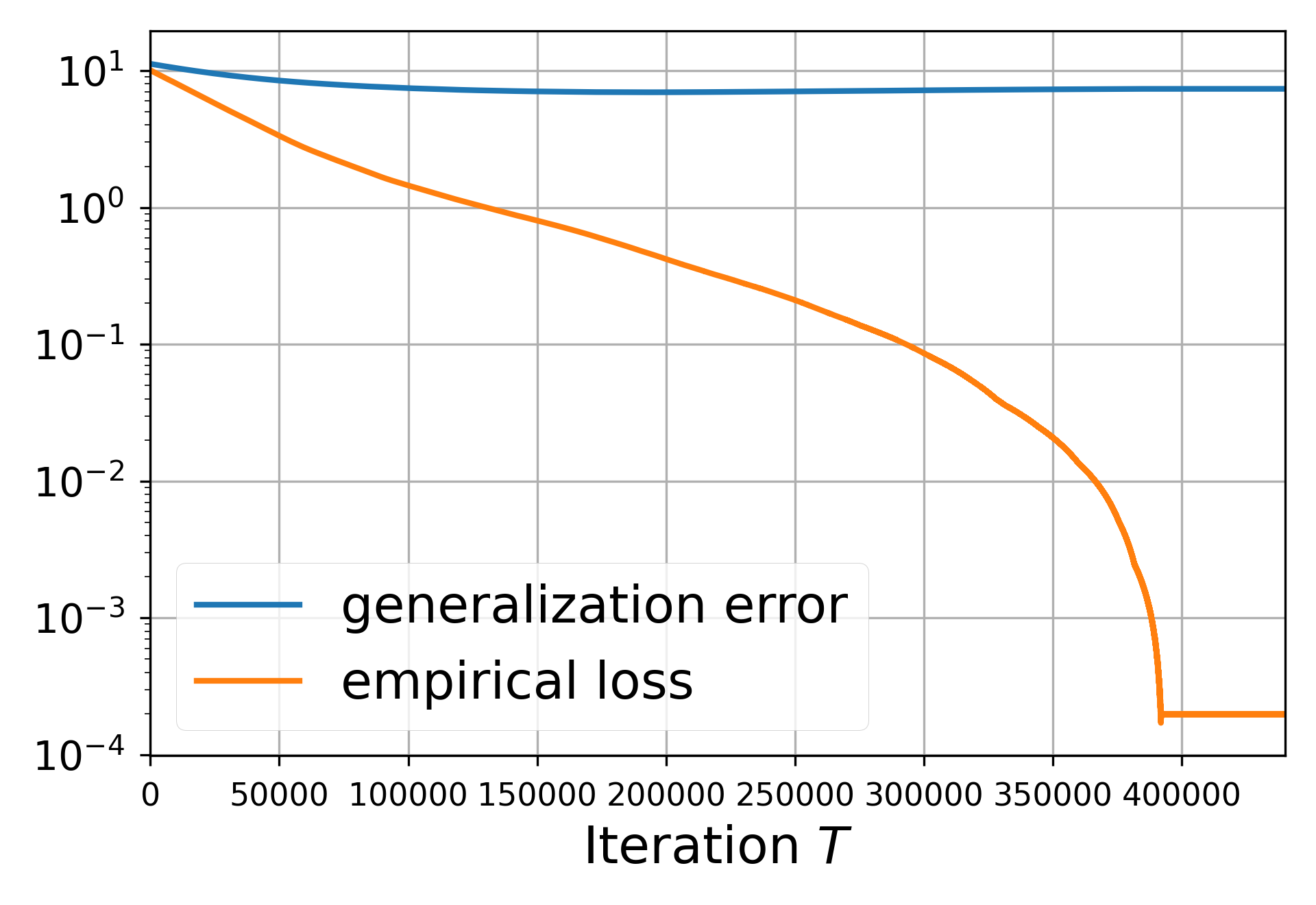}}\label{fig::1-layer-linear-regression}}
        \subfloat[2-layer]{
            {\includegraphics[width=0.32\linewidth]{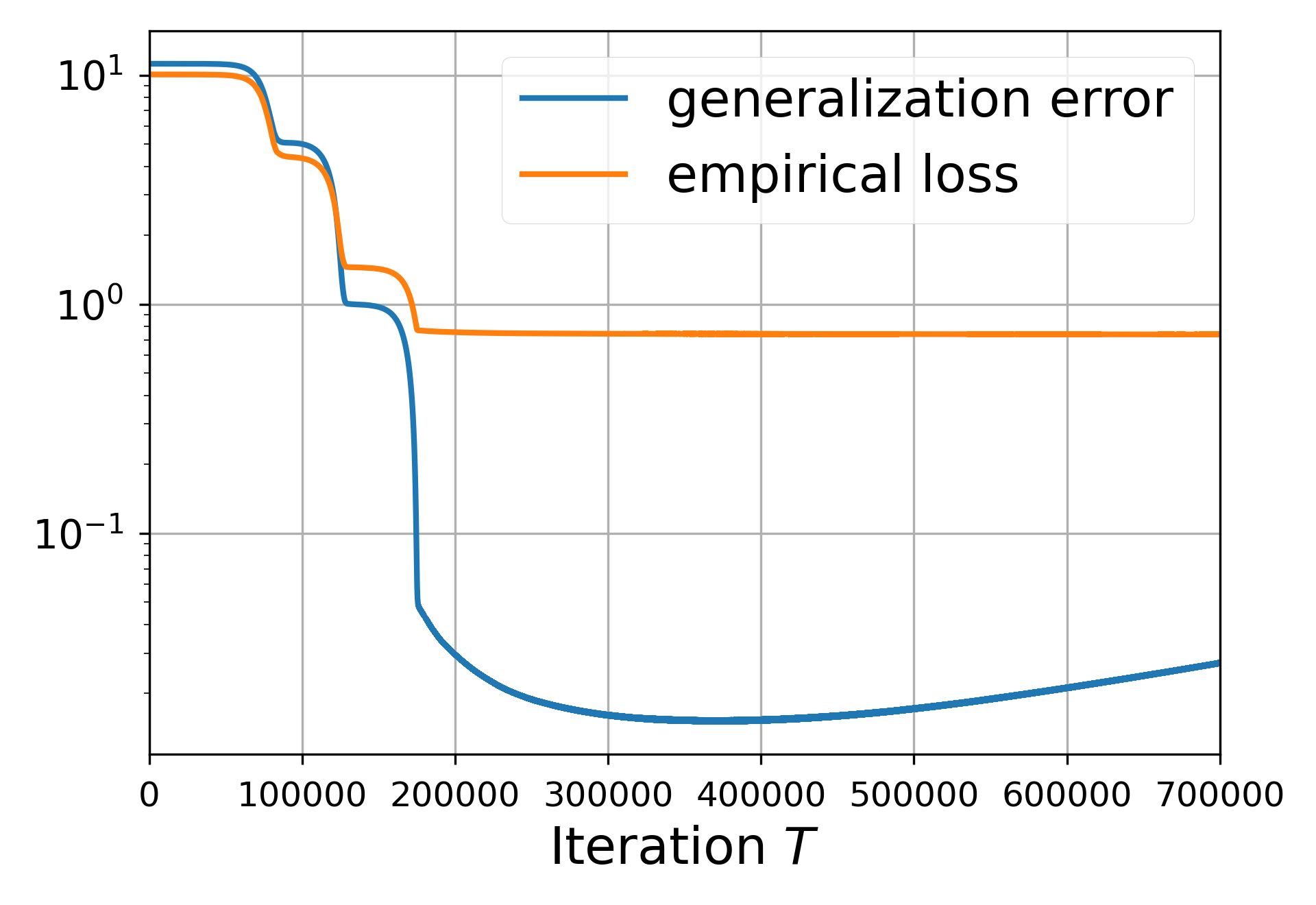}}\label{fig::2-layer-linear-regression}}
        \subfloat[3-layer]{
            {\includegraphics[width=0.32\linewidth]{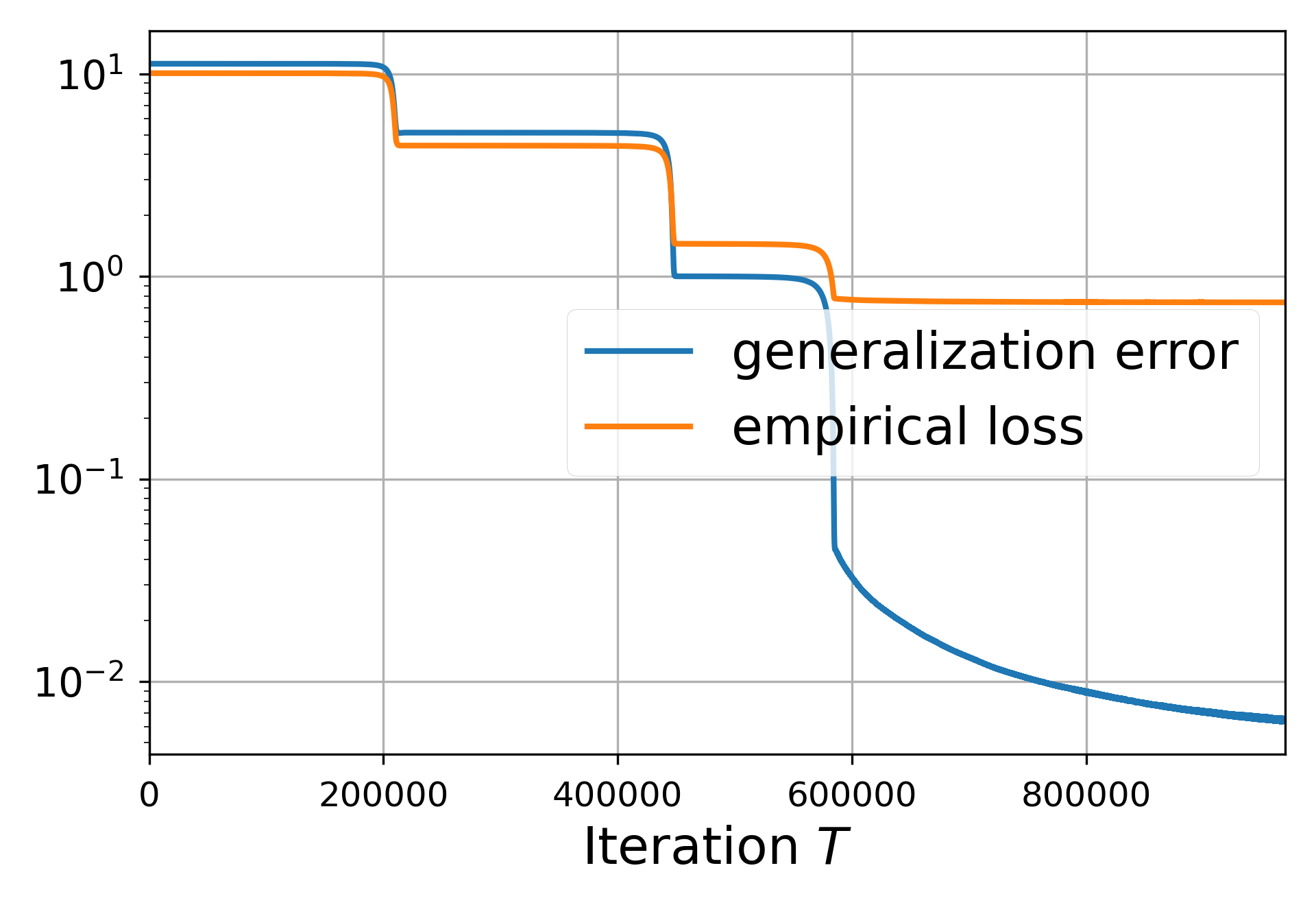}}\label{fig::3-layer-linear-regression}}
    \end{centering}
    \caption{\footnotesize
        {\bf First row.} Local landscape around the balanced true solution $\mathbf{w}^\star = (\sqrt[N]{\theta^\star}, \dots, \sqrt[N]{\theta^\star})$ for 1-, 2-, and 3-layer models. $x$ and $y$ axis correspond to the points of the form $\mathbf{w}^\star+\alpha\mathbf{d}+\beta\mathbf{h}$, for different values of $\alpha$ and $\beta$, where $\mathbf{d}$ and $\mathbf{h}$ are respectively the most descent sub-gradient direction and the most negatively curved direction of the Hessian after smoothing. {\bf Second row.}
        Generalization error and the empirical loss of the solutions found by SubGM for 1-, 2-, and 3-layer models.}
    \label{fig::simulation}
\end{figure*}

\subsection{Our Contributions}
To shed light on the effect of depth on the optimization landscape of deep models, we consider a prototypical problem in machine learning, namely \textit{robust linear regression}, where the goal is to recover a linear model from a limited number of grossly corrupted measurements. Given samples of the form $y_i = \inner{\theta^\star}{x_i}+\err_i$, we study the optimization landscape of $\ell_1$-loss under an $N$-layer model defined as $y = f_{\mathbf{w}}(x) := \inner{w_1\odot w_2\odot\dots\odot w_N}{x}$. Our results are summarized as follows:
\begin{itemize}
    \item [-] We prove that, {for any $N\geq 1$, there exists at least one true solution that is neither local nor global minimum of $\ell_1$-loss, provided that at least a fraction $p>0$ of the measurements are corrupted with noise.}
    \item [-] Despite the ubiquity of such ``hidden'' true solutions, we show that, {for any $N$-layer model with $N\geq 2$, a simple sub-gradient method (SubGM) with small initialization converges to a small neighborhood of a balanced true solution. The radius of this neighborhood shrinks with the depth of the model, resulting in more accurate solutions.} Moreover, the balancedness of the solution implies that each layer of the model inherits a similar sparsity pattern to the ground truth.
    \item [-] We prove that {deeper models take longer to train, but once trained, the algorithm will stay close to the ground truth for a longer time.} This implies that early stopping of the algorithm becomes less crucial for deeper models.
    \item [-] Finally, {we prove that depth flattens the optimization landscape around the solution obtained by SubGM.} In particular, we show that, within an $\gamma$-neighborhood of the true solution, the steepest descent direction can reduce the loss by at most $\cO\left(\gamma^N\right)$, which decreases \textit{exponentially} with $N$.
\end{itemize}

\paragraph{Motivating Example.} To showcase our results, we consider an instance of robust linear regression in the over-parameterized setting, where the dimension of $\theta^\star$ is 500 and the number of available samples is 300. Moreover, we assume that $10\%$ of the measurements are corrupted with large noise. The first row of Figure~\ref{fig::simulation} shows the landscape around the balanced ground truth ${w}_1^\star = \sqrt[N]{\theta^\star}, {w}_2^\star = \sqrt[N]{\theta^\star}, \dots, {w}_N^\star = \sqrt[N]{\theta^\star}$.\footnote{Later, we will show that a simple sub-gradient method converges to this balanced solution.} In particular, $x$ and $y$ axis show the most descent sub-gradient direction and the most negatively curved direction of the Hessian after smoothing.\footnote{To smoothen $|x|$, we replace it with $\sqrt{x^2+\epsilon}$, for $\epsilon = 10^{-7}$.} Evidently, there is a sharp transition in the landscape of $N$-layer models: for 1-layer model, the true solution has strictly negative directions along both sub-gradient and negative curvature. However, these descent directions almost disappear in 2- and 3-layer models. The second row of Figure~\ref{fig::simulation} shows the performance of SubGM on these models. It can be observed that a 1-layer model easily overfits to noise, leading to a vacuous generalization error. On the contrary, a 3-layer model can find a solution that generalizes progressively better than 1- and 2-layer models, demonstrating the algorithmic benefit of the depth.

\section{Related Works}
\label{sec::related-works}

\paragraph{Deep models:} It is known that deeper models enjoy better \textit{approximation power}.
For instance, \cite{eldan2016power, telgarsky2015representation, telgarsky2016benefits} introduce several functions that are expressible by deep models of moderate size; yet, they cannot be approximated via any shallow network of sub-exponential size. A recent work by~\cite{safran2021optimization} shows that depth separation may lead to optimization separation. In other words, functions that can be expressed by deeper models can also be efficiently learned via gradient descent. Another line of work shows that deep linear models have strong implicit bias towards the true solution~\cite{gunasekar2018implicit, arora2019implicit, chou2021more}, and that they benefit from \textit{incremental learning}~\cite{gissin2019implicit, li2020towards}. More generally,~\cite{allen2020backward} show that stochastic gradient descent on deep nonlinear models can provably learn certain complex functions by automatically decomposing them into a series of simpler ones.

\paragraph{Robust and sparse linear regression:}
Robust and sparse linear regression is a classical problem in statistics, with a wide range of applications in signal and image processing. Regularized methods, including Lasso~\cite{van2007deterministic, candes2007dantzig, bickel2009simultaneous, raskutti2011minimax}, Best Subset Selection~\cite{mazumder2017subset, bertsimas2016best}, and Forward and Backward Step-wise Regression~\cite{friedman2001elements}, are considered as most widely used methods for solving robust and sparse linear regression that come equipped with strong statistical and computational guarantees. Recently, sparse linear regression has been used to explore the implicit bias of different optimization algorithms and initialization regimes. \cite{vaskevicius2019implicit, zhao2019implicit} show that, for some unregularized overparameterized models, gradient descent (with early stopping) for sparse linear regression achieves minimax error rate. Moreover, \cite{woodworth2020kernel} study how the scale of the initial point controls the transition between the “kernel” (lazy training) and “rich” regimes, and their corresponding generalization performance. \cite{haochen2021shape} use this problem setting to explore the role of label noise in stochastic gradient descent.

\paragraph{Notations:}
For two vectors $x,y\in\R^d$, their inner product is defined as $\inner{x}{y}=x^{\top}y$, and their Hadamard product is defined as $x\odot y=[x_1y_1\ \cdots\ x_dy_d]^\top$. For simplicity of notation, we use $\prod_j w_j$ to denote the Hadamard product of $w_1,w_2,\dots,w_N\in\R^d$. For a vector $x$, $\norm{x},\norm{x}_{\infty}$, and $\norm{x}_0$ refer to $2$-norm, $\infty$-norm, and the number of nonzero elements, respectively. The symbols $a_t\lesssim b_t$ and $a_t = \mathcal{O}(b_t)$ are used to denote $a_t\leq Cb_t$, for a universal constant $C$.
The notation $a_t=\Theta(b_t)$ is used to denote $a_t=O(b_t)$ and $b_t=\Omega(a_t)$. The $\sign(\cdot)$ function is defined as $\sign(x)=x/|x|$ if $x\neq 0$, and $\sign(0)=[-1,1]$. We denote $[n]:=\{1,2,\cdots,n\}$. For a vector $x\in\mathbb{R}^d$, we define $x^a = [x_1^a\ x_2^a\ \dots\  x_d^a]^\top$, for any $a>0$. In all of our probabilistic arguments, the randomness is only over the input data and noise.

\section{Problem Formulation}
We study the problem of robust and sparse linear regression, where the goal is to estimate a $k$-sparse vector $\theta^{\star}\in \R^d$ ($k\ll d$) from a limited number of data points $\{(x_i,y_i)\}_{i=1}^{m}$, where $y_i=\inner{\theta^{\star}}{x_i}+\err_i$, $x_i$ is i.i.d. standard Gaussian vector, and $\err_i$ is noise. Moreover, for simplicity of our subsequent analysis, we assume that $\theta^\star$ is a non-negative vector. {We believe that this assumption can be relaxed without a significant change in our results.}

\begin{assumption}[Noise Model]
    \label{assumption::general-noise}
    Given a corruption probability $p$, the noise vector $\err = [\err_1\ \cdots\ \err_m]^\top\in \R^{m}$ is generated as follows: first, a subset $\cS\subset [m]$ with cardinality $pm$ is chosen uniformly at random\footnote{Here, for simplicity we assume $pm$ is an integer.}. Then, for each entry $i\in \cS$, the value of $\err_i$ is drawn independently from a distribution ${P_o}$, and all the remaining entries are set to zero. Moreover, a random variable $\zeta$ under the distribution ${P_o}$ satisfies $\E_{{P_o}}[\zeta] = 0$ and $\bP(|\zeta|\geq t_0)\geq p_0$, for some strictly positive constants $t_0$ and $p_0$.
\end{assumption}

Our considered noise model does not impose any assumption on the magnitude of the noise or the specific form of its distribution, which makes it particularly suitable for modeling outliers. Note that the assumption $\bP(|\err|\geq t_0)\geq p_0$ is very mild and satisfied for almost all common distributions. Roughly speaking, it implies that the noise takes a nonzero value with a nonzero probability.

To capture the input-output relationship, we consider a class of $N$-layer diagonal linear neural networks of the form $f_{\mathbf{w}}(x) = \inner{w_1\odot \cdots \odot w_N}{x}$, where $\mathbf{w}:=(w_1,\cdots, w_N)$ collects the weights of the layers $w_1,\cdots, w_N\in \R^d$. Due to the sparse-and-large nature of the noise, it is natural to minimize the so-called empirical risk with $\ell_1$-loss:
\begin{equation}
    \begin{aligned}
        \min_{\mathbf{w}}\cL(\mathbf{w}) & :=\frac{1}{m}\sum_{i=1}^{m}\left|f_{\mathbf{w}}(x_i)-y_i\right|=\frac{1}{m}\sum_{i=1}^{m}\left|\inner{w_1\odot \cdots \odot w_N}{x_i}-y_i\right|.
        \label{eq::empirical-loss}
    \end{aligned}
\end{equation}

Other variants of empirical risk minimization for linear regression have been studied in the literature. For instance,~\cite{chou2021more} study the solution trajectory of gradient flow on $\ell_2$-loss, showing that it converges to a solution with smallest $\ell_1$-norm. Similar analysis has also been appeared in more general deep linear neural networks~\cite{du2019width, du2018algorithmic}. However, it is well-known that $\ell_2$-loss is highly sensitive to outliers, and $\ell_1$-loss is a better alternative to identify and reject large-and-sparse noise.

A solution $\bar{\mathbf{w}}$ is called \textit{global} if it corresponds to a global minimizer of $\cL(\mathbf{w})$. Moreover, a \textit{local solution} $\bar{\mathbf{w}}$ corresponds to the minimum of $\cL(\mathbf{w})$ within an open ball centered at $\bar{\mathbf{w}}$. Finally, a \textit{true solution} $\bar{\mathbf{w}}$ satisfies $\bar w_1\odot \cdots \odot \bar w_N = \theta^\star$.

\section{Main Results}
\subsection{Absence of Benign Landscape}
\label{sec::negative}
We show that, for any arbitrary corruption probability $0<p<1/2$ and any number of layers $N\geq 1$, there exists at least one true solution with a strictly negative descent direction, provided that the problem is \textit{over-parameterized}, i.e., $m\lesssim d$.

\begin{theorem}[unidentifiable true solutions]
    \label{prop::negative}
    Define $\cW=\{\mathbf{w}:w_1\odot\cdots\odot w_N=\theta^{\star}\}$ as the set of all true solutions of an $N$-layer model. For any $N\geq 1$ and $0<p<1/2$, the following statements hold:

    - (Over-parameterized regime) If $m\leq 0.1d$, with probability of at least $1/16$, we have
    \begin{equation}
        \inf_{\mathbf{w^\star}\in\cW}\ \ \inf_{\mathbf{w}:\norm{\mathbf{w}-\mathbf{w^\star}}_{\infty}\leq \gamma}\left\{\cL(\mathbf{w})-\cL(\mathbf{w}^\star)\right\}\lesssim  -\sqrt{\frac{p_0p}{m}}d\gamma,
    \end{equation}
    for any $\gamma\lesssim t_0/\sqrt{d}\wedge 1$.

    - (Under-parameterized regime) If $m\gtrsim \frac{d}{(1-2p)^2}$, with probability of at least $1-e^{-\Omega(d)}$, we have
    \begin{equation}
        \inf_{\mathbf{w}^\star\in\cW}\ \ \inf_{\mathbf{w}}\left\{\cL(\mathbf{w})-\cL(\mathbf{w}^\star)\right\}\geq   0.
    \end{equation}
\end{theorem}

The above proposition unravels a sharp transition in the landscape of robust linear regression with an $N$-layer model: when $m\lesssim d$, some of the true solutions are likely to be non-critical points, and hence, cannot be recovered via any first-order algorithm. As soon as $m\gtrsim d$, all true solutions become global.
This is in stark contrast with the recent results on the benign landscape of robust low-rank matrix recovery with $\ell_1$-loss, which show that, under the so-called \textit{restricted isometry property} (RIP), all the true solutions are global and vice versa~\cite{li2020nonconvex, ding2021rank,fattahi2020exact}. The vector version of RIP is known to hold with $m = \tilde{\Omega}(k)$ samples (see e.g.~\cite{baraniuk2008simple} for a simple proof). Theorem~\ref{prop::negative} shows that, unlike the low-rank matrix recovery, RIP is \textit{not enough} to guarantee the equivalence between the true and global solutions in deep linear models.

Theorem~\ref{prop::negative} implies that, despite their convexity, 1-layer models are \textit{not} suitable for the robust linear regression since the set of true solutions (which is a singleton $\mathcal{W} = \{\theta^\star\}$) is  unidentifiable. However, despite the existence of unidentifiable true solutions in $N$-layer models with $N\geq 2$, we will show that a simple SubGM converges to a \textit{balanced} true solution, even if an arbitrarily large fraction of the measurements are corrupted with arbitrarily large noise values. This further sheds light on the desirable landscape of deeper models in the context of linear regression.

\begin{algorithm}
    \caption{Sub-gradient Method}
    \label{algorithm}
    \begin{algorithmic}
        \STATE {\bfseries Input:} Data points $\{(x_i,y_i)\}_{i=1}^m$, number of iterations $T$, the initial point $\mathbf{w}_0$, and the step-size $\{\eta^{(t)}\}_{t=0}^{T}$;
        \STATE {\bfseries Output:} Solution $\mathbf{w}^{(T)}$ to~\eqref{eq::empirical-loss};
        \FOR{$t\leq T$}
        \STATE Select a direction $\mathbf{d}^{(t)}$ from the sub-differential $\partial \cL(\mathbf{w}^{(t)})$ defined as:
        \begin{align}\label{eq_subdiff}
            \hspace{-7mm}\partial_{w_i} \cL(\mathbf{w})\!=\!\frac{1}{m}\sum_{j=1}^{m}\sign\left(y_j\!-\!\inner{\prod_k w_k}{x_j}\right)\!x_j\!\odot\! \prod_{k\neq i}\! w_k;
        \end{align}
        \STATE Set $\mathbf{w}^{(t+1)}\leftarrow \mathbf{w}^{(t)} - \eta^{(t)} \mathbf{d}^{(t)}$;
        \ENDFOR
    \end{algorithmic}
\end{algorithm}
\subsection{Convergence of Sub-gradient Method}
At every iteration $t$, SubGM selects a direction $\mathbf{d}^{(t)}$ from the sub-differential of the $\ell_1$-loss (defined as~\eqref{eq_subdiff}), and updates the solution as $\mathbf{w}^{(t+1)} = \mathbf{w}^{(t)} - \eta^{(t)}\mathbf{d}^{(t)}$; see Algorithm~\ref{algorithm} for details.
Our next two theorems characterizes the performance of SubGM with small initialization on $N$-layer models. We consider the cases $N=2$ and $N\geq 3$ separately, as SubGM behaves differently on these models. We define  $\kappa = \theta^\star_{\max}/\theta^\star_{\min}$ as the condition number, where $\theta^\star_{\max}$ and $\theta^\star_{\min}$ are the maximum and minimum nonzero elements of $\theta^\star$, respectively.

\begin{theorem}[2-layer model]
    \label{thm::2-layer}
    Consider the iterations of SubGM $\{\mathbf{w}^{(t)}\}_{t=0}^T$ applied to $\cL(\mathbf{w})$ with $N=2$ and step-size $\eta\lesssim 1$. Suppose that the initial point satisfies $w_1^{(0)} = w_2^{(0)} = \Theta(\sqrt{\alpha}\mathbf{1})$, where $0<\alpha\lesssim {d^2m/k}$. Moreover, suppose that $m\gtrsim \frac{k^2\kappa^2\log^2(m)\log(d)\log(\norm{\theta^{\star}}/\alpha)}{(1-p)^2}$. Then, the following statements hold with probability of $1-Ce^{-\tilde\Omega(k)}$:
    \begin{itemize}
        \item \textbf{Convergence guarantee:}
              After $\frac{1}{\eta}\log\left(\frac{1}{\alpha}\right)\lesssim \bar T\lesssim\frac{k^{3/2}}{\eta}\log\left(\frac{1}{\alpha}\right)$ iterations, we have
              \begin{equation}\nonumber
                  \hspace{-3mm}\norm{w_1^{(\bar T)}\odot w_2^{(\bar T)}-\theta^{\star}}\lesssim \eta\theta^{\star}_{\max}\vee \sqrt{d^2m}\alpha^{1-\tilde\Theta\left(\frac{k^2}{\sqrt{(1-p)^2m}}\right)}.
              \end{equation}
        \item \textbf{Balanced property:} For every $0\leq t\leq \bar T$, we have
              \begin{equation}\nonumber
                  \hspace{-3mm}\norm{w_1^{(t)}-w_2^{(t)}}_{\infty}\lesssim \alpha^{0.5-\tilde\Theta\left(\frac{k^2}{\sqrt{(1-p)^2m}}\right)}.
              \end{equation}
        \item \textbf{Long escape time:} For every $\bar T\leq t\leq \sqrt{\frac{m(1-p)^2}{k}}\bar T$, we have
              \begin{equation}\nonumber
                  \hspace{-3mm}\norm{w_1^{(t)}\!\odot\! w_2^{(t)}\!\!-\!\theta^{\star}}\!\lesssim\! \eta\theta^{\star}_{\max}\vee\sqrt{d^2m}\alpha^{0.5-\tilde\Theta\left(\frac{k^2}{\sqrt{(1-p)^2m}}\right)}.
              \end{equation}
    \end{itemize}
    Furthermore, if $m\gtrsim d\log(m)/(1-p)^2$, with probability of $1-Ce^{-\tilde\Omega(k)}$ and for every $t\geq \bar T$, we have
    \begin{equation}\nonumber
        \begin{aligned}
             & \norm{w_1^{(t)}\odot  w_2^{(t)}\!-\!\theta^{\star}}\lesssim \eta\theta^\star_{\max}\!\vee\! \sqrt{d^2m}\alpha^{1\!-\!\tilde\Theta\left(\frac{k^2}{\sqrt{m(1\!-\!p)^2}}\right)}\!\left(1\!-\!\Omega\left(\eta/\sqrt[]{d}\right)\right)^{t-\bar T}\!\!.
        \end{aligned}
    \end{equation}
\end{theorem}

We provide the main idea behind the proof of Theorem~\ref{thm::2-layer} in Section~\ref{sec:proofs}. The formal proof can be found in the appendix.
A few observations are in order based on Theorem~\ref{thm::2-layer}. First, for any $\err>0$, SubGM is guaranteed to satisfy $\norm{w_1^{(t)}\!\odot\! w_2^{(t)}\!-\!\theta^{\star}}\lesssim\err$ after $\cO(({1}/{\err})\log\left({d}/{\err}\right))$ iterations, provided that $\eta=\Theta(\err)$ and $\alpha=\err^2/(d^2m)$. Based on our numerical results (provided in the appendix), we believe that it is possible to establish a linear convergence for SubGM with a geometric step-size; a rigorous verification of this conjecture is considered as future work. Second, although SubGM converges to a vicinity of a true solution quickly, it will stay there for a significantly longer time---in particular, $\sqrt{m(1-p)^2/k}$ times longer than its initial convergence time. Such behavior is also exemplified in our simulations (see Figure~\ref{fig::2-layer-linear-regression}). After this escape time, the algorithm may slowly converge to an \textit{overfitted} solution with a better training loss. Moreover, if $m\gtrsim d$, SubGM will continuously converge to a true solution at an exponential rate, and it will never diverge. Finally, Theorem~\ref{thm::2-layer} shows that SubGM implicitly favors balanced solutions, i.e. solutions whose factors have similar magnitudes. Combined with the convergence result of SubGM, we immediately conclude that SubGM converges to a particular solution of the form $(\sqrt{\theta^\star}, \sqrt{\theta^\star})$. Therefore, the solution found by SubGM will enjoy the same (approximate) sparsity pattern as $\theta^\star$.

\begin{theorem}[$N$-layer models]
    \label{thm::N-layer}
    Consider the iterations of SubGM $\{\mathbf{w}^{(t)}\}_{t=0}^T$ applied to $\cL(\mathbf{w})$ with $N\geq 3$ and step-size $\eta\lesssim N^{-1}\kappa^{-\frac{N-2}{N}}$. Suppose that the initial point satisfies $w_j^{(0)} = \Theta({\alpha}^{1/N}\mathbf{1})$, where $0<\alpha\lesssim {d^2m/k}$. Moreover, suppose that $m\gtrsim \frac{k^2\kappa^4\log^2(m)\log(d)\log(\norm{\theta^{\star}}/\alpha)}{(1-p)^2}$. Then, the following statements hold with probability of $1-Ce^{-\tilde\Omega(k)}$:
    \begin{itemize}
        \item \textbf{Convergence guarantee:}
              After $\frac{1}{N\eta}\alpha^{-\frac{N-2}{N}}\lesssim\bar T\lesssim\frac{k^{3/2}}{N\eta}\alpha^{-\frac{N-2}{N}}$ iterations, we have
              \begin{equation}\nonumber
                  \norm{\prod w_i^{(\bar T)}\!-\!\theta^{\star}}\lesssim N\eta\theta^\star_{\max}\vee \sqrt{d^2m}\alpha\!.
              \end{equation}

        \item \textbf{Balanced property:} For every $0\leq t\leq \bar T$, we have
              \begin{equation}\nonumber
                  \begin{aligned}
                      \hspace{-3mm} & \left| w_{i,l}^{(t)}-w_{j,l}^{(t)}\right|= \cO\left(\alpha^{1/N}\right),                                  &  & \text{ for } 1\leq i<j\leq N, l:\theta^\star_l = 0,     \\
                      \hspace{-3mm} & \left| w_{i,l}^{(t)}-w_{j,l}^{(t)}\right|=\tilde \cO\left( \sqrt[N]{\theta^{\star}_l}\sqrt{k^3/m}\right), &  & \text{ for } 1\leq i<j\leq N, l:\theta^\star_l \not= 0.
                  \end{aligned}
              \end{equation}
        \item \textbf{Long escape time:} For every $\bar T\leq t\leq \sqrt{\frac{m(1-p)^2}{k}}\bar T$, we have
              \begin{equation}\nonumber
                  \norm{\prod w_i^{(t)}-\theta^{\star}}\lesssim N\eta\theta^\star_{\max}\vee \sqrt{d^2m\alpha}\!,
              \end{equation}
    \end{itemize}
    Furthermore, if $m\gtrsim d^{\frac{2N-2}{N}}\log(m)/(1-p)^2$, with probability of at least $1-Ce^{-\tilde\Omega(k)}$ and for every $t> \bar T$, we have
    \begin{equation}\nonumber
        \norm{\prod w_i^{(t)}\!-\!\theta^{\star}}\!\lesssim\! N\eta\theta^{\star}_{\max}\vee \left(\frac{\sqrt{d^2m}\alpha}{\sqrt{d^2m}\alpha N\eta d^{-\frac{N-1}{N}} (t\!-\!\bar T)\!+\!1}\right)^{\frac{N}{N-2}}.
    \end{equation}
\end{theorem}
The proof of this theorem can be found in the appendix. Theorem~\ref{thm::N-layer} sheds light on an important benefit of $N$-layer models with $N\geq 3$ compared to $2$-layer models: for sufficiently small step-size, deeper models improve the generalization error by a factor of $(1/\alpha)^{\tilde \Theta\left(k^2/\sqrt{((1-p)^2m)}\right)}$. This improvement is particularly significant when both $\alpha$ and $m$ are small. However, such improvement comes at the expense of a slower convergence rate. In particular, after setting $\eta=\Theta(\err/N)$, and $\alpha=\err/\sqrt{d^2m}$, SubGM needs $\cO\left(\left({1}/{\err}\right)^{1+\frac{N-2}{N}}\right)$ iterations to obtain an $\epsilon$-accurate solution. Evidently, the convergence rate deteriorates with $N$, ultimately approaching $\cO\left({1}/{\err}^{2}\right)$ for infinitely deep models. This can be observed in practice: Figures~\ref{fig::2-layer-linear-regression} and~\ref{fig::3-layer-linear-regression} show that 3-layer model enjoys a better generalization error compared to 2-layer model, but suffers from a slower convergence rate. This slower convergence rate also manifests itself in a more stable behavior of the algorithm: for deeper models, SubGM stays close to the ground truth for a longer time. Finally, the balanced property of the solution obtained via SubGM extends to $N$-layer models. In particular, SubGM converges to a particular solution of the form $(\sqrt[N]{\theta^\star},\dots, \sqrt[N]{\theta^\star})$, thereby inheriting the same sparsity pattern as $\theta^\star$.

\subsection{Local Landscape Around Balanced Solution}\label{sec:flatness}
In the previous section, we showed that SubGM converges to a balanced solution. In this section, we characterize the local landscape around this balanced solution, proving that it becomes flatter for deeper models.

\begin{theorem}[flatness around balanced solution]
    \label{prop::positive}
    Suppose that ${k\log(d)}/{(1-2p)^2}\lesssim m\leq 0.1d$ and $p<1/2$. Let $\mathbf{w}^\star = (\sqrt[N]{\theta^\star}, \dots, \sqrt[N]{\theta^\star})$. Then, for any $N\geq 2$ and $\gamma\leq {t_0}/{\sqrt{d}}\wedge 1$, the following statements hold:
    \begin{itemize}
        \item With probability at least $1-e^{-\Omega(k)}$, we have
              \begin{equation}\nonumber
                  \begin{aligned}
                      \hspace{-3mm}\inf_{\mathbf{w}:\norm{\mathbf{w}-\mathbf{w^\star}}_{\infty}\leq \gamma}\left\{\cL(\mathbf{w}^{\star})-\cL\left(\mathbf{w}\right)\right\}\gtrsim -\frac{d}{\sqrt[]{m}}\gamma^N.
                  \end{aligned}
              \end{equation}
        \item With probability  at least $1/16$, we have
              \begin{equation}\nonumber
                  \begin{aligned}
                      \hspace{-3mm}\inf_{\mathbf{w}:\norm{\mathbf{w}-\mathbf{w^\star}}_{\infty}\leq \gamma}\left\{\cL(\mathbf{w}^{\star})-\cL\left(\mathbf{w}\right)\right\}\lesssim -\sqrt{p_0p}\frac{d}{\sqrt[]{m}}\gamma^N.
                  \end{aligned}
              \end{equation}
    \end{itemize}
\end{theorem}

Theorem~\ref{prop::positive} shows that, within a $\gamma$-neighborhood of $\mathbf{w}^\star$, the most descent direction from $\mathbf{w}^\star$ can reduce the loss by at most $\mathcal{O}\left(d/\sqrt{m}\cdot \gamma^{N}\right)$, which decreases exponentially with $N$. Moreover, in the noisy setting, the above theorem implies that $\mathbf{w}^\star$ is likely to be neither local nor global minimum, since it has a descent direction. However, the flatness of the landscape around $\mathbf{w}^\star$ enables SubGM to stay close to the balanced solution for a long time.

\begin{remark}
    Note that the choice of $\ell_\infty$-ball for the perturbation set is to ensure that the size of the possible perturbations per layer remains independent of the depth of the model. This is indeed crucial to ensure a fair comparison between models with different depths: alternative choices of the perturbation set, such as $\ell_q$-ball with $1\leq q<\infty$ (e.g. $\ell_2$-ball) would shrink the size of the feasible per-layer perturbations with $N$, thereby leading to an unfair advantage to deeper models.
\end{remark}

\section{Proof Techniques}\label{sec:proofs}
At the crux of our proof technique for Theorems~\ref{thm::2-layer} and~\ref{thm::N-layer} lies the following decomposition of the sub-differential:
\begin{align}\nonumber
    \partial \cL(\mathbf{w}) = \underbrace{\xi\cdot\partial\bar \cL(\mathbf{w})}_{\text{expected subdiff.}} + \underbrace{\left(\partial \cL(\mathbf{w}) - \xi\cdot\partial\bar \cL(\mathbf{w})\right)}_{\text{subdiff. deviation}},\quad \text{for some strictly positive $\xi$.}
\end{align}
In the above decomposition, $\bar \cL(\mathbf{w})$ is called \textit{expected loss}, and is defined as $\bar \cL(\mathbf{w}) = \norm{w_1\odot\dots\odot w_N-\theta^\star}$. As will be shown later, $\bar \cL(\mathbf{w})$ captures the expected behavior of the empirical loss $\cL(\mathbf{w})$. To analyze the behavior of SubGM on $\cL(\mathbf{w})$, we first consider the ideal scenario, where $\cL(\mathbf{w})$ coincides with its expectation. Then, we extend our analysis to the general case by controlling the sub-differential deviation. In particular, we show that the desirable convergence properties of SubGM extends to $\cL(\mathbf{w})$, provided that its sub-differentials are ``direction-preserving'', i.e., $\mathbf{d} \approx \xi\bar{\mathbf{d}}$, for every $\mathbf{d}\in\partial \cL(\mathbf{w}),\bar{\mathbf{d}}\in\partial \bar\cL(\mathbf{w})$ and some $\xi>0$. To formalize this idea, we first provide a more concise characterization of $\partial\cL(\mathbf{w})$:
\begin{align*}
    \partial_{w_i} & \cL(\mathbf{w}) = \left\{q\odot\prod_{k\not=i}w_k: q\in \mathcal{Q}\left(\theta^\star-\prod_{k}w_k\right)\right\},\text{ where } \mathcal{Q}(z) = \frac{1}{m}\sum_{i=1}^m\sign\left(\inner{x_i}{z}+\err_i\right)x_i.
\end{align*}

\begin{definition}[approximately sparse vectors]
    We say a vector $v\in \R^d$ is \textit{$(k,\vartheta)$-approximately sparse} if there exists a vector $u$, such that $\norm{u}_0\leq k$ and $\norm{u-v}\leq \vartheta$.
\end{definition}

\begin{proposition}[direction-preserving property]\label{prop:sign-RIP}
    Suppose that $m\gtrsim \frac{k\log^2(m)\log(d)\log(R/\vartheta)}{(1-p)^2\delta^2}$ for some $R,\vartheta, \delta>0$. Then, with probability of at least $1-Ce^{-\Omega(m\delta^2)}$, the following inequality holds for any $q\in\mathcal{Q}(z)$ and any $(k,\vartheta)$-approximately sparse vector $z$ that satisfies $\sqrt{{dm}/{k}}\vartheta\log\left({1}/{\vartheta}\right)\lesssim \norm{z}\leq R$:
    \begin{equation}\label{eq_sign-RIP}
        \norm{q-\sqrt{\frac{2}{\pi}}\left(1-p+ p \mathbb{E}\left[e^{-\err^{2} /\left(2\left\|z\right\|\right)}\right]\right)\frac{z}{\norm{z}}}_{\infty}\leq \delta.
    \end{equation}
    Moreover, if $m\gtrsim \frac{d\log(m)}{(1-p)^2\delta^2}$, with probability of $1-Ce^{-\Omega(m\delta^2)}$,~\eqref{eq_sign-RIP} holds for every $z\in\mathbb{R}^d$.
\end{proposition}

Proposition~\ref{prop:sign-RIP} is analogous to \textit{Sign-RIP} condition introduced in~\cite{ma2022global, ma2021sign} for the robust low-rank matrix recovery, and is at the heart of our proofs for Theorems~\ref{thm::2-layer} and~\ref{thm::N-layer}. Suppose that $\theta^\star-\prod_{k}w_k$ is a $(k,\vartheta)$-approximately sparse and satisfies~\eqref{eq_sign-RIP}. Then, we have $\norm{\mathbf{d}-\bar{\mathbf{d}}}_{\infty}\leq \left(\max_i\left\{\prod_{k\not=i}w_k\right\}\right)\delta$, which in turn provides an upper bound on the sub-differential deviation.

\subsection{Proof Sketch of Theorem~\ref{thm::2-layer}}
To streamline the presentation, here we only provide simplified versions of our key ideas, which inevitably lead to looser guarantees.  To streamline the proof, we assume that $\theta^{\star}_1\geq \cdots\geq\theta^{\star}_k>\theta_{k+1}^{\star}=\cdots=\theta_d^{\star}=0$. Moreover, for simplicity of notation, we denote $u = w_1$ and $v = w_2$. Consider the following decomposition:
\begin{align}
    u\odot v = [\underbrace{u_1v_1 \dots u_kv_k}_{S}\ \underbrace{u_{k+1}v_{k+1}\dots u_dv_d}_{E}]^\top.
\end{align}
The vectors $S$ and $E$ are called \textit{signal} and \textit{residual terms}, respectively. Evidently, we have $u\odot v = \theta^\star$ if and only if $S = [\theta^\star_1,\dots,\theta^\star_k]^\top$ and $E = 0$. Based on this observation, our goal is to show that the signal term converges to $[\theta^\star_1,\dots,\theta^\star_k]^\top$ exponentially fast, while the error term remains small throughout the solution trajectory.

\begin{lemma}[signal dynamic; informal]
    Suppose that~\eqref{eq_sign-RIP} holds for $z = \theta^\star-u^{(t)}\odot v^{(t)}$, and $\norm{\theta^\star-u^{(t)}\odot v^{(t)}}\gtrsim \eta\norm{\theta^\star}$. Then, we have
    \begin{equation}
        \begin{aligned}
             & u^{(t+1)}_{i}\!v^{(t+1)}_{i}
            \geq \left(1\!\!+\!2\eta\left( \frac{\theta^{\star}_i-u^{(t)}_{i}v^{(t)}_{i}}{\norm{u^{(t)}\odot v^{(t)}-\theta^{\star}}}+\delta_i\right)\right)u^{(t)}_{i}v^{(t)}_{i},
            \label{eq::signal-dynamics}
        \end{aligned}
    \end{equation}
    for some $|\delta_i|\leq \delta$ and every $i = 1,\dots,k$.
\end{lemma}

\begin{lemma}[residual dynamic; informal]\label{lem:error-dynamic}
    Suppose that~\eqref{eq_sign-RIP} holds for $z = \theta^\star-u^{(t)}\odot v^{(t)}$, and $\norm{\theta^\star-u^{(t)}\odot v^{(t)}}\gtrsim \eta\norm{\theta^\star}$. Then, we have
    \begin{equation}
        \begin{aligned}
            \left(u^{(t+1)}_{i}\right)^2\!\!\!+\!\left(v^{(t+1)}_{i}\right)^2\!\!\leq\! \left(1\!+\!\cO(\eta\delta)\right)\!\left(\left(u^{(t)}_{i}\right)^2\!\!\!+\!\left(v^{(t)}_{i}\right)^2\right),
            \label{eq::error-dynamics}
        \end{aligned}
    \end{equation}
    for every $i = k+1,\dots,d$.
\end{lemma}

\begin{lemma}[difference dynamic; informal]\label{lem:diff}
    Suppose that~\eqref{eq_sign-RIP} holds for $z = \theta^\star-u^{(t)}\odot v^{(t)}$, and $\norm{\theta^\star-u^{(t)}\odot v^{(t)}}\gtrsim \eta\norm{\theta^\star}$. Then, we have
    \begin{equation}\label{eq_diff}
        \begin{aligned}
             & u^{(t+1)}_{i}-v^{(t+1)}_{i}=\left(u^{(t)}_{i}-v^{(t)}_{i}\right)\left(1-\eta\frac{\theta^{\star}_i-u^{(t)}_{i}v^{(t)}_{i}}{\norm{u^{(t)}\odot v^{(t)}-\theta^{\star}}}+\eta\delta_i\right),
        \end{aligned}
    \end{equation}
    for some $|\delta_i|\leq \delta$ and every $i = 1,\dots,d$.
\end{lemma}

\paragraph{Convergence guarantee.}
For any fixed $i = 1,\dots,k$, we show that $u^{(t)}_{i}\!v^{(t)}_{i} = \theta^\star_i \pm \cO(\delta)\norm{\theta^\star}$ after $\mathcal{O}({\norm{\theta^\star}}/({\eta\theta^\star_i})\log(1/\alpha))$ iterations. To see this, suppose that $T_i$ is the largest iteration such that $u^{(t)}_{i}\!v^{(t)}_{i}\leq \theta^\star_i$ for every $t\leq T_i$. Moreover, suppose that $\norm{u^{(t)}\odot v^{(t)}}\leq C\norm{\theta^\star}$, for sufficiently large $C$ (this is proven in the appendix). Under these assumptions,~\eqref{eq::signal-dynamics} reduces to
\begin{equation}
    u^{(t+1)}_{i}v^{(t+1)}_{i}\geq \left(1+\Omega(1)\frac{\eta\theta^{\star}_i}{\norm{\theta^{\star}}}\right)u^{(t)}_{i}v^{(t)}_{i}.
\end{equation}
which implies that $T_i \lesssim {\norm{\theta^\star}}/({\eta\theta^\star_i})\log(1/\alpha)$. For any $t>T_i$, define $y_i^{(t)} = \theta^\star_i-u^{(t)}_{i}v^{(t)}_{i}$. One can write
\begin{equation}
    \begin{aligned}
        y^{(t+1)}_i\leq\left(1-\Omega(1)\frac{\eta\theta^{\star}_i}{\norm{\theta^{\star}}}\right)y^{(t)}_i+\eta\delta \theta^{\star}_i.
    \end{aligned}
\end{equation}
Hence, with additional $\cO\left(\norm{\theta^{\star}}/(\eta\theta^{\star}_1)\right)$ iterations,
we have $u^{(t)}_{i}v^{(t)}_{i}=\theta_i^{\star}\pm \cO(\delta) \norm{\theta^{\star}}$.
On the other hand, Lemma~\ref{lem:error-dynamic} implies that, for any $i = k+1,\dots,d$ and $t\lesssim {\norm{\theta^\star}}/({\eta\theta^\star_{k}})\log(1/\alpha)$, we have
\begin{align*}
    \left(u^{(t)}_{i}\right)^2+\left(v^{(t)}_{i}\right)^2 & \lesssim \alpha\left(1\!+\!\cO(\eta\delta)\right)^{\cO\left(\frac{\norm{\theta^\star}}{\eta\theta_{k}^{\star}}\log\left(\frac{1}{\alpha}\right)\right)}\lesssim \alpha^{1-\cO\left(\sqrt{k}\kappa\delta\right)},
\end{align*}
where $\kappa = \theta^\star_1/\theta^\star_k$ is the condition number of $\theta^\star$. Combining the above dynamics, we have
\begin{align*}
    \norm{u^{(t)}\!\odot\! v^{(t)}\!-\!\theta^\star}\!\lesssim\! \eta\norm{\theta^\star}\!\vee\! \sqrt{k}\norm{\theta^\star}\delta\!\vee\! \sqrt{d}\alpha^{1-\cO\left(\sqrt{k}\kappa\delta\right)}.
\end{align*}
In the appendix, we provide a more refined analysis that relaxes the dependency of the final error on $\delta$ and $\kappa$.

\paragraph{Long escape time.} We show in the appendix that after the first stage, the residual becomes the dominant term in the final error. This together with Lemma~\ref{lem:error-dynamic} implies that, for every $t\lesssim \frac{\norm{\theta^\star}}{\eta\theta_{k}^{\star}\sqrt{\delta}}\log(1/\alpha)$, we have $\norm{E}\lesssim \sqrt{d}\alpha^{1-\sqrt{k}\kappa\sqrt{\delta}}.$

\paragraph{Balanced property.}
We have $u_i^{(t)}v_i^{(t)}\leq \theta^\star_i$ for every $i \in [k]$, and $|u_i^{(t)}v_i^{(t)}|\lesssim \alpha^{1-\cO(\sqrt{k}\kappa\delta)}$ for every $i = k+1,\dots,d$. Therefore, Lemma~\ref{lem:diff} can be invoked to verify $\left|u_i^{(t+1)}-v_i^{(t+1)}\right|\leq \left(1+\cO(\eta\delta)\right)\left|u_i^{(t)}-v_i^{(t)}\right|.$

This in turn leads to
\begin{align*}
    \left|u_i^{(t)}-v_i^{(t)}\right| & \lesssim\sqrt{\alpha}\left(1+\cO(\eta\delta)\right)^{\cO\left(\frac{\norm{\theta^\star}}{\eta\theta_{k}^{\star}}\log\left(\frac{1}{\alpha}\right)\right)}\lesssim \alpha^{0.5-\cO(\sqrt{k}\kappa\delta)}.
\end{align*}

\begin{figure*}[t]
    \begin{center}
        \subfloat[\footnotesize 2-layer matrix recovery]{
            {\includegraphics[width=0.32\linewidth]{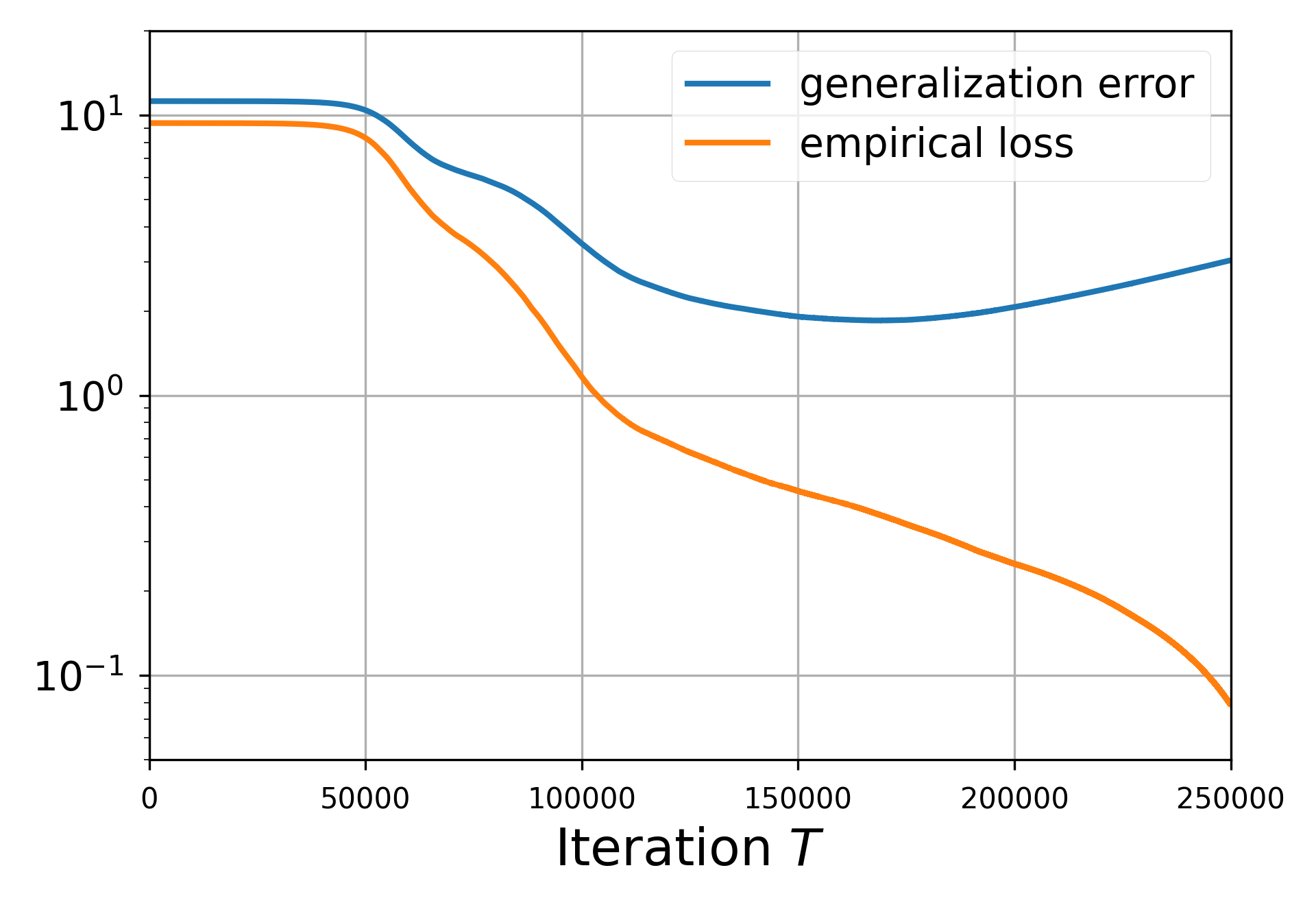}}\label{fig::2-layer-matrix-recovery}}
        \subfloat[\footnotesize 3-layer matrix recovery]{
            {\includegraphics[width=0.32\linewidth]{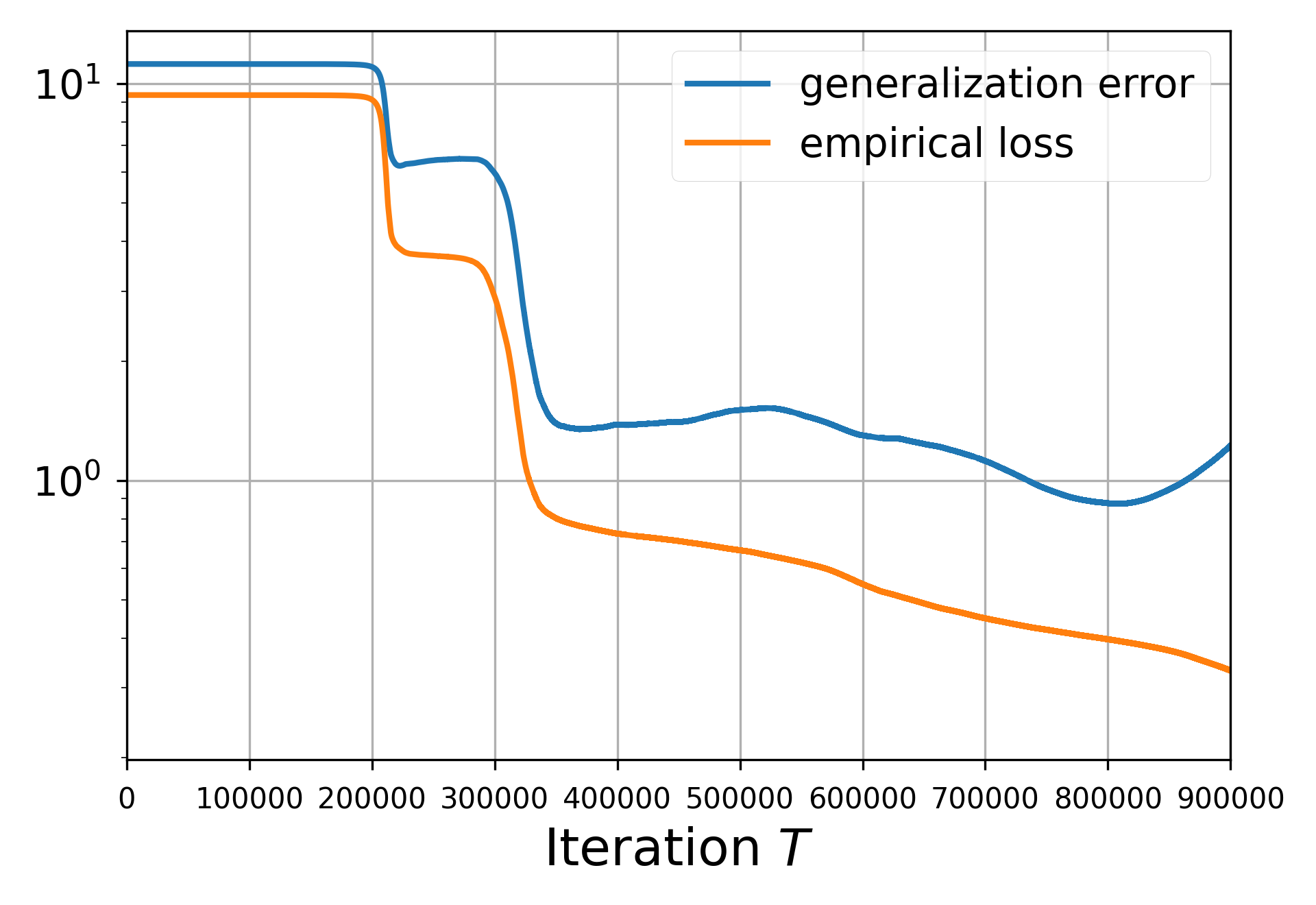}}\label{fig::3-layer-matrix-recovery}}
        \subfloat[\footnotesize 4-layer matrix recovery]{
            {\includegraphics[width=0.32\linewidth]{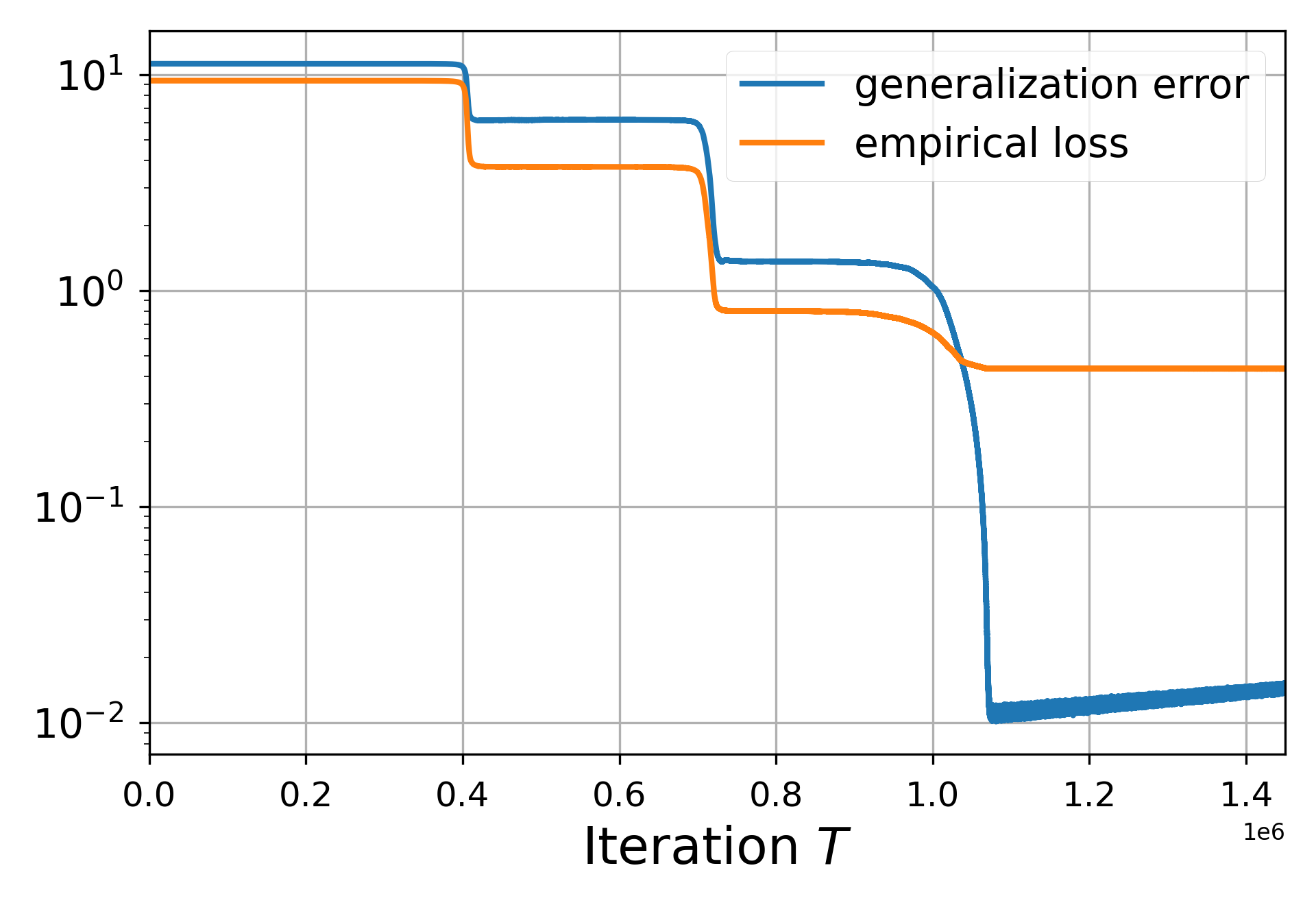}}\label{fig::4-layer-matrix-recovery}}\\\vspace{-2mm}
        \subfloat[\footnotesize 2-layer ReLU network]{
            {\includegraphics[width=0.32\linewidth]{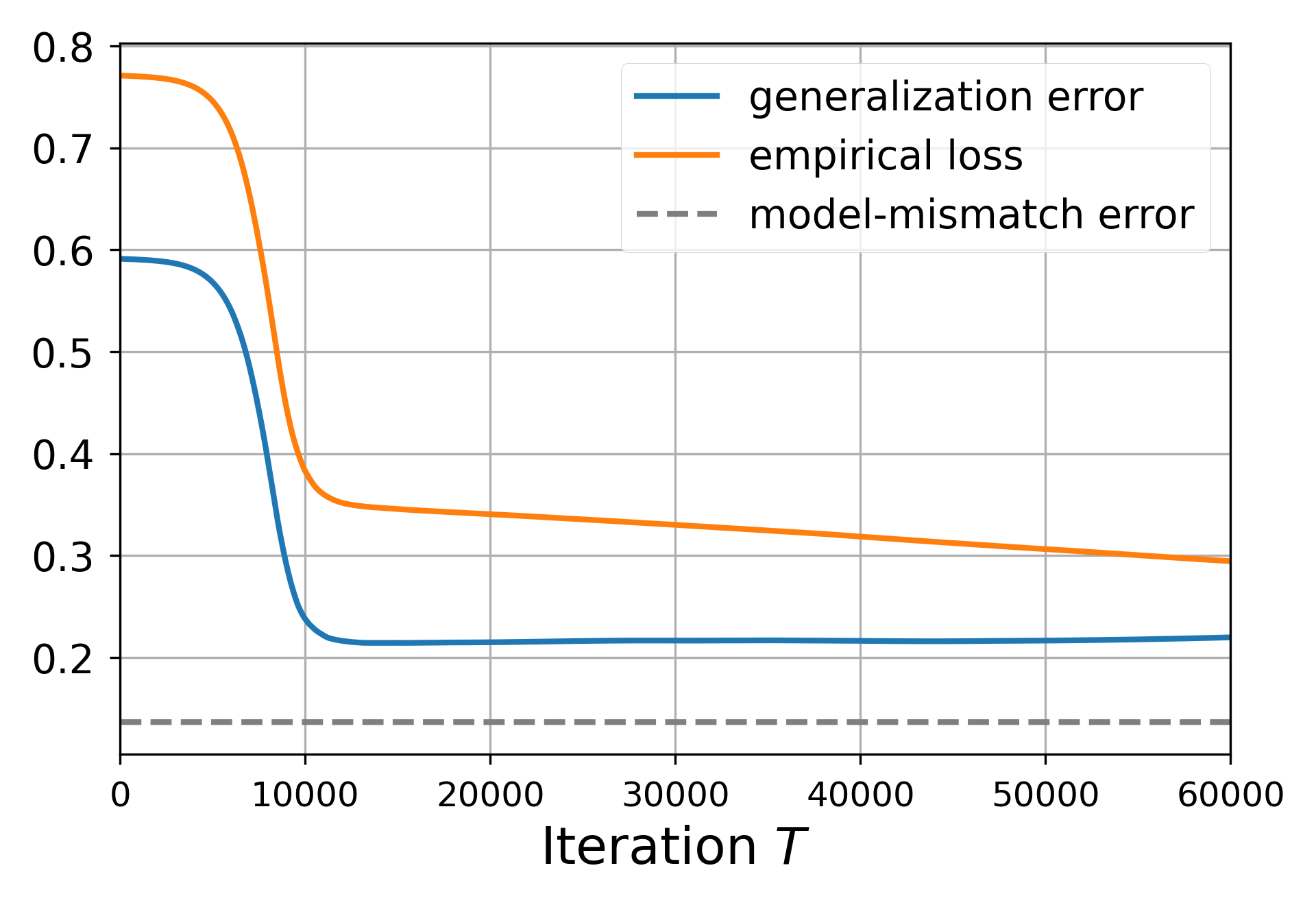}}\label{fig::2-nn}}
        \subfloat[\footnotesize 3-layer ReLU network]{
            {\includegraphics[width=0.32\linewidth]{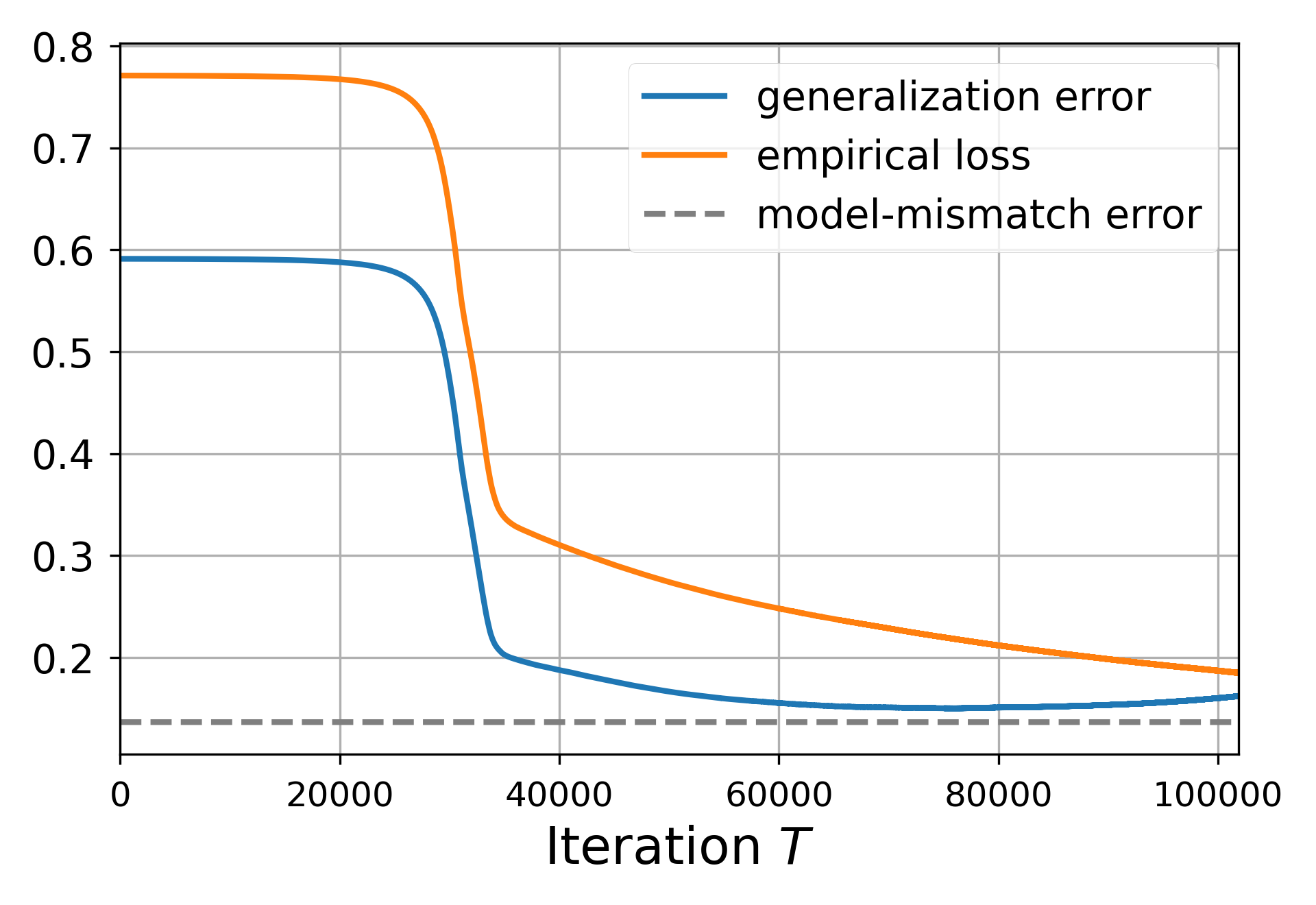}}\label{fig::3-layer-nn}}
        \subfloat[\footnotesize 4-layer ReLU network]{
            {\includegraphics[width=0.32\linewidth]{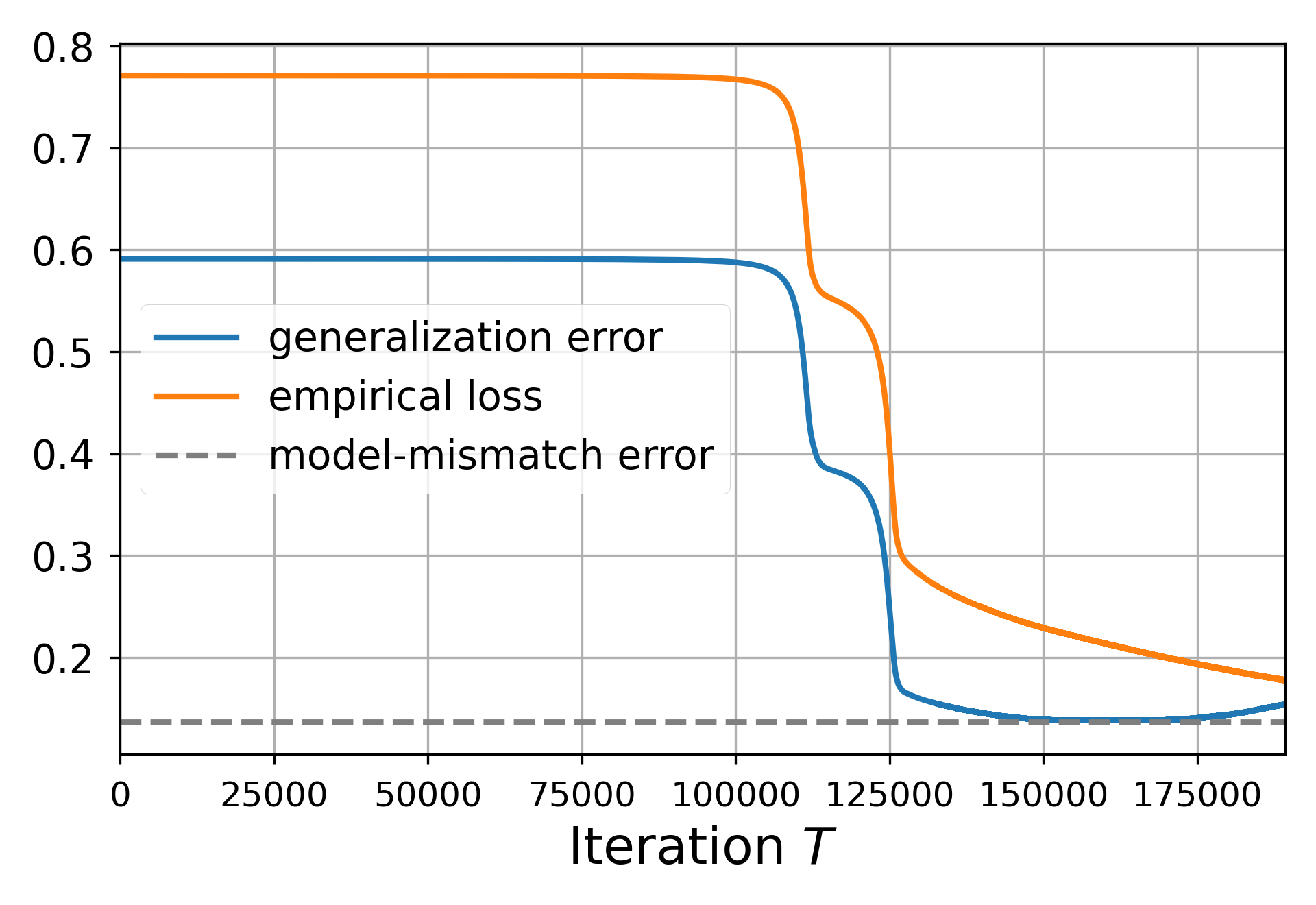}}\label{fig::4-layer-nn}}
    \end{center}\vspace{-3mm}
    \caption{\footnotesize \textbf{Deep matrix recovery (first row).} The ground truth $X^{\star}\in \R^{20\times 20}$ with $\rank(X^{\star})=3$ is chosen randomly. The elements of the measurement matrices are selected from $\cN(0,1)$, and the sample size is set to $m=180$. The corruption probability is set to $p=0.05$ with distribution $\cN(0, 100)$. We use SubGM with step-size $\eta=0.001$ and Gaussian initialization with an initialization scale $\alpha=1\times 10^{-3}$.
    \textbf{ReLU models (second row).} The samples are chosen from $y_i=\sin(\theta^{\star\top}x_i)+\err_i$ where $\theta^{\star}\in \R^{50}$ is randomly generated with $\norm{\theta^{\star}}_0=2$, $x_i\sim \cN(0,I_{50})$, and $\err_i\sim \cN(0, 25)$ with corruption probability $p=0.05$. The sample size is set to be $m=1500$. We use SubGM with step-size $\eta=0.001$ and Gaussian initialization $\cN(0, \alpha^{2/N}/d)$.
    }
    \label{fig::application}
\end{figure*}

\section{Numerical Experiments: Beyond Linear Regression}

In this section, we empirically verify that the benefits of depth extend to the robust variants of deep matrix recovery and ReLU networks with $\ell_1$-loss.

\paragraph{Deep Matrix Recovery.}
In low-rank matrix recovery, the goal is to recover a low-rank matrix $X^\star\in\R^{d\times d}$, from a limited number of noisy measurements of the form $y_i = \inner{A_i}{X^\star}+\err_i$.
To recover $X^\star$, we consider a deep factorized model of the form $W_1W_2\dots W_N$, where $W_i\in\R^{d\times d}$ for $i=1,\dots,N$, and minimize the $\ell_1$-loss $(1/m) \sum_{i=1}^{m}\left|y_i-\inner{A_i}{W_1W_{2}\cdots W_N}\right|$ via SubGM.
When $N=2$, the above model reduces to the famous Burer-Monteiro approach~\cite{burer2003nonlinear}. We assume that $5\%$ of the measurements are grossly corrupted with noise. The first row of Figure~\ref{fig::application} shows the performance of SubGM on 2-, 3-, and 4-layer models. It can be seen that the 4-layer model outperforms shallower models, achieving a generalization error that is proportional to the step-size.

\paragraph{Deep ReLU Network on Synthetic Dataset.} As another experiment, we analyze the effect of depth on the performance of SubGM with ReLU networks and $\ell_1$-loss. Given an input $x\in \R^d$, the output of an $N$-layer ReLU network is defined as $f_{\mathbf{W}}(x)= W_{N} \sigma\left(W_{N-1} \cdots \sigma\left(W_{1} x\right) \cdots\right)$,
where $W_{1} \in \R^{m \times d}, W_{2}, \cdots, W_{N-1} \in \R^{m \times m}$, and $W_{N} \in \mathbb{R}^{1 \times m}$. Moreover, $\sigma(x)=\max \{0, x\}$ is the ReLU activation function. Given the true function $f^\star(x) = \sin(\theta^{\star\top}x)$, our goal is to train a ReLU model to approximate $f^\star$ as accurately as possible. To this goal, we minimize the $\ell_1$-loss $(1/m) \sum_{i=1}^{m}\left|y_i-f_{\mathbf{W}}(x_i)\right|$.
The second row of Figure~\ref{fig::application} illustrates the performance of SubGM. It is worth noting that, unlike robust linear regression and deep matrix recovery, there always exists a non-diminishing model-mismatch error between the true and considered ReLU model (shown as a dashed line). Nonetheless, SubGM can achieve this model-mismatch error on a 4-layer ReLU model with only 1500 samples, even if $5\%$ of the measurements are corrupted with large noise.

\paragraph{Deep ReLU Network on CIFAR Dataset}
We verify that the desirable performance of SubGM with $\ell_1$-loss can be extended to its stochastic variant with mini-batches on CIFAR-10 and CIFAR-100 \cite{krizhevsky2009learning}, outperforming cross-entropy (CE) loss, which is considered as one of the most suitable loss functions for CIFAR datasets. To show this,
we use standard ResNet architectures \cite{he2016deep} with $\ell_1$-loss and compare it with the cross-entropy loss on noisy CIFAR datasets, where we randomize the labels of $10\%$ of the training dataset. For CIFAR-100 experiment, we use the “loss scaling” trick introduced in \cite{hui2020evaluation}. The training details are deferred to Section~\ref{sec::cifar}. The best test accuracy for both CIFAR-10 and CIFAR-100 is reported in Table~\ref{table::table}. One can see that $\ell_1$-loss outperforms cross-entropy loss significantly, demonstrating that our framework may be extended to more realistic settings. Moreover, we do observe that the deeper model performs better on CIFAR-100, which aligns with our theoretical result. Based on our simulations, an interesting and important future direction would be to study the optimization landscape of $\ell_1$-loss with more general neural network architectures.

\begin{table}[h]
    \centering
    \begin{tabular}{cccc|ccc}
        \toprule
        \multirow{2}{*}{Method} & \multicolumn{3}{c}{CIFAR-10} & \multicolumn{3}{c}{CIFAR-100}                                                                                     \\
        \cmidrule{2-4} \cmidrule{5-7}                                                                                                                                              \\
        {}                      & ResNet-18                    & ResNet-34                     & ResNet-50          & ResNet-18          & ResNet-34          & ResNet-50          \\
        \midrule
        CE loss                 & $91.52\%$                    & $91.53\%$                     & $90.87\%$          & $70.17\%$          & $71.22\%$          & $71.30\%$          \\
        $\ell_1$-loss           & $\mathbf{94.16\%}$           & $\mathbf{93.13\%}$            & $\mathbf{92.68\%}$ & $\mathbf{73.69\%}$ & $\mathbf{74.27\%}$ & $\mathbf{75.19\%}$ \\
        \bottomrule
    \end{tabular}
    \caption{Test accuracy for ResNet-18, 34, 50 on CIFAR-10 and CIFAR-100 datasets with noise.}
    \label{table::table}
\end{table}

\section{Conclusion}
Modern problems in machine learning are naturally nonconvex but can be solved reasonably well in practice. To explain this, a recent body of work has postulated that many optimization problems in machine learning are ``convex-like'', i.e., they are devoid of spurious local minima. Our work shows that such global property is too restrictive to hold  even in the context of linear regression, and instead propose a more refined \textit{trajectory analysis} to better capture the landscape of the problem around the solution trajectory. We show that convex models may be fundamentally ill-suited for linear models, and deeper models--despite their nonconvexity--have provably better optimization landscape around the solution trajectory. Empirically, we show that our analysis may extend beyond linear regression; a formal verification of this conjecture is considered as an enticing challenge for future research.

\section*{Acknowledgements}
We thank Richard Y. Zhang, C\'edric J\'osz, and Tiffany Wu for helpful discussions and feedback. We are also thankful for an anonymous reviewer for pointing out the relationship between the perturbation ball and the depth of linear models.
This research is supported, in part, by NSF Award DMS-2152776, ONR Award N00014-22-1-2127, MICDE Catalyst Grant, MIDAS PODS grant and Startup funding from the University of Michigan.
\bibliography{ref}
\bibliographystyle{plain}


\newpage
\appendix

\section{Additional Experiments}
In this section, we provide additional experiments on the performance of SubGM on deep models. Our goal is to verify our theoretical results and show the benefits of both small initialization and geometric step-size. Moreover, we show that the desirable performance of SubGM can be observed in its stochastic variant, as well as for different architectures of ResNets with $\ell_1$-loss, and more realistic CIFAR-10 dataset. All simulations are run on a desktop computer with an Intel Core i9 3.50 GHz CPU and 128GB RAM. The reported results are for an implementation in Python.
\subsection{Deeper Linear Models}
\begin{figure*}
    \begin{centering}
        \subfloat[4-layer]{
            {\includegraphics[width=0.32\linewidth]{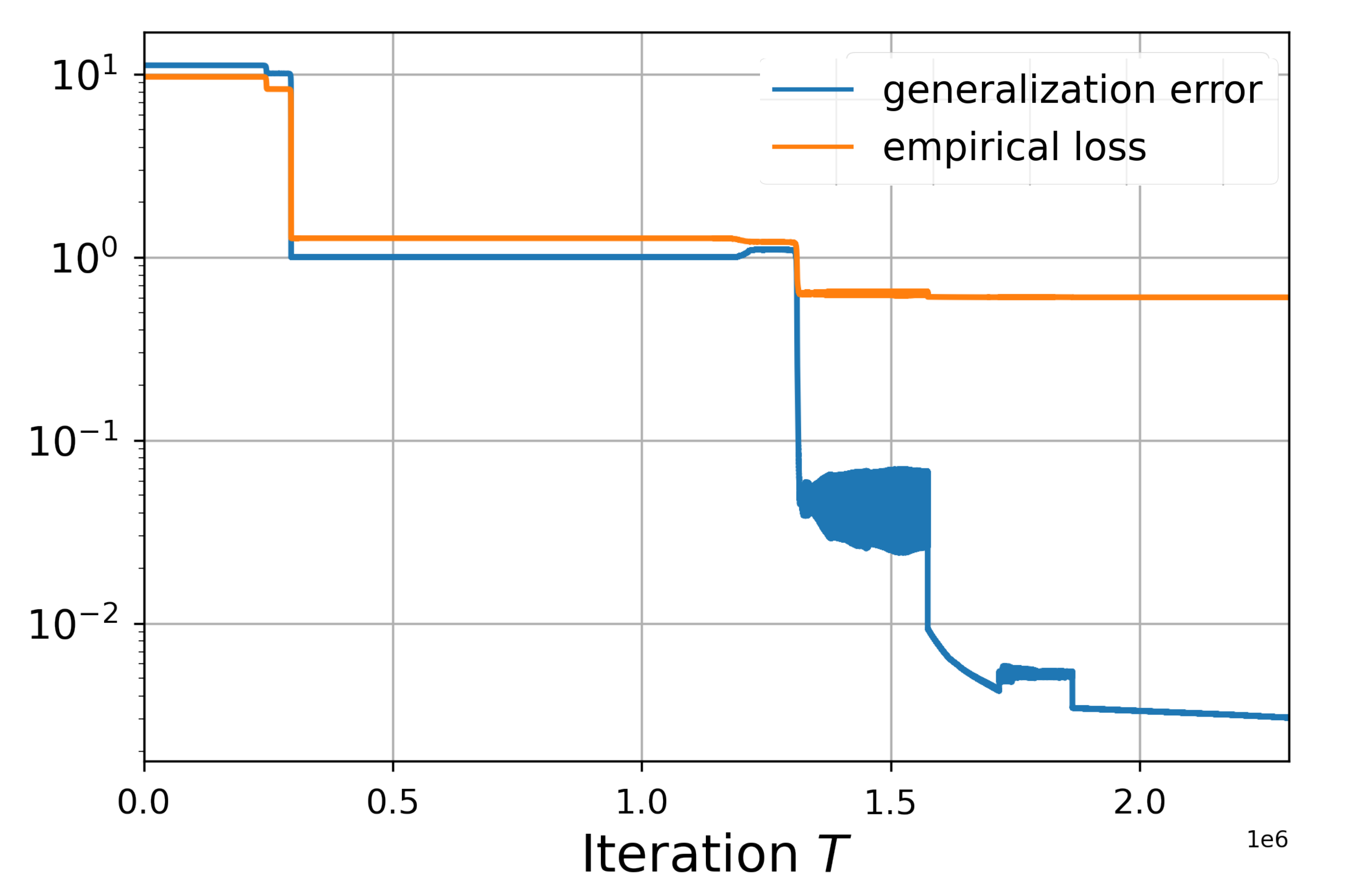}}\label{fig::4-layer}}
        \subfloat[5-layer]{
            {\includegraphics[width=0.32\linewidth]{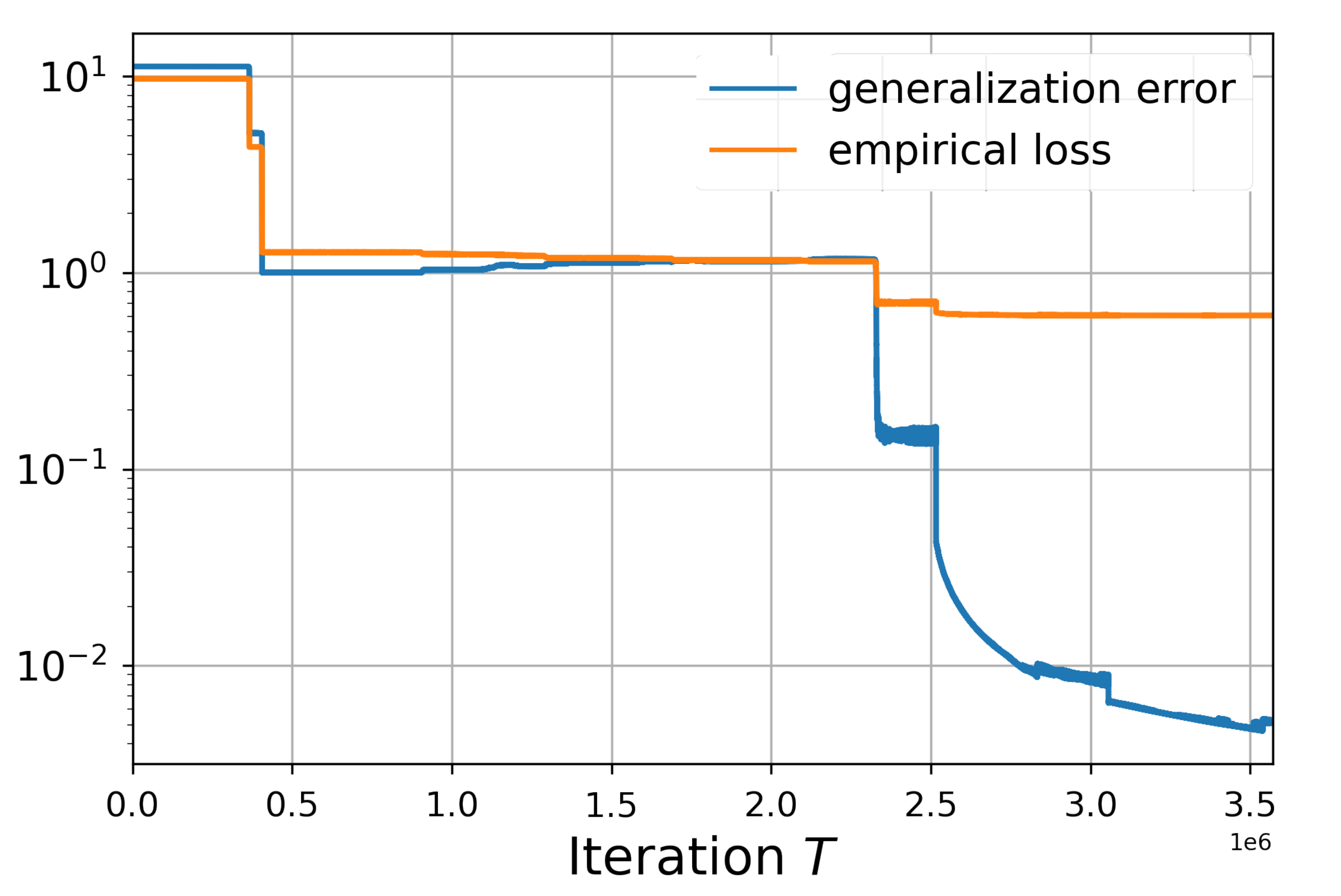}}\label{fig::5-layer}}
        \subfloat[6-layer]{
            {\includegraphics[width=0.32\linewidth]{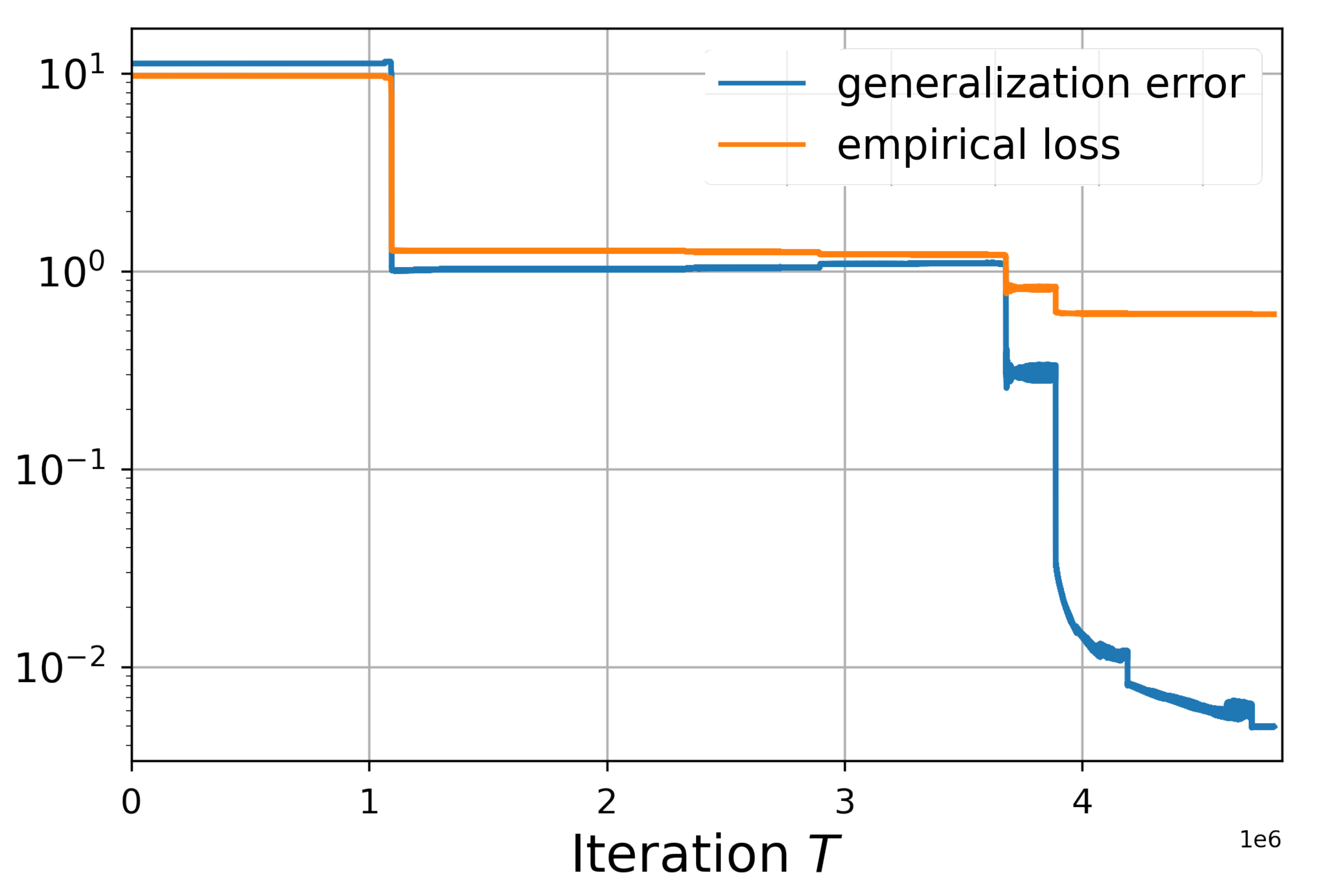}}\label{fig::6-layer}}
    \end{centering}
    \caption{\footnotesize
        The optimization trajectories of deep models ($N=4,5,6$).}
    \label{fig::deeper-model}
\end{figure*}
\begin{figure*}
    \begin{centering}
        \subfloat[2-layer, $\alpha=0.1$]{
            {\includegraphics[width=0.32\linewidth]{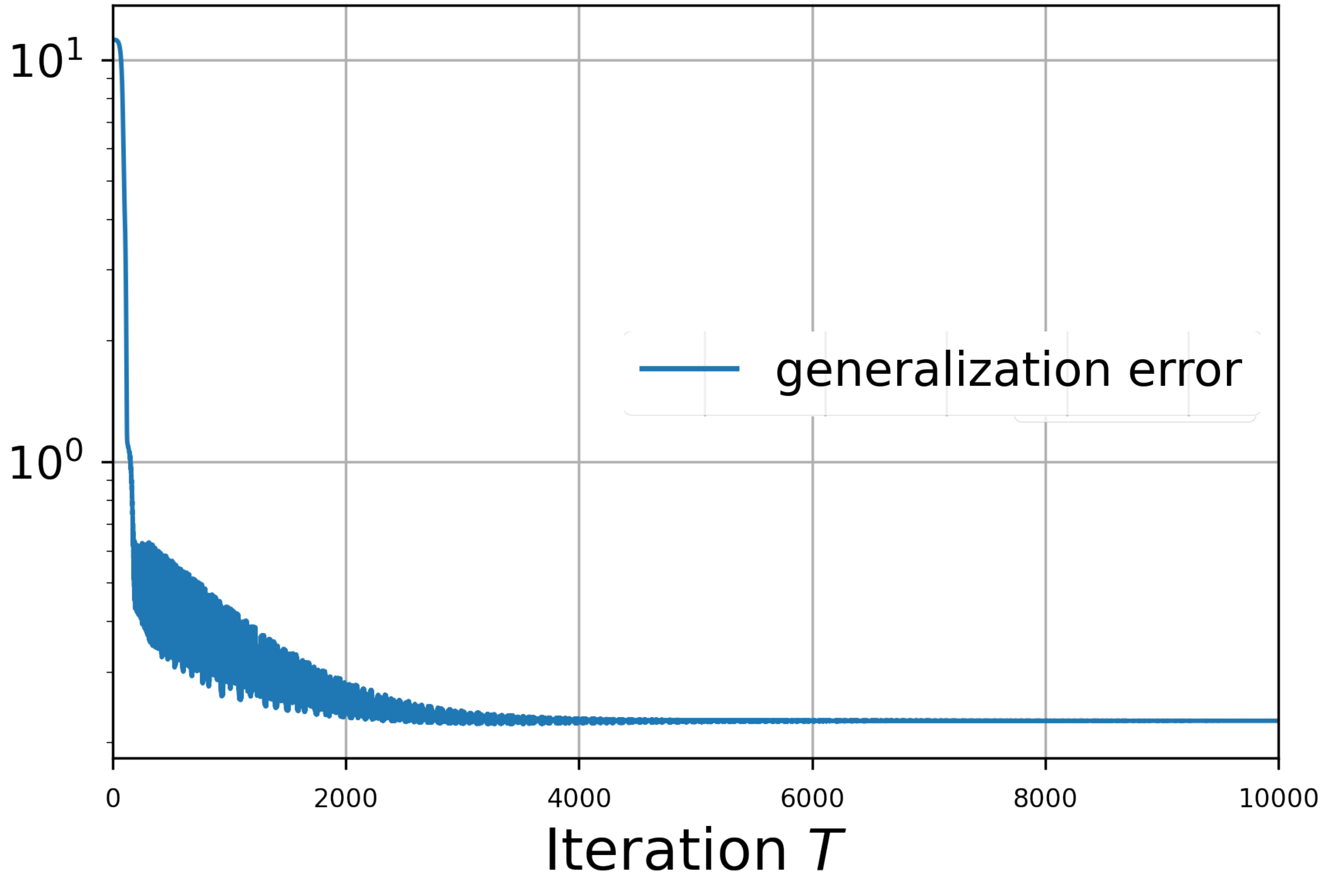}}\label{fig::2-layer-init-1}}
        \subfloat[2-layer, $\alpha=0.01$]{
            {\includegraphics[width=0.32\linewidth]{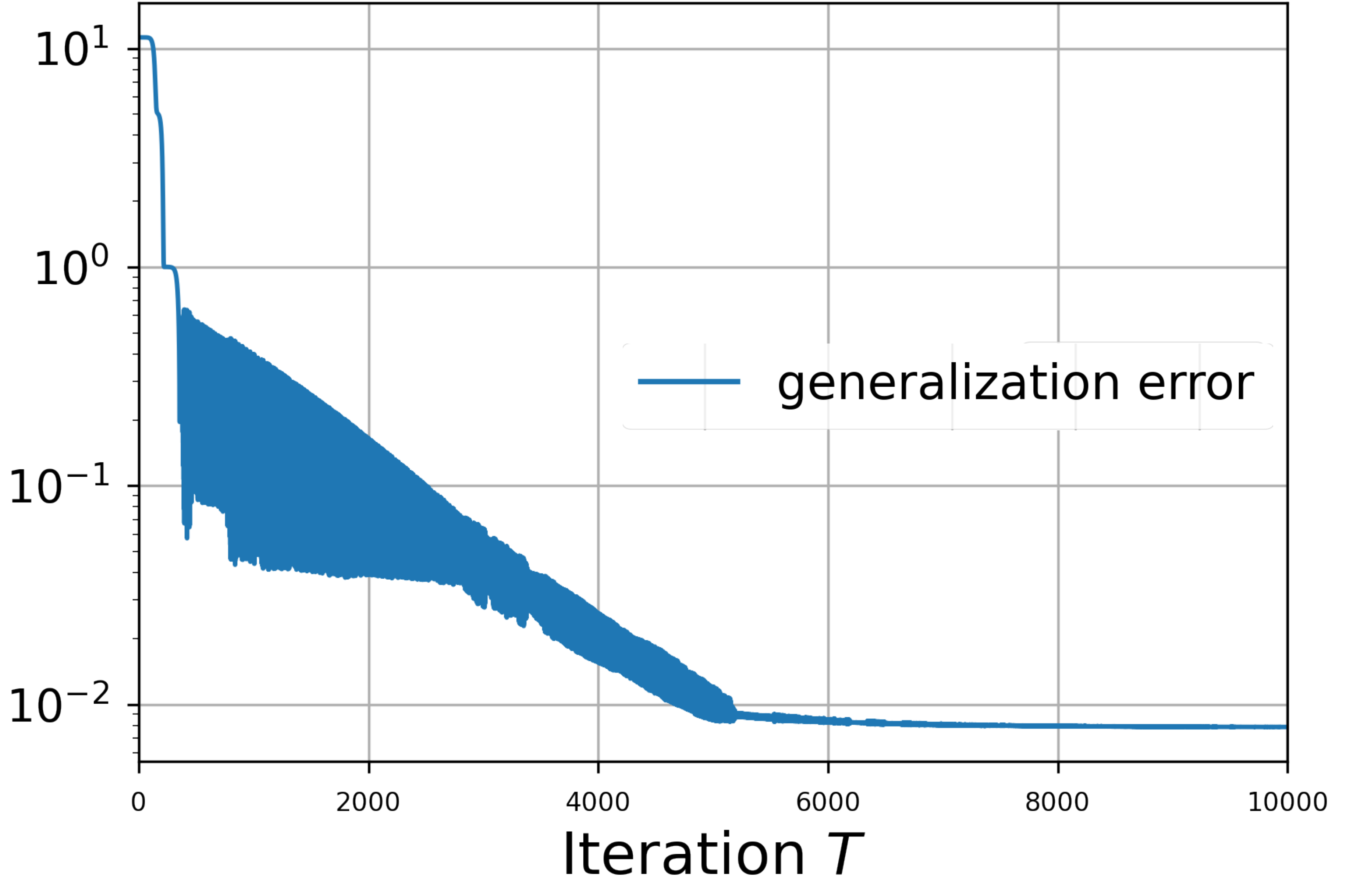}}\label{fig::2-layer-init-2}}
        \subfloat[2-layer, $\alpha=0.001$]{
            {\includegraphics[width=0.32\linewidth]{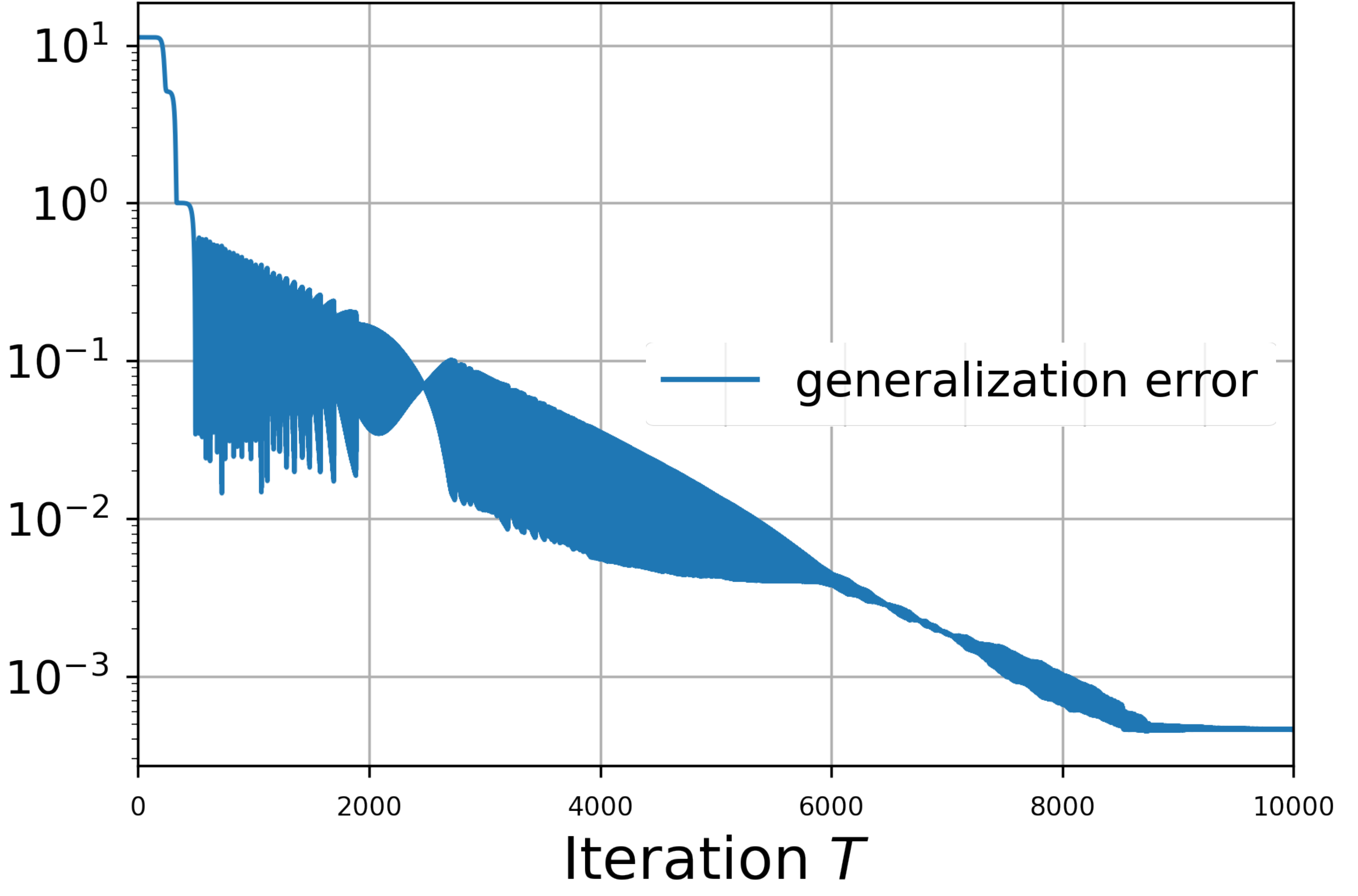}}\label{fig::2-layer-init-3}}\\
        \subfloat[3-layer, $\alpha=0.1$]{
            {\includegraphics[width=0.32\linewidth]{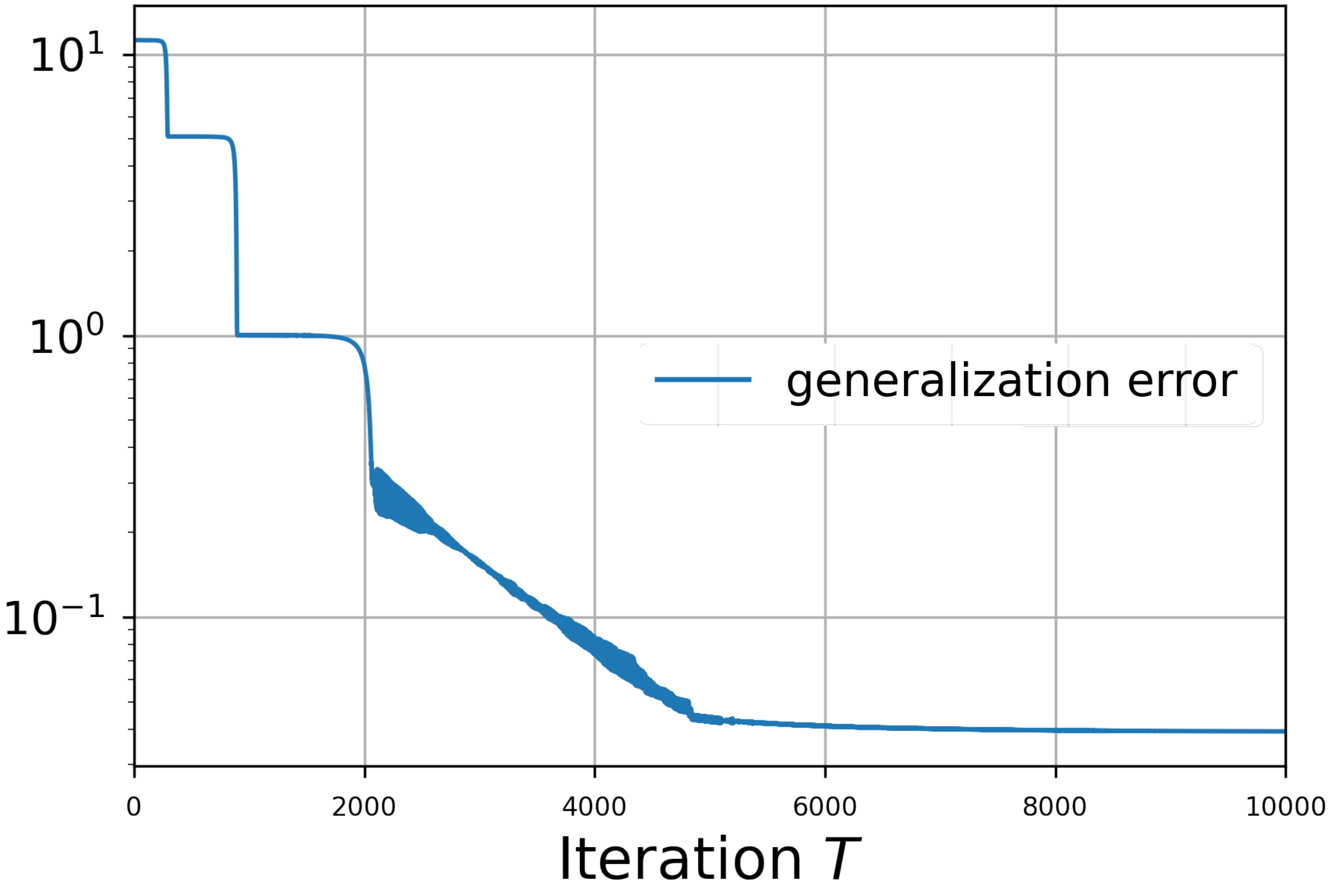}}\label{fig::3-layer-init-1}}
        \subfloat[3-layer, $\alpha=0.01$]{
            {\includegraphics[width=0.32\linewidth]{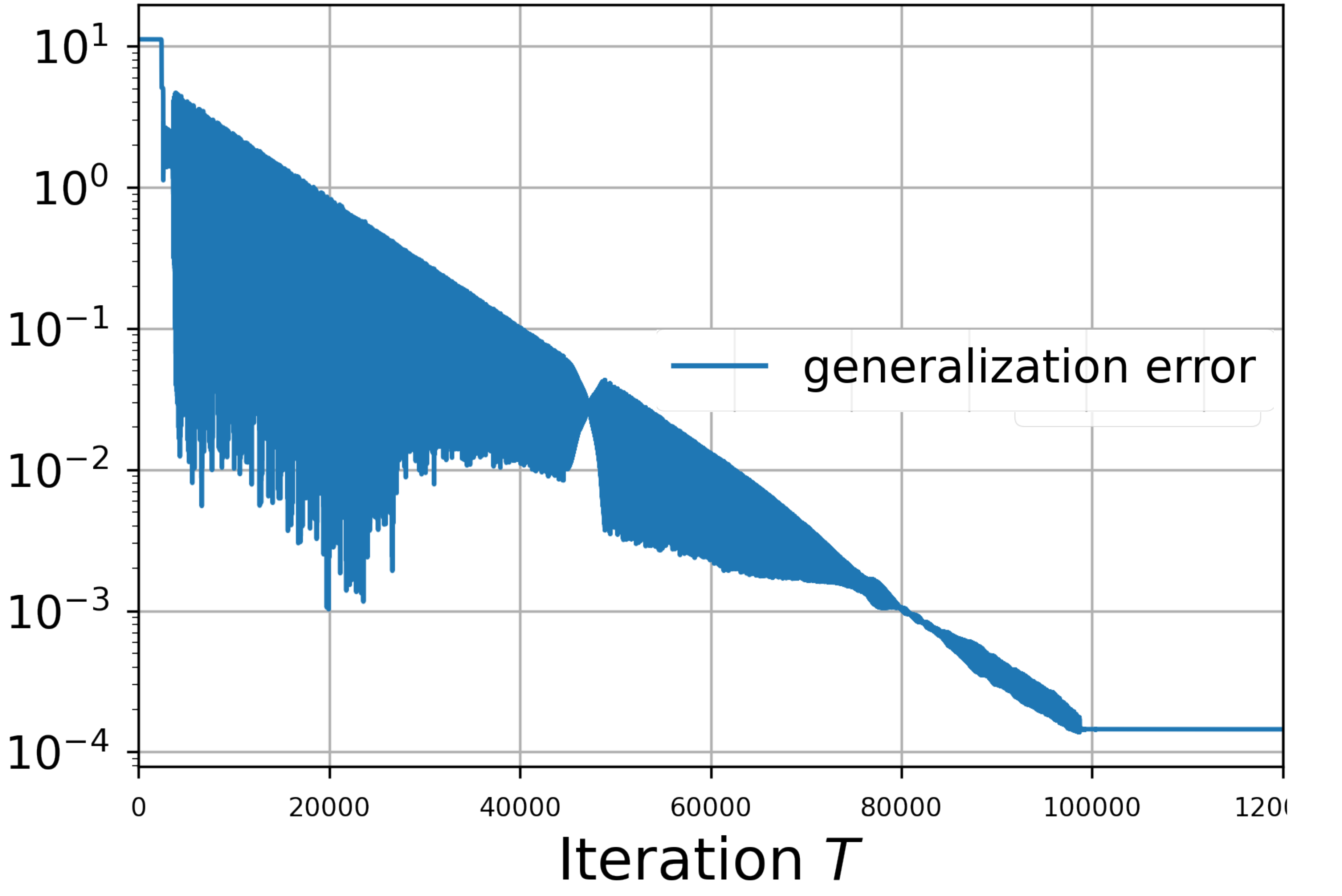}}\label{fig::3-layer-init-2}}
        \subfloat[3-layer, $\alpha=0.001$]{
            {\includegraphics[width=0.32\linewidth]{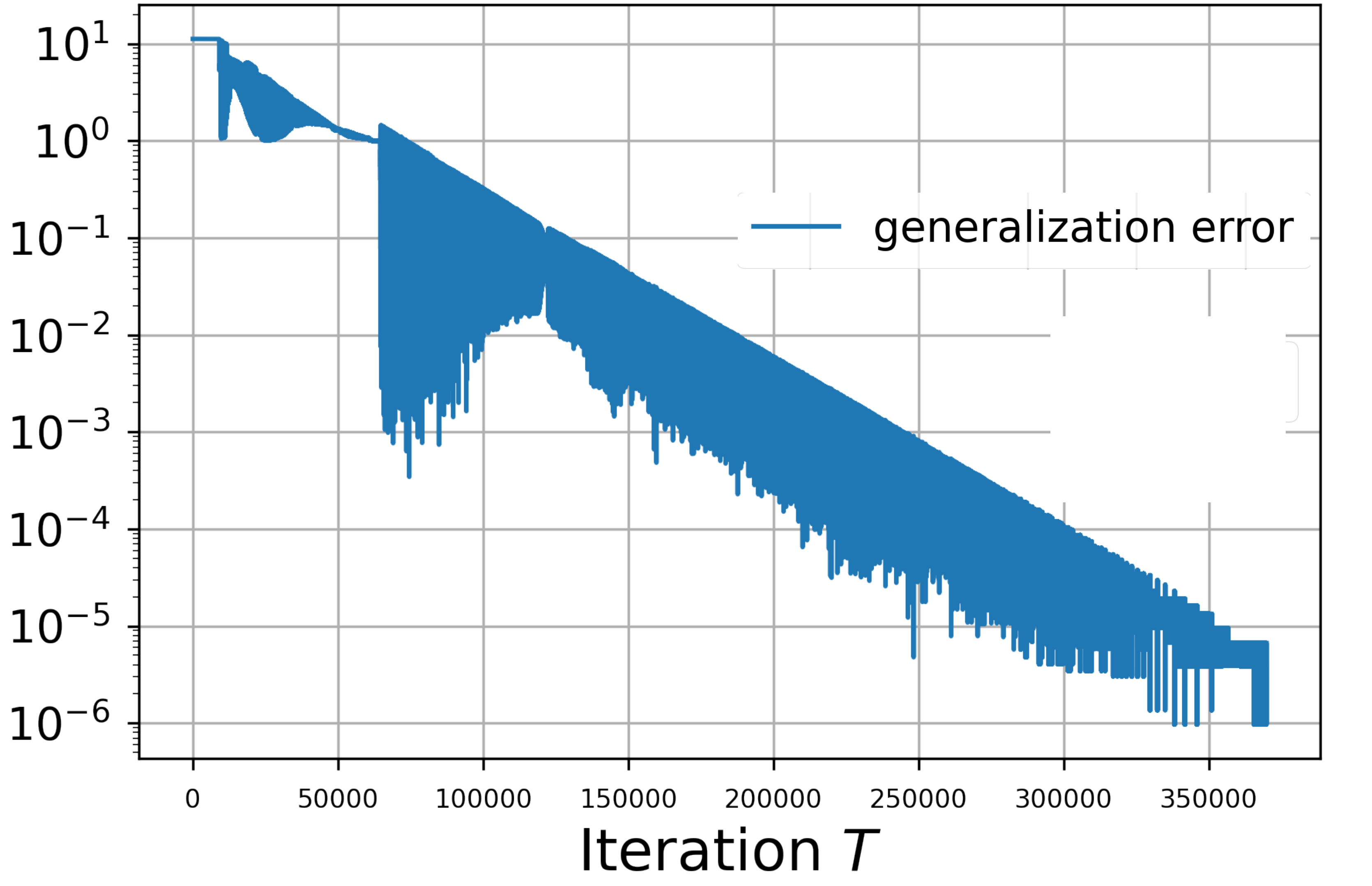}}\label{fig::3-layer-init3}}
    \end{centering}
    \caption{\footnotesize
        The optimization trajectories of 2 and 3-layer models with different initialization size $\alpha=0.1, 0.01, 0.001$ using exponentially decayed step-size.}
    \label{fig::exp}
\end{figure*}
In this experiment, we study the performance of SubGM on deeper models ($N=4,5,6$). To accelerate the training process, we first use a large step-size $\eta_1=1\times 10^{-3}$, and then progressively apply smaller step-sizes $\eta_2=1\times 10^{-4}$ and $\eta_3=1\times 10^{-5}$ as the training loss continues to decay. As shown in Figure~\ref{fig::deeper-model}, deeper models share similar generalization error, outperforming 1- and 2-layer models presented in Figure~\ref{fig::simulation}.
\subsection{Geometric Step-size}
As shown in our simulations, training $N$-layer models may require millions of iterations even on a small synthetic dataset. As proven in Theorem~\ref{thm::N-layer}, the training process may become even slower for deeper models. In this experiment, we explore the performance of geometric (i.e., exponentially decaying) step-size on the same dataset. SubGM with a geometric step-size has been widely used for the optimization of $\ell_1$-loss \cite{li2020nonconvex, ma2021sign}, and more general sharp weakly convex functions \cite{davis2018stochastic}. Figure~\ref{fig::exp} shows that a geometric steps-size can lead to a 1000-fold reduction in the required number of iterations. Moreover, a geometric step-size improves the convergence rate to linear. The theoretical justification of this improvement is left as an enticing challenge for future research. Finally, it can be observed that SubGM with geometric step-size performs surprisingly well on deeper models, achieving a generalization errors in the order of $10^{-6}$. This further supports the benefits of the depth.

\subsection{Experiments on CIFAR Dataset}
\label{sec::cifar}
In this section, we provide the training details for the experiments on both CIFAR-10 and CIFAR-100 where $10\%$ of the training data points are randomly labeled. For CIFAR-100 experiment, we use the “loss scaling” trick introduced in \cite{hui2020evaluation}.
In particular, we denote the neural network by $f_{\theta}:\bR^d\to \bR^C$, where $d$ is the input dimension and $C$ is the number of class. The standard $\ell_1$-loss for the one-hot encoded label vector can be written (at a single point) as
\begin{equation}
    \ell=\left|f_\theta(x)[c]-1\right|+\sum_{c'\neq c}\left|f_\theta(x)[c']\right|.
\end{equation}
Here $c$ is the position of the label and $f_{\theta}(x)[i]$ is the $i$-th coordinate of the prediction. The rescaled $\ell_1$-loss is defined by two parameters $k$ and $M$
as follows:
\begin{equation}
    \ell_{\text{scaling}}=k\cdot\left|f_\theta(x)[c]-M\right|+\sum_{c'\neq c}\left|f_\theta(x)[c']\right|.
\end{equation}

In our simulation, we choose $k=5$, and $M=2$. The evolution of the training and test accuracy for CIFAR-10 and CIFAR-100 with both $\ell_1$ and CE losses are shown in Figures~\ref{fig::cifar-10} and~\ref{fig::cifar-100}.
\begin{figure*}[t]
    \begin{center}
        \subfloat[\footnotesize ResNet-18, $\ell_1$-loss]{
            {\includegraphics[width=0.33\linewidth]{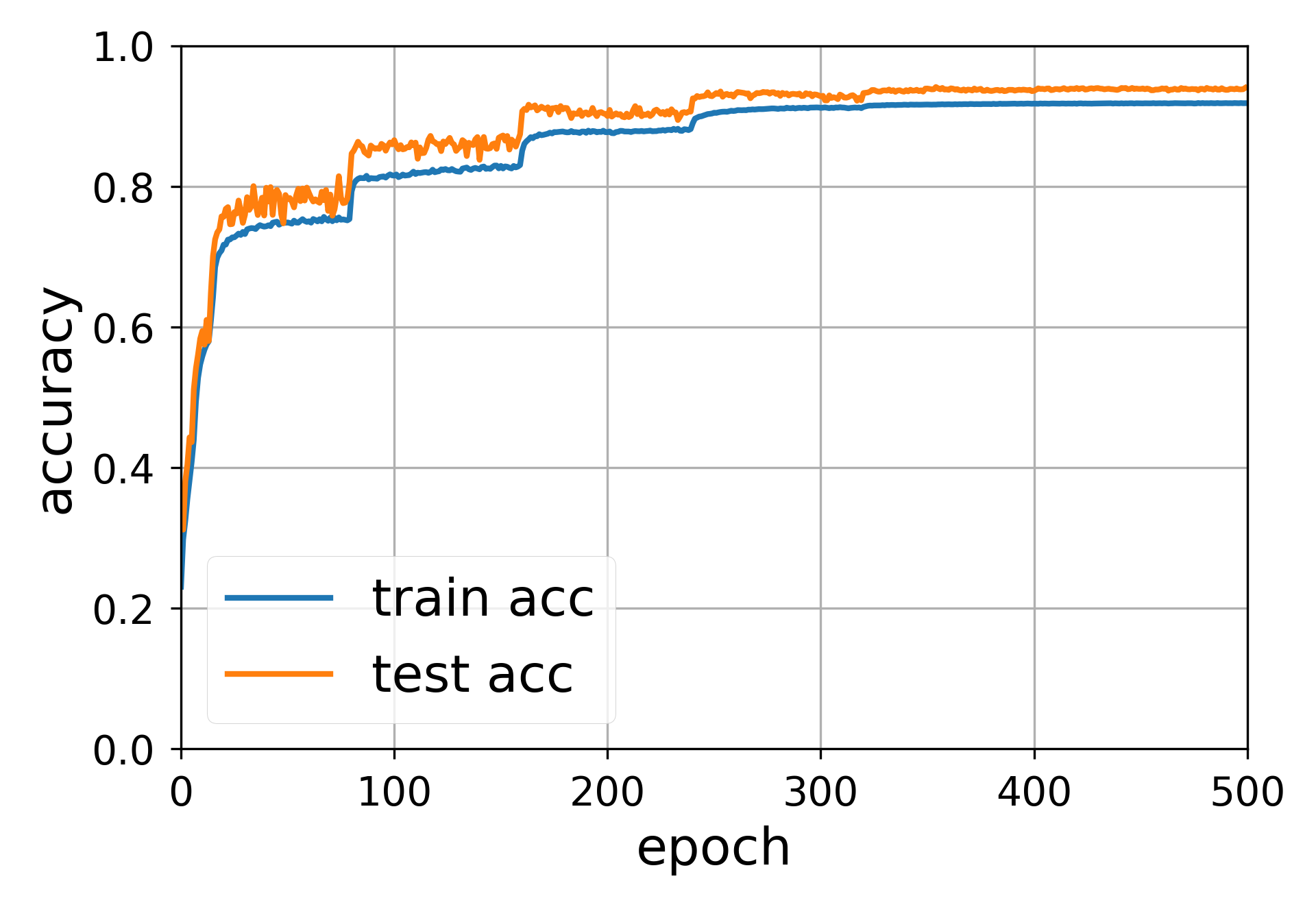}}\label{fig::resnet-18-l1}}
        \subfloat[\footnotesize ResNet-34, $\ell_1$-loss]{
            {\includegraphics[width=0.33\linewidth]{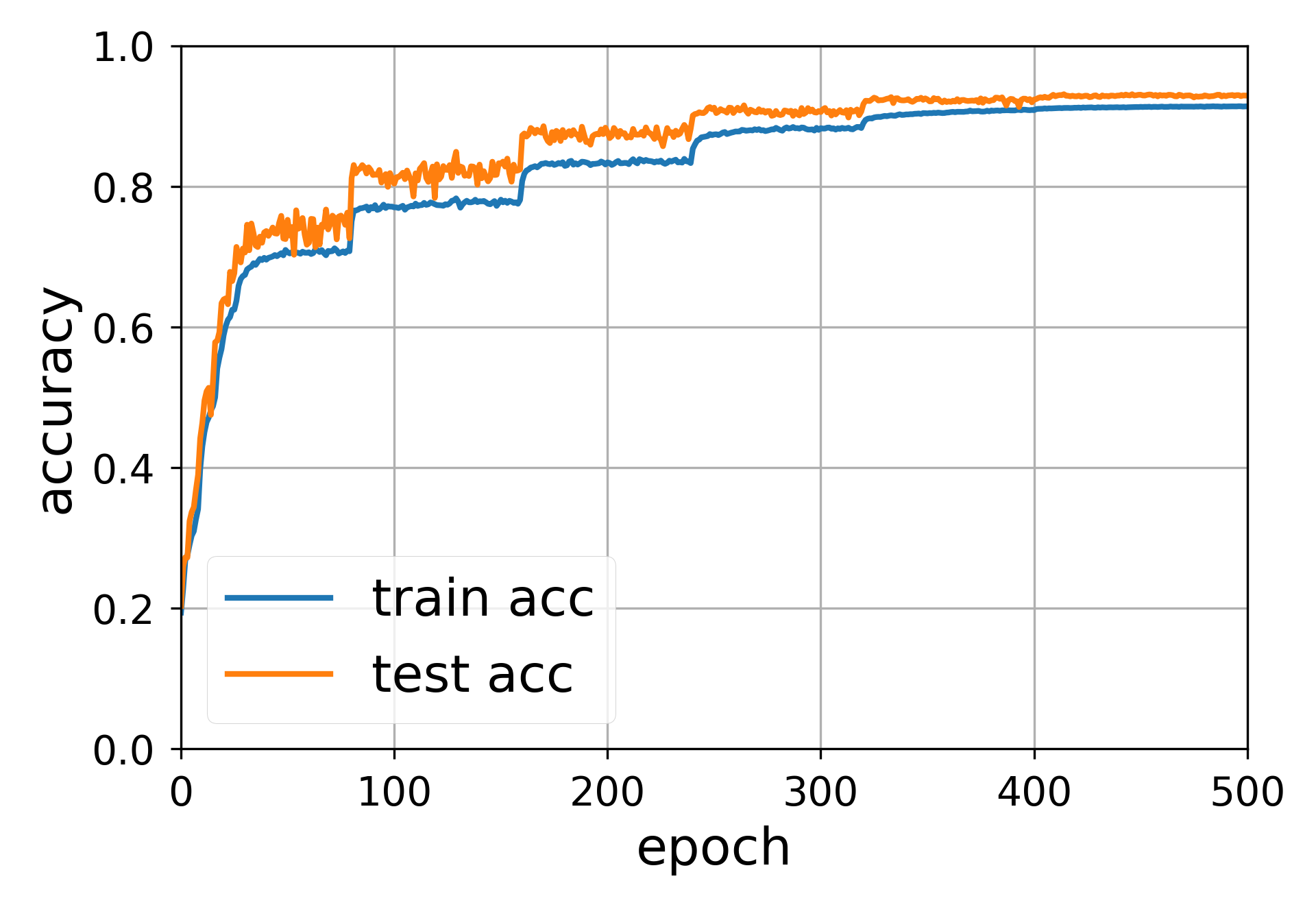}}\label{fig::resnet-34-l1}}
        \subfloat[\footnotesize ResNet-50, $\ell_1$-loss]{
            {\includegraphics[width=0.33\linewidth]{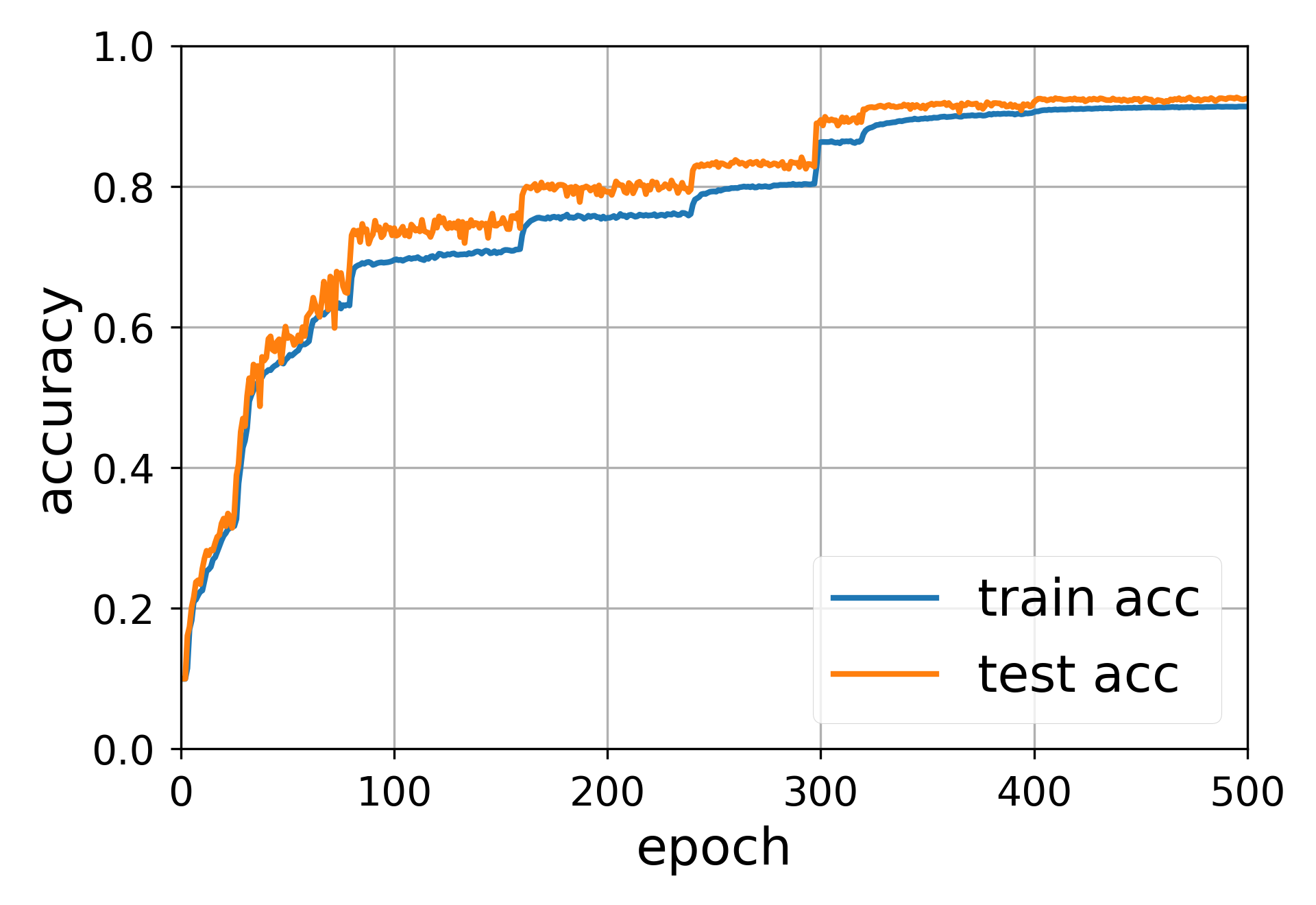}}\label{fig::resnet-50-l1}}\\
        \subfloat[\footnotesize ResNet-18, CE]{
            {\includegraphics[width=0.33\linewidth]{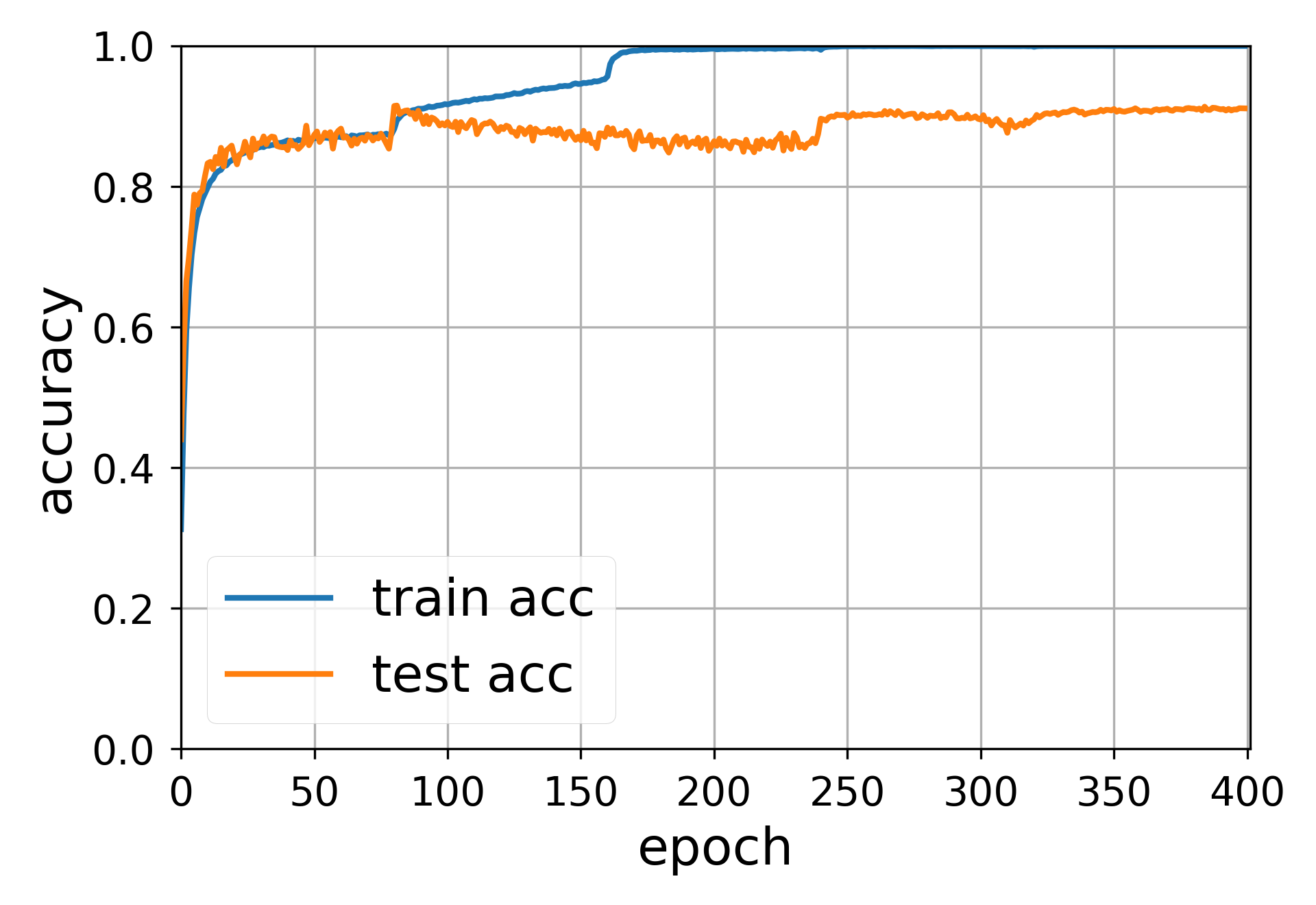}}\label{fig::resnet-18-CE}}
        \subfloat[\footnotesize ResNet-34, CE]{
            {\includegraphics[width=0.33\linewidth]{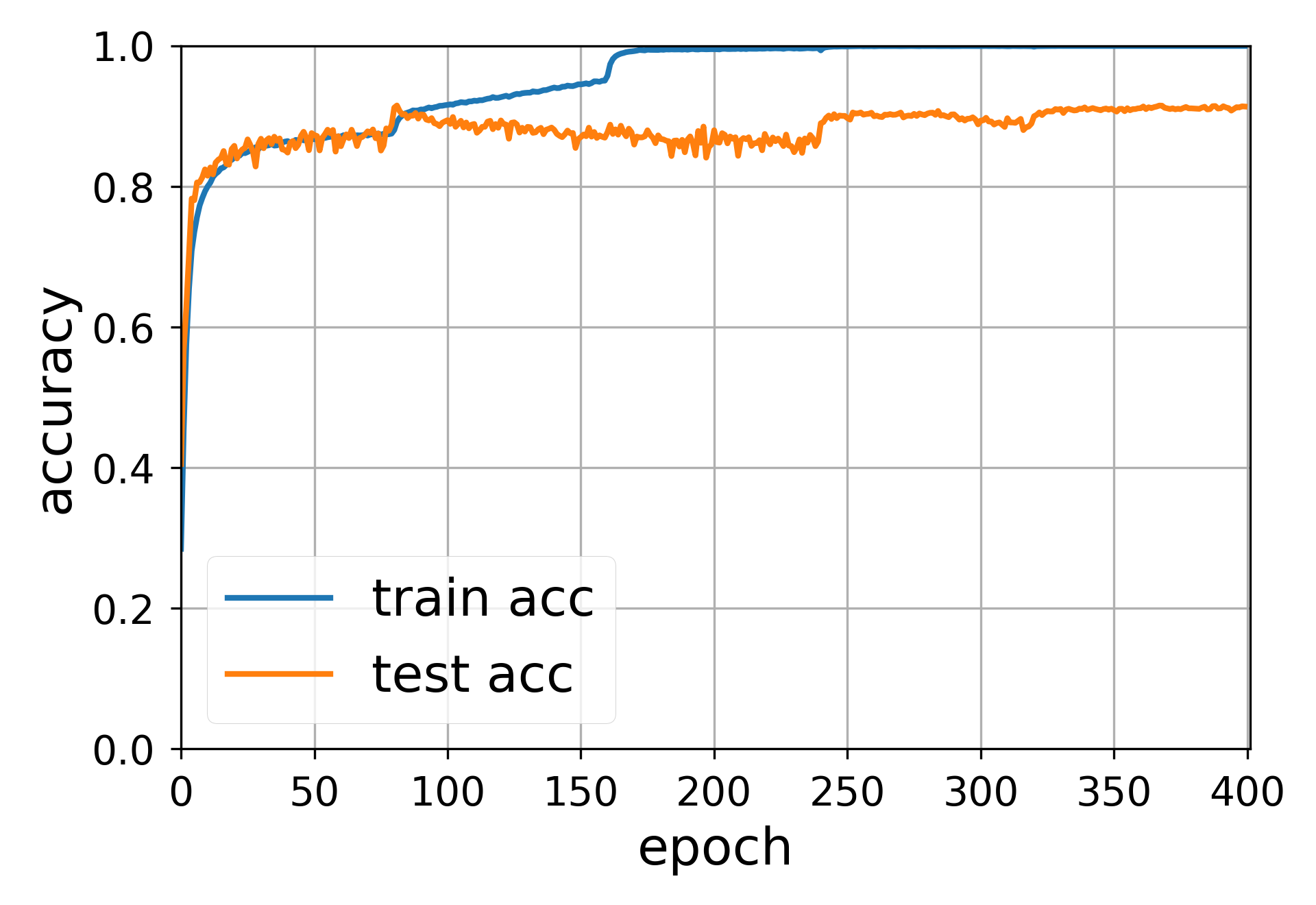}}\label{fig::resnet-34-CE}}
        \subfloat[\footnotesize ResNet-50, CE]{
            {\includegraphics[width=0.33\linewidth]{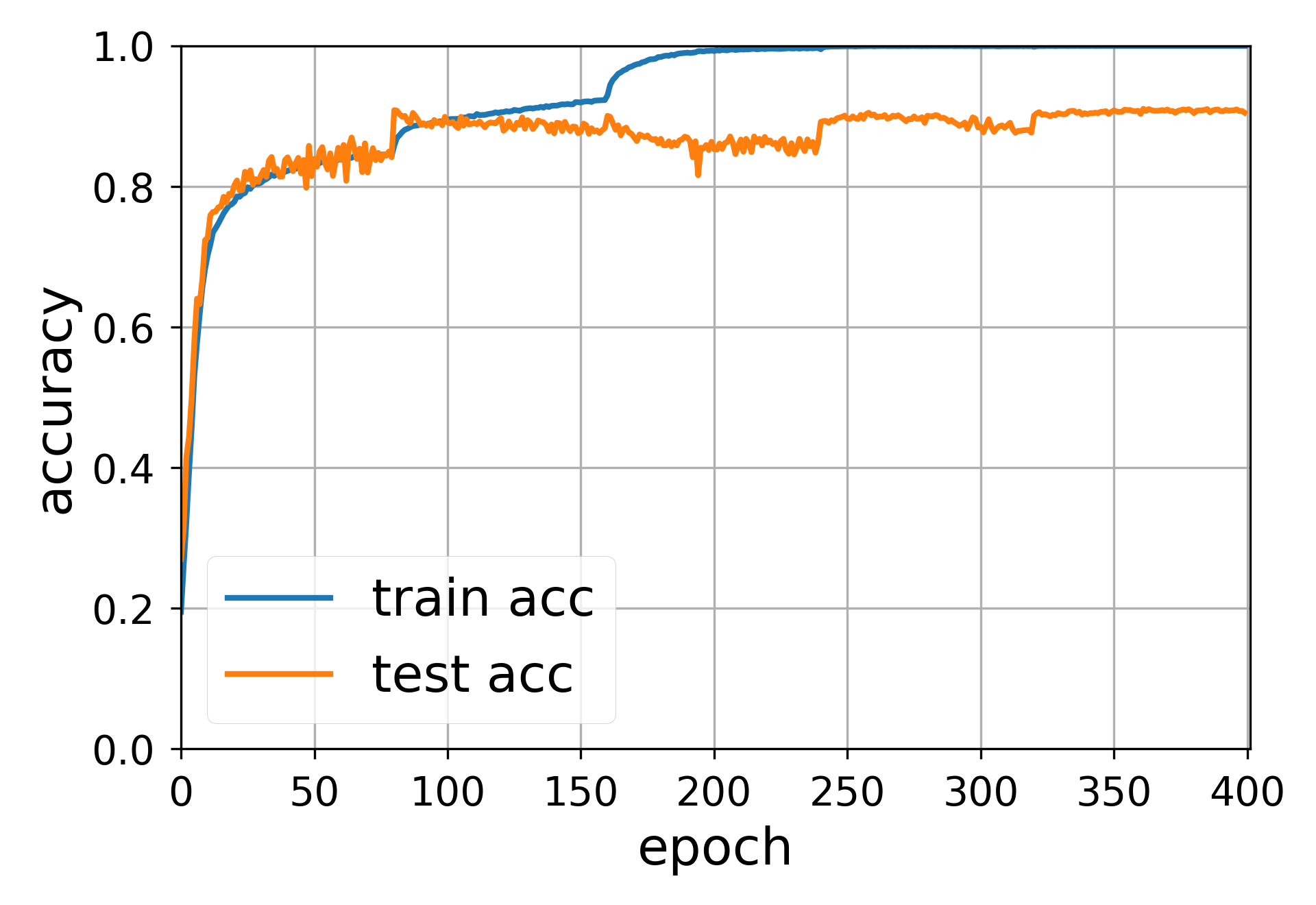}}\label{fig::resnet-50-CE}}
    \end{center}
    \caption{\footnotesize
        We apply ResNet-18, 34, 50 on noisy CIFAR-10 with both $\ell_1$-loss and cross-entropy loss (CE). For the training dataset, we randomly choose $10\%$ samples and replace their labels with uniform random labels. We use SGD with initial learning rate $\eta=0.1$, momentum $0.9$, batch size $B=32$. For every $80$ epochs, we decay the learning rate by a factor $0.33$. We use standard data augmentation. The initialization is set by default in PyTorch.
    }
    \label{fig::cifar-10}
\end{figure*}

\begin{figure*}[t]
    \begin{center}
        \subfloat[\footnotesize ResNet-18, $\ell_1$-loss]{
            {\includegraphics[width=0.33\linewidth]{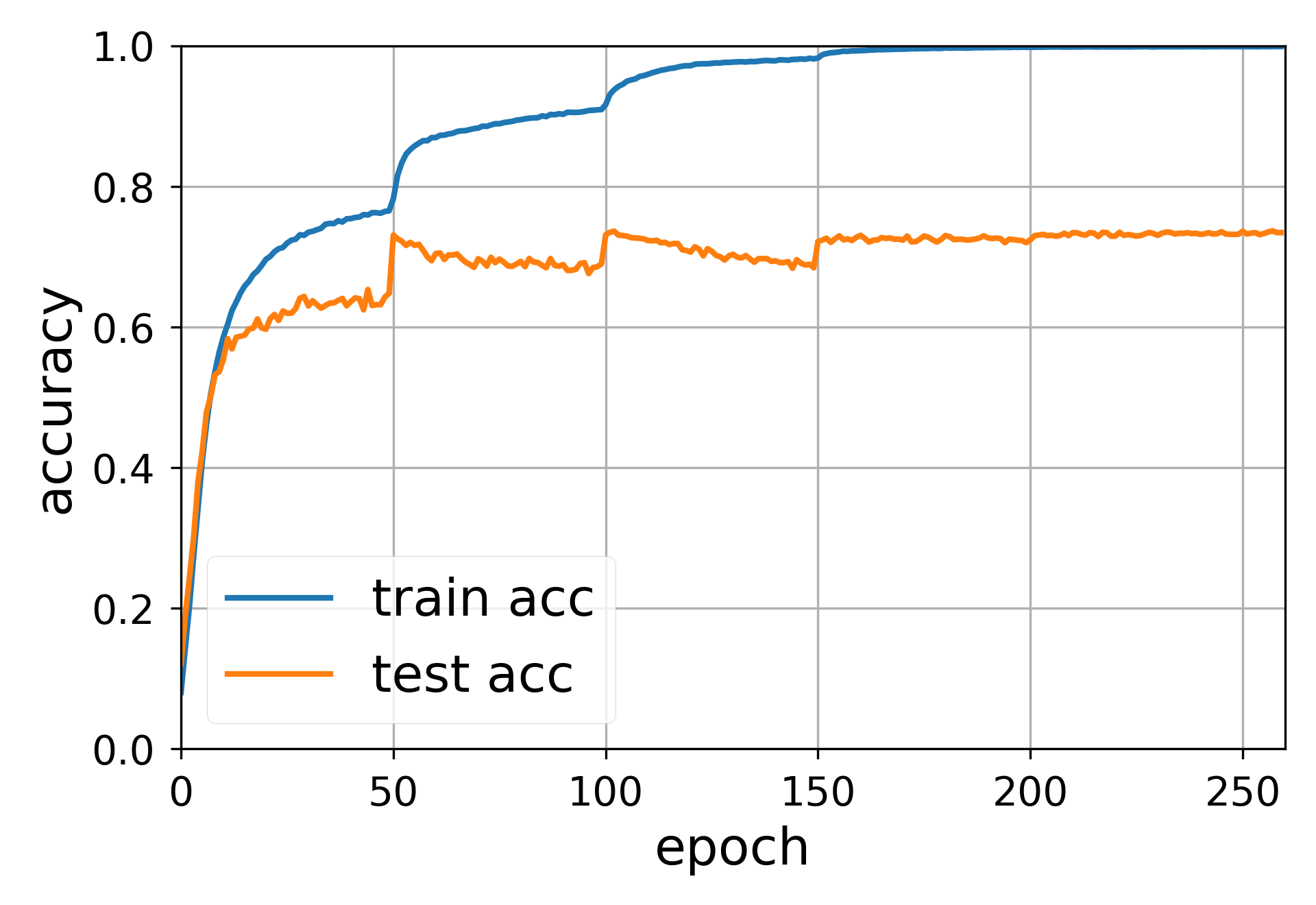}}\label{fig::resnet-18-l1-100}}
        \subfloat[\footnotesize ResNet-34, $\ell_1$-loss]{
            {\includegraphics[width=0.33\linewidth]{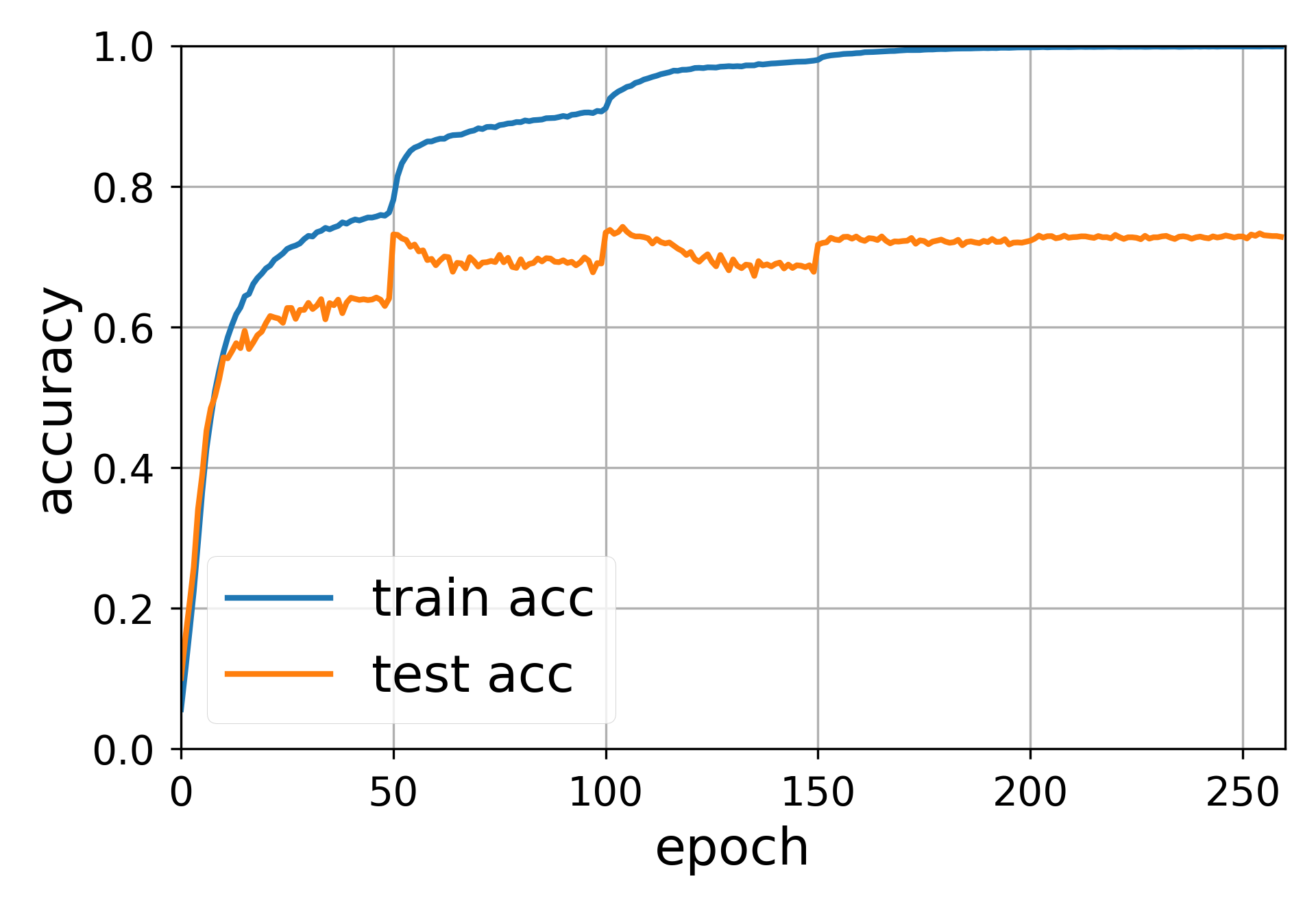}}\label{fig::resnet-34-l1-100}}
        \subfloat[\footnotesize ResNet-50, $\ell_1$-loss]{
            {\includegraphics[width=0.33\linewidth]{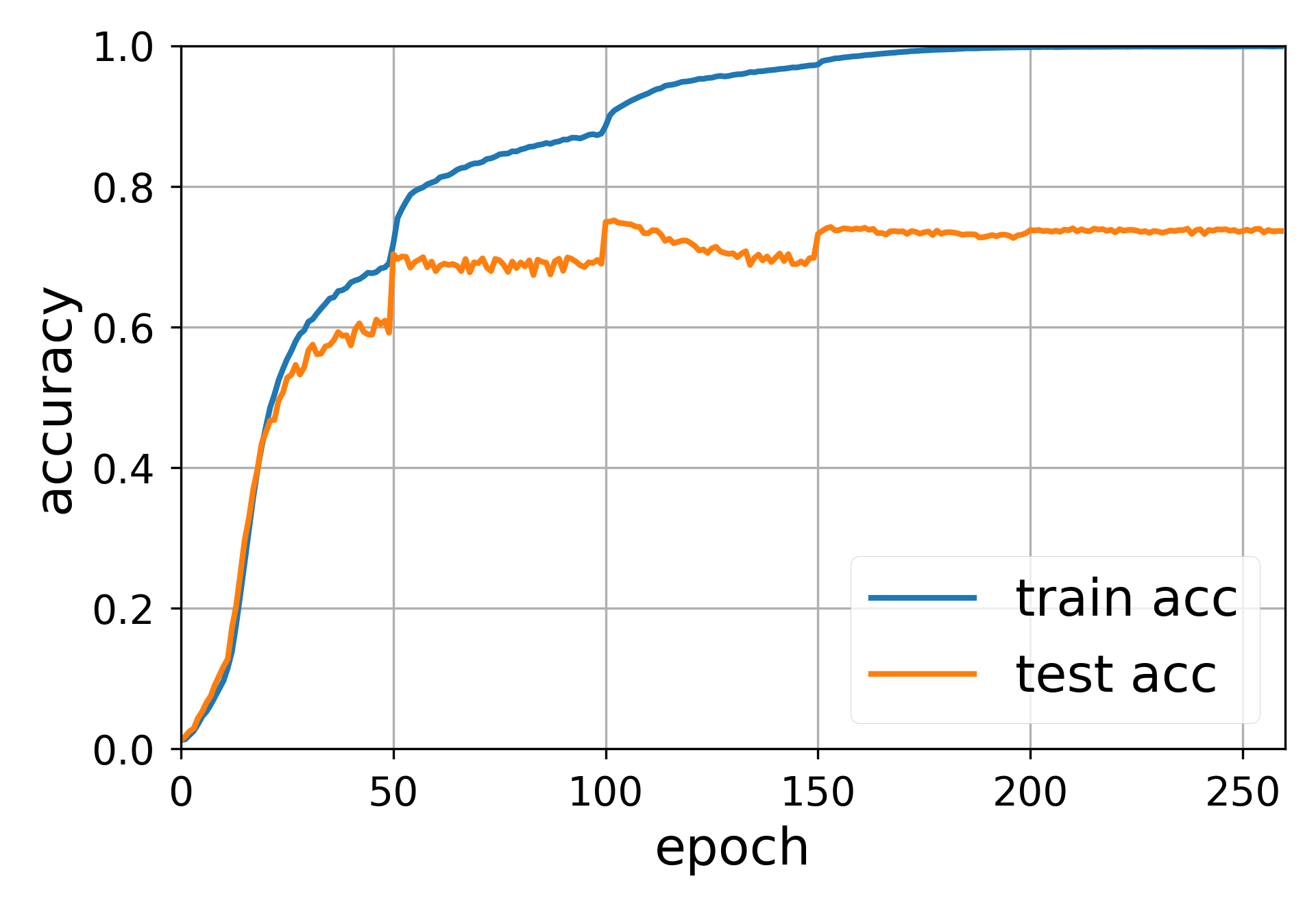}}\label{fig::resnet-50-l1-100}}\\
        \subfloat[\footnotesize ResNet-18, CE]{
            {\includegraphics[width=0.33\linewidth]{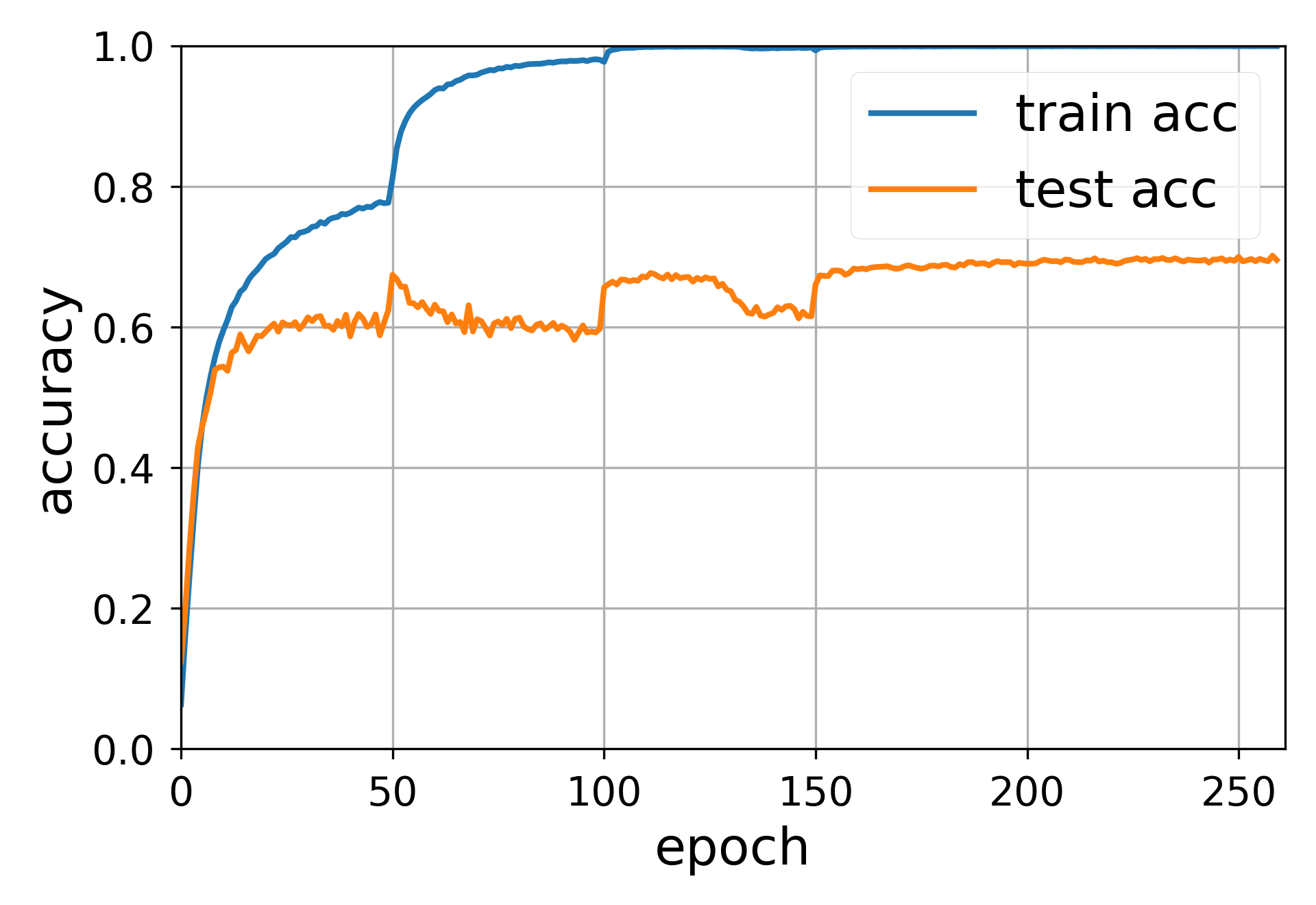}}\label{fig::resnet-18-CE-100}}
        \subfloat[\footnotesize ResNet-34, CE]{
            {\includegraphics[width=0.33\linewidth]{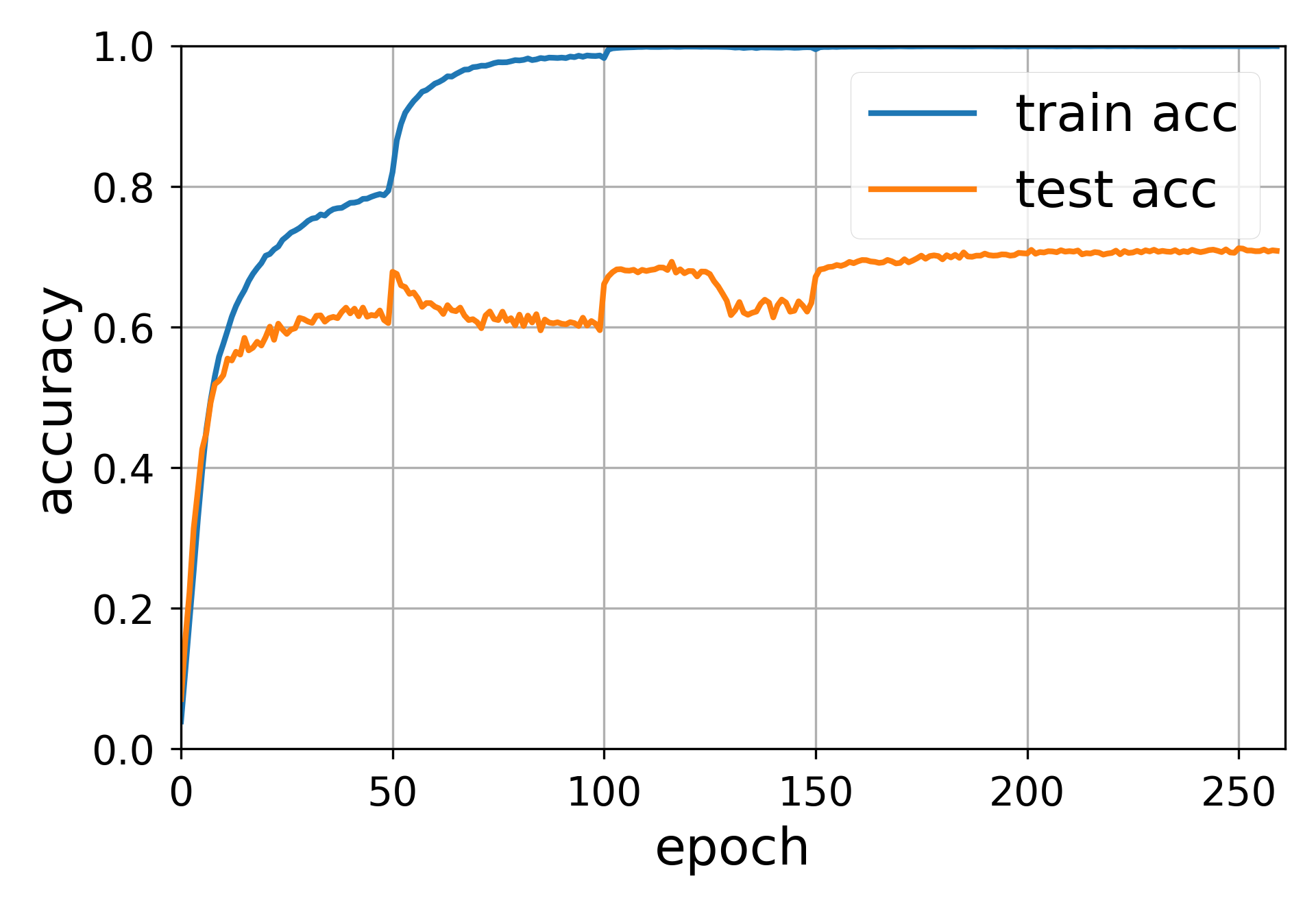}}\label{fig::resnet-34-CE-100}}
        \subfloat[\footnotesize ResNet-50, CE]{
            {\includegraphics[width=0.33\linewidth]{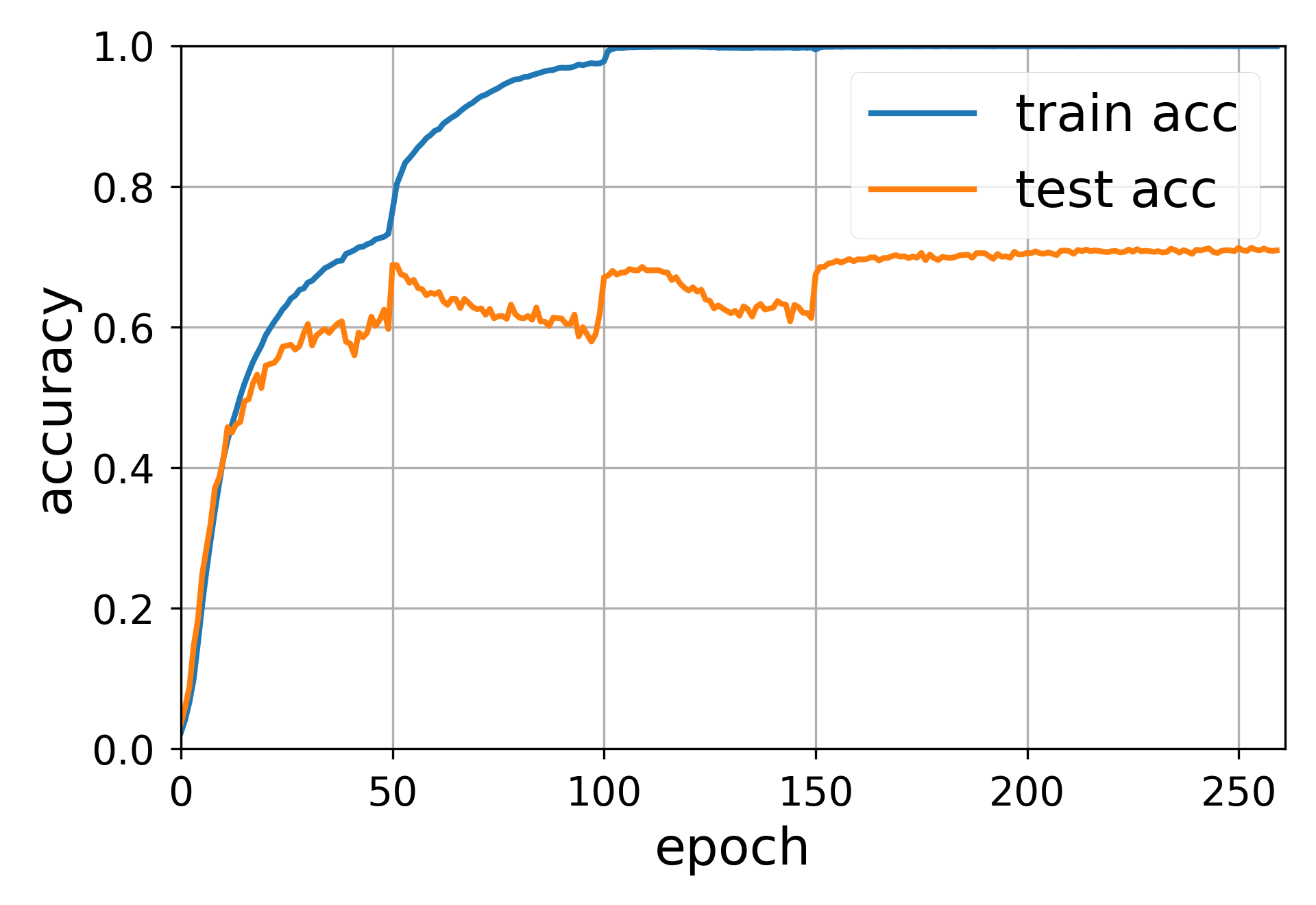}}\label{fig::resnet-50-CE-100}}
    \end{center}
    \caption{\footnotesize
        We apply ResNet-18, 34, 50 on noisy CIFAR-100 with both $\ell_1$-loss and cross-entropy loss (CE). The data generator and optimizer are the same as those in CIFAR-10 experiment. For each trial, we run $260$ epochs and decay the learning rate with factor $0.33$ for each $50$ epochs.}
    \label{fig::cifar-100}
\end{figure*}

\section{Proofs of Landscape Analysis}
\subsection{Proof of Theorem~\ref{prop::negative}}
\subsubsection*{Over-parameterized Regime}
\label{sec::proof-negative}
We first provide the proof for the 1-layer model. The values of $\cL\left(\theta^{\star}\right),\cL\left(\theta^{\star}+\Delta\theta\right)$ are provided as
\begin{equation}
    \cL\left(\theta^{\star}\right)=\frac{1}{m}\sum_{i\in \cS} |\err_i|, \quad \cL\left(\theta^{\star}+\Delta\theta\right)=\frac{1}{m}\sum_{i\in \bar \cS}\left|\inner{x_i}{\Delta\theta}\right|+\frac{1}{m}\sum_{i\in \cS}|\inner{x_i}{\Delta\theta}-\err_i|.
\end{equation}
Here $\cS\in [m]: \{1,2,\dots,m\}$ is the support of the noise vector $\err=[\err_1,\cdots,\err_m]^{\top}$, and $\bar \cS=[m]-\cS$. Hence, we have
\begin{equation}
    \cL\left(\theta^{\star}+\Delta\theta\right)-\cL\left(\theta^{\star}\right)=\frac{1}{m}\sum_{i\in \cS} \left(|\inner{x_i}{\Delta\theta}-\err_i|-|\err_i|\right)+\frac{1}{m}\sum_{i\in \bar \cS}\left|\inner{x_i}{\Delta\theta}\right|.
\end{equation}
Define the subspace
\begin{equation}
    U:=\left\{u\in \bR^d:\norm{u}_{\infty}\leq \gamma,\inner{u}{x_i}=0,\forall i\in \bar \cS, \text{ and }\inner{x_i}{u}\err_i>0, |\inner{x_i}{u}|\leq |\err_i|,\forall i\in \cS\right\}.
\end{equation}
Note that, since $m\leq 0.1d$, the number of constraints in $U$ is upper bounded by $2m\leq 0.2d$. Therefore, $\dim (U)\geq 0.8d$. On the other hand, since $U\subset \{u\in \bR^d:\norm{u}_{\infty}\leq \gamma\}$, we have
\begin{equation}
    \begin{aligned}
        \inf_{\norm{\Delta\theta}_{\infty}\leq\gamma}\{\cL\left(\theta^{\star}+\Delta\theta\right)-\cL\left(\theta^{\star}\right)\} & \leq \inf_{\Delta\theta\in U}\{\cL\left(\theta^{\star}+\Delta\theta\right)-\cL\left(\theta^{\star}\right)\}\leq \inf_{u\in U}-\frac{1}{m}\sum_{i\in \cS}|\inner{x_i}{u}|,
    \end{aligned}
\end{equation}
\begin{sloppypar}
    \noindent where the last inequality is due to our definition of the set $U$. Define the event $\cE:=\{\text{There are at least } p_0pm \text{ elements of } |\err| \text{ larger than } t_0\}$. Define $\cM_0 = \{i: |\err_i|\geq t_0\}$. Using the tail bound of binomial distribution, we have $\bP(E)\geq \frac{1}{4}$. Conditioned on the event $\cE$, we have $m_0=|\cM_0|\geq p_0pm$ and
\end{sloppypar}
\begin{equation}
    \inf_{u\in U}-\frac{1}{m}\sum_{i\in \cS}|\inner{x_i}{u}|\leq -\sup_{u\in U}\frac{1}{m}\sum_{i\in\cM_0}|\inner{x_i}{u}|.
\end{equation}
Hence, it suffices to provide a lower bound for $\sup_{u\in U}\frac{1}{m}\sum_{i\in\cM_0}|\inner{x_i}{u}|$. To this goal, we define $V:=\left\{u\in \bR^d:\norm{u}_{\infty}\leq \gamma,\inner{u}{x_i}=0,\forall i\in \bar \cS, \text{ and }\inner{x_i}{u}\err_i>0,\forall i\in \cS\right\}$. Hence, we have
\begin{equation}
    \begin{aligned}
        \sup_{u\in U}\frac{1}{m}\sum_{i\in\cM_0}|\inner{x_i}{u}| & \geq\sup_{u\in V}\frac{1}{m}\sum_{i\in\cM_0}|\inner{x_i}{u}|\wedge t_0                                                                                                                                                   \\
                                                                 & \geq \sup_{u\in V}\frac{1}{m}\sum_{i\in\cM_0}|\inner{x_i}{u}|\mathbbm{1}\left(\left|\inner{x_i}{u}\right|\leq t_0\right)                                                                                                 \\
                                                                 & \geq \underbrace{\sup_{u\in V}\frac{1}{m}\sum_{i\in\cM_0}|\inner{x_i}{u}|}_{(A)}-\underbrace{\sup_{u\in V}\frac{1}{m}\sum_{i\in\cM_0}|\inner{x_i}{u}|\mathbbm{1}\left(\left|\inner{x_i}{u}\right|\geq t_0\right)}_{(B)},
    \end{aligned}
\end{equation}
where the first inequality follows from the definition of $U,V$, and $\mathcal{M}_0$. With probability at least $1-\delta$ where $\log(1/\delta)\lesssim d$, we have
\begin{equation}
    \begin{aligned}
        (A) & \geq \sup_{u\in V}\frac{1}{m}\sum_{i\in\cM_0}\inner{x_i}{u}                                                                                                               \\
            & =\frac{m_0}{m}\sup_{u\in V}\frac{1}{m_0}\sum_{i\in\cM_0}\inner{x_i}{u}                                                                                                    \\
            & \stackrel{(a)}{\geq} \frac{m_0}{m}\left(\bE\left[\sup_{u\in V}\frac{1}{m_0}\sum_{i\in\cM_0}\inner{x_i}{u}\right]-\sqrt[]{d}\gamma\sqrt{\frac{\log(1/\delta)}{m_0}}\right) \\
            & \stackrel{(b)}{\gtrsim} \frac{m_0}{m}\sqrt[]{d}\gamma\left(\sqrt{\frac{d}{m_0}}-\sqrt{\frac{\log(1/\delta)}{m_0}}\right)                                                  \\
            & \gtrsim\sqrt{\frac{p_0p}{m}}d\gamma.
    \end{aligned}
\end{equation}
Here in (a), we used Theorem 3.1 in \cite{viens2007supremum}, and in (b) we used Sudakov's inequality. To provide a bound for (B), we first use Hölder's inequality to write
\begin{equation}
    \begin{aligned}
        (B) & \leq \sup_{u\in V}\frac{1}{m}\left(\sum_{i\in\cM_0}|\inner{x_i}{u}|\right)\sup_{i\in \cM_0}\mathbbm{1}\left(\left|\inner{x_i}{u}\right|\geq t_0\right) \\
            & \leq \sup_{u\in V}\frac{1}{m}\left(\sum_{i\in\cM_0}|\inner{x_i}{u}|\right)\sup_{i\in\cM_0}\mathbbm{1}\left(\norm{x_i}\gamma\geq t_0\right).
    \end{aligned}
\end{equation}
For the second part, we have
\begin{equation}
    \begin{aligned}
        \bP\left(\sup_{i\in\cM_0}\mathbbm{1}\left(\norm{x_i}\gamma\geq t_0\right)=1\right) & \leq \bP\left(\mathbbm{1}\left(\sup_{i\in\cM_0}\norm{x_i}\gamma\geq t_0\right)=1\right) \\
                                                                                           & =\bP\left(\sup_{i\in\cM_0}\norm{x_i}\gamma\geq t_0\right)                               \\
                                                                                           & \leq m_0\bP\left(\norm{x_i}\gamma\geq t_0\right)                                        \\
                                                                                           & \lesssim e^{\log(m_0)-C\left(\frac{t_0}{\gamma}\right)^2}                               \\
                                                                                           & \leq e^{-\Omega(d)}.
    \end{aligned}
\end{equation}
provided that $\gamma \lesssim \frac{t_0}{\sqrt{d}}\wedge 1$. Hence, we have
\begin{equation}
    (B)=0,\qquad \text{ with probability of } e^{-\Omega(d)}.
\end{equation}
Therefore, combining (A) and (B) with the choice of $\delta=\frac{1}{2}$, we have
\begin{equation}
    \inf_{\norm{\Delta\theta}_\infty\leq\gamma}\{\cL\left(\theta^{\star}+\Delta\theta\right)-\cL\left(\theta^{\star}\right)\}\lesssim -\sqrt{\frac{p_0p}{m}}d\gamma,
\end{equation}
with probability of at least $\frac{1}{16}$.

For general $N$-layer models, we consider the true solution $\mathbf{w}\in \cW$ with $w_1=\theta^{\star}$, and $w_2=\cdots w_N=\mathbf{1}$. Moreover, for $\Delta \mathbf{w}$ we choose $\Delta w_2=\cdots \Delta w_N=\mathbf{0}$. It is easy to verify that
\begin{equation}
    \cL(\mathbf{w})-\cL(\mathbf{w}+\Delta \mathbf{w})=\frac{1}{m}\sum_{i\in \cS} \left(|\inner{x_i}{\Delta w_1}-\err_i|-|\err_i|\right)+\frac{1}{m}\sum_{i\in \bar \cS}\left|\inner{x_i}{\Delta w_1}\right|.
\end{equation}
Therefore, an argument similar to 1-layer model can be used to write
\begin{equation}
    \inf_{\mathbf{w}\in\cW}\ \ \inf_{\mathbf{w}':\norm{\mathbf{w}-\mathbf{w}'}_\infty\leq \gamma}\left\{\cL(\mathbf{w})-\cL(\mathbf{w}')\right\}\leq \inf_{\norm{\Delta w_1}_\infty\leq \gamma}\left\{\cL(\mathbf{w})-\cL(\mathbf{w}+\Delta \mathbf{w})\right\}\lesssim  -\sqrt{\frac{p_0p}{m}}d\gamma,
\end{equation}
with probability of $\frac{1}{16}$.

\subsubsection*{Under-parameterized Regime}
Given $\mathbf{w}\in \cW$ and any $\Delta\mathbf{w}$, consider $\mathbf{w}'=\mathbf{w}+\Delta\mathbf{w}$ and define
\begin{equation}
    \Delta \theta = \sum_{i=1}^{N}(\theta^{\star})^{\frac{N-i}{N}}\odot \sum_{j_1,\cdots,j_i} \Delta w_{j_1}\odot\cdots\odot\Delta w_{j_i}.
\end{equation}
We have
\begin{equation}
    \begin{aligned}
        \cL(\mathbf{w})-\cL(\mathbf{w}+\Delta \mathbf{w}) & =\frac{1}{m}\sum_{i\in \cS} \left(|\inner{x_i}{\Delta \theta}-\err_i|-|\err_i|\right)+\frac{1}{m}\sum_{i\in \bar \cS}\left|\inner{x_i}{\Delta \theta}\right| \\
                                                          & \geq \frac{1}{m}\sum_{i\in \bar \cS}\left|\inner{x_i}{\Delta \theta}\right|-\frac{1}{m}\sum_{i\in \cS}\left|\inner{x_i}{\Delta \theta}\right|.
    \end{aligned}
\end{equation}
Hence, it suffices to show that, with probability of $1-e^{-\Omega(d)}$,
\begin{equation}
    \begin{aligned}
         & \inf_{\Delta \theta\in \R^d}\left\{\frac{1}{m}\sum_{i\in \bar \cS}\left|\inner{x_i}{\Delta \theta}\right|-\frac{1}{m}\sum_{i\in \cS}\left|\inner{x_i}{\Delta \theta}\right|\right\}\geq 0.
    \end{aligned}
\end{equation}
Note that the above inequality is invariant with respect to scaling. Hence, it suffices to show that it holds for arbitrary $\Delta \theta\in \bS^{d-1}$ where $\bS^{d-1}:=\{x\in \bR^d:\norm{x}=1\}$ is the standard sphere. Hence, it suffices to show
\begin{equation}
    \begin{aligned}
         & \inf_{\Delta \theta\in \bS^{d-1}}\left\{\frac{1}{m}\sum_{i\in \bar \cS}\left|\inner{x_i}{\Delta \theta}\right|-\frac{1}{m}\sum_{i\in \cS}\left|\inner{x_i}{\Delta \theta}\right|\right\} \\&\geq \inf_{\Delta \theta\in \bS^{d-1}}\left\{\frac{1}{m}\sum_{i\in \bar \cS}\left|\inner{x_i}{\Delta \theta}\right|\right\}-\sup_{\Delta \theta\in \bS^{d-1}}\left\{\frac{1}{m}\sum_{i\in \cS}\left|\inner{x_i}{\Delta \theta}\right|\right\}\geq 0.
    \end{aligned}
\end{equation}

For the first term, applying Lemma~\ref{lem::uniform-concentration-of-absolute-value}, we have that with probability at least $1-e^{-\Omega(d)}$
\begin{equation}
    \inf_{\Delta \theta\in \bS^{d-1}}\frac{1}{m}\sum_{i\in \bar \cS}\left|\inner{x_i}{\Delta \theta}\right|\geq \sqrt{\frac{2}{\pi}}(1-p)-\sqrt{\frac{(1-p)d}{m}}.
\end{equation}
Similarly, for the second part, with probability of at least $1-e^{-\Omega(d)}$, we have
\begin{equation}
    \sup_{\Delta \theta\in \bS^{d-1}}\frac{1}{m}\sum_{i\in \cS}\left|\inner{x_i}{\Delta \theta}\right|\leq \sqrt{\frac{2}{\pi}}p+\sqrt{\frac{pd}{m}}.
\end{equation}
Combining both parts, we have that with probability at least $1-e^{-\Omega(d)}$
\begin{equation}
    \begin{aligned}
        \inf_{\Delta \theta\in \bS^{d-1}}\left\{\frac{1}{m}\sum_{i\in \bar \cS}\left|\inner{x_i}{\Delta \theta}\right|\right\}-\sup_{\Delta \theta\in \bS^{d-1}}\left\{\frac{1}{m}\sum_{i\in \cS}\left|\inner{x_i}{\Delta \theta}\right|\right\} & \geq  \sqrt{\frac{2}{\pi}}\left(1-2p\right)-2\sqrt{\frac{d}{m}}\geq 0.
    \end{aligned}
\end{equation}
The last inequality follows from the fact that $m\gtrsim\frac{d}{(1-2p)^2}$. This completes the proof.$\hfill\square$

\subsection{Proof of Theorem~\ref{prop::positive}}
\label{sec::proof-positive}

Given any $\Delta\mathbf{w}=[\Delta w_1,\cdots,\Delta w_N]^{\top}$, the following equality holds for any point $\mathbf{w}=\mathbf{w}^{\star}+\Delta \mathbf{w}$ where $\mathbf{w}^{\star}=[\sqrt[N]{\theta^{\star}},\cdots,\sqrt[N]{\theta^{\star}}]^{\top}$:
\begin{equation}
    \label{eq::33}
    \begin{aligned}
        (\sqrt[N]{\theta^{\star}}+\Delta w_1)\odot\cdots \odot (\sqrt[N]{\theta^{\star}}+\Delta w_N)-\theta^{\star} & =\underbrace{\sum_{i=1}^{N-1}(\theta^{\star})^{\frac{N-i}{N}}\odot \sum_{j_1,\cdots,j_i} \Delta w_{j_1}\odot\cdots\odot\Delta w_{j_i}}_{:=\Delta \theta_1} \\&\quad+\underbrace{\Delta w_1\odot\cdots\odot\Delta w_N}_{:=\Delta \theta_2}.
    \end{aligned}
\end{equation}
Hence, we have
\begin{equation}
    \cL(\mathbf{w}^{\star})-\cL\left(\mathbf{w}^{\star}+\Delta\mathbf{w}\right)=\frac{1}{m}\sum_{i\in \cS} \left(|\inner{x_i}{\Delta \theta_1+\Delta \theta_2}-\err_i|-|\err_i|\right)+\frac{1}{m}\sum_{i\in \bar \cS}\left|\inner{x_i}{\Delta \theta_1+\Delta \theta_2}\right|.\label{eq_theta12}
\end{equation}

For simplicity, we denote $\Theta_1$ as the set of $\Delta\theta_1$ defined in~\eqref{eq_theta12} with $\norm{\Delta \mathbf{w}}_{\infty}\leq \gamma$. Similarly, $\Theta_2$ is the set of $\Delta\theta_2$ defined in~\eqref{eq_theta12} with $\norm{\Delta \mathbf{w}}_{\infty}\leq \gamma$.
\paragraph{Lower bound.}
To prove the lower bound, one can write
\begin{equation}
    \begin{aligned}
        \cL(\mathbf{w}^{\star})-\cL\left(\mathbf{w}^{\star}+\Delta\mathbf{w}\right) & \geq \frac{1}{m}\sum_{i\in \bar \cS}\left|\inner{x_i}{\Delta \theta_1+\Delta \theta_2}\right|-\frac{1}{m}\sum_{i\in \cS}\left|\inner{x_i}{\Delta \theta_1+\Delta \theta_2}\right|                         \\
                                                                                    & \geq \frac{1}{m}\sum_{i\in \bar \cS}\left|\inner{x_i}{\Delta \theta_1}\right|-\frac{1}{m}\sum_{i\in \cS}\left|\inner{x_i}{\Delta \theta_1}\right|-\frac{1}{m}\sum_{i=1}^{m}|\inner{x_i}{\Delta \theta_2}|
    \end{aligned}
\end{equation}
Hence, we have
\begin{equation}
    \begin{aligned}
        \inf_{\norm{\Delta \mathbf{w}}_{\infty}\leq \gamma}\left\{\cL(\mathbf{w}^{\star})-\cL\left(\mathbf{w}^{\star}+\Delta\mathbf{w}\right)\right\}\geq & \inf_{\Delta\theta_1\in\Theta_1}\left\{\frac{1}{m}\sum_{i\in \bar \cS}\left|\inner{x_i}{\Delta \theta_1}\right|-\frac{1}{m}\sum_{i\in \cS}\left|\inner{x_i}{\Delta \theta_1}\right|\right\} \\
                                                                                                                                                          & -\sup_{\Delta\theta_2\in\Theta_2}\left\{\frac{1}{m}\sum_{i=1}^{m}|\inner{x_i}{\Delta \theta_2}|\right\}.
    \end{aligned}
\end{equation}
First, we bound the term $\inf_{\Delta\theta_1\in\Theta_1}\left\{\frac{1}{m}\sum_{i\in \bar \cS}\left|\inner{x_i}{\Delta \theta_1}\right|-\frac{1}{m}\sum_{i\in \cS}\left|\inner{x_i}{\Delta \theta_1}\right|\right\}$. It is easy to see that the vector $\Delta \theta_1$ is $k$-sparse and has the same sparsity pattern as $\theta^\star$. Therefore, according to Lemma~\ref{lem::uniform-concentration-of-absolute-value}, there exist universal constants $C,c>0$ such that the following hold
\begin{equation}\label{eq_theta1}
    \bP\left(\sup_{\Delta \theta_1\in \Theta_1}\left|\frac{1}{m'\norm{\Delta \theta_1}}\sum_{i\in\bar\cS}|\inner{x_i}{\Delta \theta_1}|-\sqrt[]{\frac{2}{\pi}}\right|\geq C\sqrt[]{\frac{k}{m'}}+\delta\right)\leq e^{-cm'\delta^2},
\end{equation}
and
\begin{equation}\label{eq_theta2}
    \bP\left(\sup_{\Delta \theta_1\in \Theta_1}\left|\frac{1}{m''\norm{\Delta \theta_1}}\sum_{i\in \cS}|\inner{x_i}{\Delta \theta_1}|-\sqrt[]{\frac{2}{\pi}}\right|\geq C\sqrt[]{\frac{k}{m''}}+\delta\right)\leq e^{-cm''\delta^2}.
\end{equation}
Here $m'=(1-p)m$, and $m''=pm$. The inequality~\eqref{eq_theta1} implies that
\begin{equation}
    \begin{aligned}
        \bP\left(\frac{1}{m}\sum_{i\in \bar \cS}\left|\inner{x_i}{\Delta \theta_1}\right|\geq \norm{\Delta \theta_1}\left(\sqrt{\frac{2}{\pi}}(1-p)-C\sqrt{\frac{(1-p)k}{m}}-\delta_1\right),\forall \Delta \theta_1\in\Theta_1\right)\geq 1-e^{-\frac{cm\delta_1^2}{1-p}}.
    \end{aligned}
\end{equation}
Similarly, the inequality~\eqref{eq_theta2} leads to
\begin{equation}
    \bP\left(\frac{1}{m}\sum_{i\in \cS}\left|\inner{x_i}{\Delta \theta_1}\right|\leq \norm{\Delta \theta_1}\left(\sqrt{\frac{2}{\pi}}p+C\sqrt{\frac{pk}{m}}+\delta_2\right),\forall \Delta \theta_1\in\Theta_1\right)\geq 1-e^{-\frac{cm\delta_2^2}{p}}.
\end{equation}
Upon setting $\delta_1=\sqrt[]{\frac{(1-p)k}{m}}$ and $\delta_2=\sqrt[]{\frac{pk}{m}}$, with probability of $1-e^{-\Omega(k)}$, for all $\Delta \theta_1\in\Theta_1$, we have
\begin{equation}
    \frac{1}{m}\sum_{i\in \bar \cS}\left|\inner{x_i}{\Delta \theta_1}\right|-\frac{1}{m}\sum_{i\in \cS}\left|\inner{x_i}{\Delta \theta_1}\right|\geq \sqrt[]{\frac{2}{\pi}}(1-2p)-C'\sqrt[]{\frac{k}{m}}\geq 0.
\end{equation}
In the last inequality we used the assumption $m\gtrsim \frac{k}{(1-2p)^2}$. The above argument implies
\begin{equation}
    \inf_{\Delta \theta_1\in \Theta_1}\left\{\frac{1}{m}\sum_{i\in \bar \cS}\left|\inner{x_i}{\Delta \theta_1}\right|-\frac{1}{m}\sum_{i\in \cS}\left|\inner{x_i}{\Delta \theta_1}\right|\right\}\geq 0,
\end{equation}
\begin{sloppypar}
    \noindent with probability of at least $1-e^{-\Omega(k)}$.
    Now we turn to bound the second part $\sup_{\Delta\theta_2\in\Theta_2}\left\{\frac{1}{m}\sum_{i=1}^{m}|\inner{x_i}{\Delta \theta_2}|\right\}$. To this goal, we first apply Lemma~\ref{lem::uniform-concentration-of-absolute-value}, which leads to
\end{sloppypar}
\begin{equation}
    \bP\left(\sup_{\Delta\theta_2\in\Theta_2}\left|\frac{1}{m\norm{\Delta\theta_2}}\sum_{i=1}^{m}|\inner{x_i}{\Delta \theta_2}|-\sqrt[]{\frac{2}{\pi}}\right|\geq C\sqrt[]{\frac{d}{m}}+\delta\right)\leq e^{-cm\delta^2}.
\end{equation}
Therefore, upon setting $\delta=\sqrt[]{\frac{d}{m}}$, with probability of $1-e^{-\Omega(d)}$, we have
\begin{equation}
    \begin{aligned}
        \sup_{\Delta \theta_2\in \Theta_2}\left\{\frac{1}{m}\sum_{i=1}^{m}|\inner{x_i}{\Delta \theta_2}|\right\} & \leq \sup_{\Delta \theta_2\in \Theta_2}\norm{\Delta\theta_2}\left(\sqrt{\frac{2}{\pi}}+(C+1)\sqrt{\frac{d}{m}}\right)      \\
                                                                                                                 & \lesssim \sqrt{\frac{d}{m}}\sup_{\norm{\Delta \mathbf{w}}_{\infty}\leq \gamma}\norm{\Delta w_1\odot\cdots\odot\Delta w_N}  \\
                                                                                                                 & =\sqrt{\frac{d}{m}}\sup_{|\Delta w_i[j]|\leq \gamma} \sqrt[]{\sum_{j\in [d]} \left(\prod_{i\in [N]}\Delta w_i[j]\right)^2} \\
                                                                                                                 & =\frac{d}{\sqrt[]{m}}\gamma^N.
    \end{aligned}
\end{equation}
Here $\Delta w_i[j]$ is the $j$-th element of $\Delta w_i$. Therefore, we conclude that
\begin{equation}
    \inf_{\norm{\Delta \mathbf{w}}_\infty\leq \gamma}\left\{\cL(\mathbf{w}^{\star})-\cL\left(\mathbf{w}^{\star}+\Delta\mathbf{w}\right)\right\}\gtrsim -\frac{d}{\sqrt[]{m}}\gamma^N,
\end{equation}
with probability of $1-e^{-\Omega(k)}$, thereby completing the proof of the lower bound.
\paragraph{Upper bound.}
To this goal, we first define a restricted set of perturbation vectors
\begin{equation}
    \cV:=\left\{\mathbf{v}:\norm{\mathbf{v}}_{\infty}\leq \gamma, v_{i}[j]=0, \forall i\in [N], j\in \supp(\theta^{\star})\right\},
\end{equation}
where $\supp(\theta^{\star})$ is the support of $\theta^{\star}$. Based on this definition, we have
\begin{equation}
    \begin{aligned}
        \inf_{\norm{\Delta \mathbf{w}}_{\infty}\leq \gamma}\left\{\cL(\mathbf{w}^{\star})-\cL\left(\mathbf{w}^{\star}+\Delta\mathbf{w}\right)\right\} & \leq \inf_{\Delta \mathbf{w}\in \cV}\left\{\cL(\mathbf{w}^{\star})-\cL\left(\mathbf{w}^{\star}+\Delta\mathbf{w}\right)\right\}                                                                                  \\
                                                                                                                                                      & =\inf_{\Delta \mathbf{w}\in \cV}\left\{\frac{1}{m}\sum_{i\in \cS} \left(|\inner{x_i}{\Delta \theta_2}-\err_i|-|\err_i|\right)+\frac{1}{m}\sum_{i\in \bar \cS}\left|\inner{x_i}{\Delta \theta_2}\right|\right\}.
    \end{aligned}
\end{equation}
where $\Delta\theta_2=\Delta w_1\odot\cdots\odot\Delta w_N$ is the same as before. Note that for any $\norm{\Delta\mathbf{w}}_{\infty}\leq\gamma$, we have $\norm{\Delta \theta_2}\leq \sqrt[]{d}\gamma^N$. Moreover, this bound is attainable when $\Delta w_i\equiv [\pm\gamma,\cdots,\pm\gamma]^{\top}$. Hence, we have
\begin{equation}
    \begin{aligned}
         & \inf_{\Delta \mathbf{w}\in \cV}\left\{\frac{1}{m}\sum_{i\in \cS} \left(|\inner{x_i}{\Delta \theta_2}-\err_i|-|\err_i|\right)+\frac{1}{m}\sum_{i\in \bar \cS}\left|\inner{x_i}{\Delta \theta_2}\right|\right\}                            \\
         & \leq \inf_{\norm{\Delta \theta_2}\leq \sqrt[]{d}\gamma^N}\left\{\frac{1}{m}\sum_{i\in \cS} \left(|\inner{x_i}{\Delta \theta_2}-\err_i|-|\err_i|\right)+\frac{1}{m}\sum_{i\in \bar \cS}\left|\inner{x_i}{\Delta \theta_2}\right|\right\}.
    \end{aligned}
\end{equation}
An argument similar to the proof of Theorem~\ref{prop::negative} can be used to show that, with probability of at least $1/16$, we have
\begin{equation}
    \inf_{\norm{\Delta \theta_2}\leq \sqrt[]{d}\gamma^N}\left\{\frac{1}{m}\sum_{i\in S} \left(|\inner{x_i}{\Delta \theta_2}-\err_i|-|\err_i|\right)+\frac{1}{m}\sum_{i\in \bar S}\left|\inner{x_i}{\Delta \theta_2}\right|\right\}\lesssim -\sqrt{\frac{p_0p(d-k)}{m}}\sqrt[]{d}\gamma^N.
\end{equation}
Recalling that $k\ll d$, we have with probability of $1/16$
\begin{equation}
    \inf_{\norm{\Delta \mathbf{w}}_\infty\leq \gamma}\left\{\cL(\mathbf{w}^{\star})-\cL\left(\mathbf{w}^{\star}+\Delta\mathbf{w}\right)\right\}\lesssim -\sqrt{p_0p}\frac{d}{\sqrt[]{m}}\gamma^N.
\end{equation}
This completes the proof.$\hfill\square$


\section{Proofs of Convergence Analysis}
\subsection{Proof of Theorem~\ref{thm::2-layer}}
\label{sec::proof-2-layer}
For simplicity of notation, we denote $u = w_1$ and $v = w_2$. Moreover, without loss of generality, we assume that the elements of $\theta^\star$ are arranged in descending order, i.e., $\theta^{\star}_1\geq \cdots \geq\theta^{\star}_k>\theta^{\star}_{k+1}=\cdots =\theta^{\star}_d=0$, and the initial point satisfies $u_i,v_i=\Theta\left(\sqrt{\alpha}\right), \forall i\in [d]$. Moreover, for $v\in \R^d$, we define $v_{:i}=[v_1,\cdots, v_i]^{\top}$ and $v_{{i:}}=[v_i,\cdots,v_d]^{\top}$. For short, we denote $v_{-i}=v_{i+1:}$. Finally, we define $\kappa=\theta^{\star}_1/\theta^{\star}_k$ as the condition number. Moreover, without loss of generality, we assume that $\theta_1^{\star}\geq 1\geq \theta_k^{\star}$.

First, according to Proposition~\ref{prop:sign-RIP}, the sub-differential of $\cL(u,v)$ is uniformly concentrated around its population gradient. In particular, with probability at least $1-Ce^{-cm\delta^2}$, we have
\begin{equation}
    u_{i,t+1}=u_{i,t}+\eta \frac{\theta^{\star}_i-u_{i,t}v_{i,t}}{\norm{u_t\odot v_t-\theta^{\star}}}v_{i,t}+\eta\delta_iv_{i,t}, \text{ and } \left|\delta_i\right|\leq \delta, \forall i\in [d],
\end{equation}
\begin{equation}
    v_{i,t+1}=v_{i,t}+\eta \frac{\theta^{\star}_i-u_{i,t}v_{i,t}}{\norm{u_t\odot v_t-\theta^{\star}}}u_{i,t}+\eta\delta_iu_{i,t}, \text{ and } \left|\delta_i\right|\leq \delta, \forall i\in [d].
\end{equation}
Hence, we have
\begin{equation}
    u_{i,t+1}v_{i,t+1}=u_{i,t}v_{i,t}+\eta\left( \frac{\theta^{\star}_i-u_{i,t}v_{i,t}}{\norm{u_t\odot v_t-\theta^{\star}}}+\delta_i\right)\left(u_{i,t}^2+v_{i,t}^2\right)+\eta^2\left( \frac{\theta^{\star}_i-u_{i,t}v_{i,t}}{\norm{u_t\odot v_t-\theta^{\star}}}+\delta_i\right)^2u_{i,t}v_{i,t}.
    \label{eq::2-layer-update}
\end{equation}
Moreover, we have
\begin{equation}
    u_{i,t+1}^2+v_{i,t+1}^2=\left(u_{i,t}^2+v_{i,t}^2\right)\left(1+\eta^2\left(\frac{\left(\theta^{\star}_i-u_{i,t}v_{i,t}\right)}{\norm{u_t\odot v_t-\theta^{\star}}}+\delta_i\right)^2\right)+4\eta\left(\frac{\theta^{\star}_i-u_{i,t}v_{i,t}}{\norm{u_t\odot v_t-\theta^{\star}}}+\delta_i\right)u_{i,t}v_{i,t}.
    \label{eq::square}
\end{equation}

\subsubsection*{Signal Dynamics}
We first study the behavior of the signal term $S_t=u_{:k,t}v_{:k,t}$ for the first $k$ components of the model $u_{t}\odot v_{t}$. We divide the dynamics into $k+1$ stages. In the first $k$ stages, each component $u_{i,t}v_{i,t}$ converges to $\theta_i^{\star}$ sequentially. Once all the components are close to the ground truth, the distance between signal term and the ground truth $\norm{S_t-\theta^{\star}_{:k}}$ will further decrease to $\cO\left(\sqrt{d^2m}\alpha^2\vee \eta\theta^{\star}_1\right)$.
\paragraph{Stage 1:} In this stage, the first component $u_1v_1$ grows to $\theta_1-\delta \norm{\theta^{\star}}$ within $\Theta\left(\frac{\norm{\theta^{\star}}}{\eta\theta^{\star}_1}\log\left(\frac{1}{\alpha}\right)\right)$ iterations. At the initial point, we have $u_{1,0}v_{1,0}=\Theta(\alpha)$. For iteration $t+1$, according to \eqref{eq::2-layer-update}, we have
\begin{equation}
    \begin{aligned}
        u_{1,t+1}v_{1,t+1} & \geq u_{1,t}v_{1,t}+\eta\left( \frac{\theta^{\star}_1-u_{1,t}v_{1,t}}{\norm{u_t\odot v_t-\theta^{\star}}}+\delta_1\right)\left(u_{1,t}^2+v_{1,t}^2\right) \\
                           & \geq u_{1,t}v_{1,t}+2\eta\left( \frac{\theta^{\star}_1-u_{1,t}v_{1,t}}{\norm{u_t\odot v_t-\theta^{\star}}}+\delta_1\right)u_{1,t}v_{1,t}.
        \label{eq::85}
    \end{aligned}
\end{equation}
We further divide our analysis into two substages. In the first substage, we have $u_1v_1\leq \theta^{\star}_1/2$. Note that $\norm{u_t\odot v_t-\theta^{\star}}=\cO(\norm{\theta^{\star}})$ and $|\delta_1|\leq \delta \lesssim \frac{1}{\kappa}$. Hence, \eqref{eq::85} can be further simplified as
\begin{equation}
    u_{1,t+1}v_{1,t+1}\geq \left(1+\frac{1}{4}\frac{\eta\theta^{\star}_1}{\norm{\theta^{\star}}}\right)u_{1,t}v_{1,t}.
\end{equation}
Therefore, this stage ends within $\cO\left(\frac{\norm{\theta^{\star}}}{\eta\theta^{\star}_1}\log\left(\frac{1}{\alpha}\right)\right)$ iterations. In the second substage, we have $u_1v_1\geq \theta^{\star}_1/2$. Upon defining $x_t=\theta^{\star}_1-u_{1,t}v_{1,t}$, one can write
\begin{equation}
    \begin{aligned}
        x_{t+1} & \leq \left(1-2\eta\frac{u_{1,t}v_{1,t}}{\norm{u_t\odot v_t-\theta^{\star}}}\right)x_t-2\eta\delta_1u_t\odot v_t \\
                & \leq\left(1-\frac{\eta\theta^{\star}_1}{\norm{\theta^{\star}}}\right)x_t+\eta\delta \theta^{\star}_1.
    \end{aligned}
\end{equation}
Hence, within additional $\cO\left(\norm{\theta^{\star}}/(\eta\theta^{\star}_1)\right)$ iterations,
we have $u_{1,T_1}v_{1,T_1}=\theta_1^{\star}\pm \delta \norm{\theta^{\star}}$.
Overall, within $T_1=\cO\left(\frac{\norm{\theta^{\star}}}{\eta\theta^{\star}_1}\log\left(\frac{1}{\alpha}\right)\right)$ iterations, we have $u_{1,T_1}v_{1,T_1}=\theta_1^{\star}\pm \delta \norm{\theta^{\star}}$. Now, we turn to show the lower bound on $T_1$ by analyzing the trajectory of $u_{1,t}^2+v_{1,t}^2$. Due to \eqref{eq::square}, when $u_{1,t}^2+v_{1,t}^2\leq \frac{\theta^{\star}_1}{2}$, we have
\begin{equation}
    u_{1,t+1}^2+v_{1,t+1}^2\leq \left(u_{1,t}^2+v_{1,t}^2\right)\left(1+10\frac{\eta\theta^{\star}_1}{\norm{\theta^{\star}}}\right).
\end{equation}
Hence, at least $\Omega\left(\frac{\norm{\theta^{\star}}}{\eta\theta^{\star}_1}\log\left(\frac{1}{\alpha}\right)\right)$ iterations are needed for $u_{1,t}^2+v_{1,t}^2$ to be larger than $\frac{\theta^{\star}_1}{2}$. Since $u_{1,t}^2+v_{1,t}^2\geq u_{1,t}v_{1,t}$, we immediately obtain $T_1=\Omega\left(\frac{\norm{\theta^{\star}}}{\eta\theta^{\star}_1}\log\left(\frac{1}{\alpha}\right)\right)$.

\paragraph{Stage 2 to Stage \textit{k}:} In the next $k-1$ stages, each component $u_{i}v_i$ will converge to $\theta^{\star}_i\pm \delta \norm{\theta^{\star}}$ sequentially. To show this, we use an inductive argument. In each stage $i$, we assume that the first $i-1$ components have already converged close to $\theta^{\star}_{j}, \forall j\in [i-1]$. Hence, we have $\norm{u_t\odot v_t-\theta^{\star}}=\Theta\left(\norm{\theta^{\star}_{-(i-1)}}\right)$. Repeating the procedure in Stage 1, we can show that, at stage $i$, $T_i=\cO\left(\frac{\norm{\theta^{\star}_{-(i-1)}}}{\eta\theta^{\star}_i}\log\left(\frac{1}{\alpha}\right)\right)$ iterations are needed for $u_{i}v_i$ to converge to $\theta_i^{\star}\pm \delta \norm{\theta^{\star}}$. Overall, after $T=T_1+\cdots T_k=\cO\left(\frac{k^{3/2}}{\eta}\log\left(\frac{1}{\alpha}\right)\right)$ iterations, we have
\begin{equation}
    \norm{u_{:k,T}v_{:k,T}-\theta^{\star}_{:k,T}}\lesssim \sqrt{k}\delta \norm{\theta^{\star}}.
\end{equation}

\paragraph{Stage $k+1$:} In the final stage, the signal term will quickly decrease to $\cO\left(\sqrt{d^2m}\alpha^{1-\Theta(\delta)}\vee \eta\theta^{\star}_1\right)$ within $T_{k+1}=\cO\left(\frac{\delta\norm{\theta^{\star}}}{\eta\theta^{\star}_k}\right)$ iterations. To show this, we write
\begin{equation}
    \theta^{\star}_{:k}-S_{t+1}=\theta^{\star}_{:k}-S_t-\eta \frac{u_{:k,t}^2+v_{:k,t}^2}{\norm{u_t\odot v_t-\theta^{\star}}}\odot \left(\theta^{\star}_{:k}-S_t\right) -\eta \bm{\delta} \left(u_t^2+v_t^2\right)+\eta^2\left(\frac{\theta^{\star}_{:k}-S_t}{\norm{u_t\odot v_t-\theta^{\star}}}+\bm{\delta}\right)^2\odot S_t.
\end{equation}
Here we denote $\bm{\delta}=[\delta_1,\cdots,\delta_k]^{\top}$.
Note that $\norm{u_t\odot v_t-\theta^{\star}}\leq \norm{\theta^{\star}_{:k}-S_t}+\norm{E_t}$. Moreover, based on our assumption, we have $\norm{E_t}\lesssim \sqrt{d}\alpha^{1-\Theta\left(\delta\right)}$. Finally, the balanced property implies that $|u_{i,t}-v_{i,t}|=\cO\left(\alpha^{1-\Theta(\delta)}\right)$ (this will be proven later). Hence, we have
\begin{equation}
    \begin{aligned}
        \norm{\theta^{\star}_{:k}-S_{t+1}} & \leq \norm{\left(\theta^{\star}_{:k}-S_t\right)\odot\left(\bm{1}-\eta \frac{u_t^2+v_t^2}{\norm{u_t\odot v_t-\theta^{\star}}}+\eta^2\frac{\left(\theta^{\star}_{:k}-S_{t}\right)\odot S_t}{\norm{u_t\odot v_t-\theta^{\star}}^2}\right)}+4\eta\delta \norm{\theta^{\star}} \\
                                           & \leq \norm{\theta^{\star}_{:k}-S_t}\left(1-\frac{\eta}{2} \frac{\theta^{\star}_k}{\norm{\theta^{\star}_{:k}-S_t}+\norm{E_t}}\right)+4\eta\delta \norm{\theta^{\star}}                                                                                                     \\
                                           & \leq \norm{\theta^{\star}_{:k}-S_t}-0.25\eta\theta^{\star}_k+4\eta\delta\norm{\theta^{\star}}                                                                                                                                                                             \\
                                           & \leq \norm{\theta^{\star}_{:k}-S_t}-0.1\eta\theta^{\star}_k.
    \end{aligned}
\end{equation}
Here, we used the fact that $\delta \lesssim \frac{1}{\kappa}$.
On the other hand, we have
\begin{equation}
    \begin{aligned}
        \norm{S_{t+1}-S_t} & \leq \frac{\eta}{\norm{u_t\odot v_t-\theta^{\star}}}\norm{\left(u_t^2+v_t^2\right)\odot \left(\theta^{\star}_{:k}-S_t\right)}+4\eta\delta \norm{\theta^{\star}} \\
                           & \stackrel{(a)}{\leq} 2\eta \theta^{\star}_1 +4\eta\delta\norm{\theta^{\star}}                                                                                   \\
                           & \leq 3\eta\theta^{\star}_1,
    \end{aligned}
\end{equation}
where in (a) we used Lemma~\ref{lem::hadamard-product}. The above inequality indicates that the signal propagation in each step is upper bounded by $\cO\left(\eta\theta^{\star}_1\right)$. Hence, we conclude that within $T_{k+1}=\cO\left(\frac{\delta\norm{\theta^{\star}}}{\eta\theta^{\star}_k}\right)$ iterations, we have $\norm{\theta^{\star}_{:k}-S_t}\lesssim \sqrt{d^2m}\alpha^{1-\Theta\left(\delta\right)}\vee \eta\theta_1^{\star}$. Since $\delta\lesssim \frac{1}{\kappa}$, the total iteration complexity is upper bounded by $T'=T_1+\cdots T_{k+1}=\cO\left(\frac{k^{3/2}}{\eta}\log\left(\frac{1}{\alpha}\right)\right)$.

\subsubsection*{Residual Dynamics}
Now, we analyze the residual dynamics. Instead of analyzing the dynamics of $u_{i,t}v_{i,t}$, we analyze its surrogate $u_{i,t}^2+v_{i,t}^2$. Based on \eqref{eq::square}, we can naturally bound it as follows
\begin{equation}
    \begin{aligned}
        u_{i,t+1}^2+v_{i,t+1}^2\leq u_{i,t}^2+v_{i,t}^2+6\eta\delta u_{i,t}v_{i,t}\leq \left(1+3\eta\delta\right)\left(u_{i,t}^2+v_{i,t}^2\right).
    \end{aligned}
\end{equation}
Therefore, during the training process, we can bound the residual term as
\begin{equation}
    u_{i,t}^2+v_{i,t}^2\lesssim \alpha\left(1+\eta\delta\right)^{\cO\left(\frac{k^{3/2}}{\eta}\log\left(\frac{1}{\alpha}\right)\right)}\lesssim \alpha^{1-\cO\left(k^{3/2}\delta\right)}.
\end{equation}
Hence, we have
\begin{equation}
    \norm{E_t}\leq \left(\sum_{i=k+1}^{d}u_{i,t}^2+v_{i,t}^2\right)^{1/2}\lesssim \sqrt{d}\alpha^{1-\cO\left(k^{3/2}\delta\right)}.
\end{equation}

Therefore, we conclude that within $\bar T=\cO\left(\frac{k^{3/2}}{\eta}\log\left(\frac{1}{\alpha}\right)\right)$ iterations, we have
\begin{equation}
    \norm{u_{\bar T}\odot v_{\bar T}-\theta^{\star}}\leq \norm{\theta^{\star}_{:k}-S_{\bar T}}+\norm{E_{\bar T}}\lesssim \sqrt{d^2m}\alpha^{1-\cO\left(\delta\right)}\vee \eta\theta_1^{\star}.
\end{equation}
Therefore, with probability at least $1-e^{-C\log^2(m)\log(d)\log\left(\norm{\theta^{\star}}/\alpha\right)}$, we have
\begin{equation}
    \norm{u_{\bar T}\odot v_{\bar T}-\theta^{\star}}\lesssim \sqrt{d^2m}\alpha^{1-\tilde\Theta\left(\frac{k^2}{\sqrt{m(1-p)^2}}\right)}\vee \eta\theta_1^{\star}.
\end{equation}

\subsubsection*{Long Escape Time}
Based on the above analysis, it can be seen that, after $\bar{T}$ iterations, the residual term is the dominant term, which may cause the algorithm to diverge, as captured by Figure~\ref{fig::simulation}. We now show that the residual term will not diverge within $T'=\sqrt{\frac{m(1-p)^2}{k}}\bar T$. To this goal, recall that for $\forall k+1\leq i\leq d$, we have
\begin{equation}
    u_{i,t+1}^2+v_{i,t+1}^2\leq \left(1+3\eta\delta\right)\left(u_{i,t}^2+v_{i,t}^2\right).
\end{equation}
Hence, for $t\leq \frac{1}{6\eta\delta}\log\left(\frac{1}{\alpha}\right)$, we have
\begin{equation}
    u_{i,t}^2+v_{i,t}^2\lesssim \alpha \left(1+3\eta\delta\right)^{\frac{1}{6\eta\delta}\log\left(\frac{1}{\alpha}\right)}\leq \alpha e^{0.5 \log\left(\frac{1}{\alpha}\right)}=\sqrt{\alpha}.
\end{equation}
The proof is completed by noticing that $m=\tilde\Omega\left(\frac{k}{(1-p)^2\delta^2}\right)$.

\subsubsection*{Balanced Property}
To prove the balanced property, we directly calculate the dynamic of the difference $u_{i,t}-v_{i,t}$. One can write
\begin{equation}
    u_{i,t+1}-v_{i,t+1}=\left(u_{i,t}-v_{i,t}\right)\left(1-\eta\frac{\theta^{\star}_i-u_{i,t}v_{i,t}}{\norm{u_t\odot v_t-\theta^{\star}}}-\eta\delta_i\right).
\end{equation}
Since $\theta^{\star}_i-u_{i,t}v_{i,t}\geq 0$, we have
\begin{equation}
    \left|u_{i,t+1}-v_{i,t+1}\right|\leq \left|u_{i,t}-v_{i,t}\right|\left(1+\eta\delta\right),
\end{equation}
for $0\leq i\leq k$. We conclude that
\begin{equation}
    \left|u_{i,t}-v_{i,t}\right|\leq \sqrt{\alpha}\left(1+\eta\delta\right)^t\lesssim \alpha^{0.5-\Theta(k^{3/2}\delta)}.
\end{equation}
for $\forall t\lesssim\frac{k^{3/2}}{\eta}\log\left(\frac{1}{\alpha}\right)$. On the other hand, for $i\geq k$, we can write
\begin{equation}
    \left|u_{i,t}-v_{i,t}\right|\leq \sqrt{u_{i,t}^2+v_{i,t}^2}\lesssim \alpha^{0.5-\Theta(k^{3/2}\delta)}.
\end{equation}
The proof is completed by noticing that $m=\tilde\Omega\left(\frac{k}{(1-p)^2\delta^2}\right)$.
\subsubsection*{Convergence in Under-parameterized Regime}
In this section, we study the under-parameterized regime, where we assume that $m=\tilde\Omega(\frac{d}{(1-p)^2})$. The analysis of the signal term is the same as the over-parameterized regime and hence omitted for brevity. Here we only analyze the residual dynamic. One can write
\begin{equation}
    E_{t+1}=E_t\left(\bm{1}-\eta\frac{u_{k+1:,t}^2+v_{k+1:,t}^2}{\norm{u_t\odot v_t-\theta^{\star}}}\right)+\eta\delta_{k+1:}\left(u_{k+1:,t}^2+v_{k+1:,t}^2\right)+\eta^2\left( \frac{-E_t}{\norm{u_t\odot v_t-\theta^{\star}}}+\delta_{k+1:}\right)^2E_t.
\end{equation}
When the residual term becomes the dominant term, i.e., $\norm{\theta^{\star}_{:k}-S_t}\leq \norm{E_t}$, we have the simplified dynamic
\begin{equation}
    \begin{aligned}
        \norm{E_{t+1}} & \leq \norm{E_t\left(\bm{1}-0.5\eta\frac{u_{k+1:,t}^2+v_{k+1:,t}^2}{\norm{u_t\odot v_t-\theta^{\star}}}\right)+\eta\delta_{k+1:}\left(u_{k+1:,t}^2+v_{k+1:,t}^2\right)}+2\eta^2\left(\norm{E_t}^3+\delta^2\norm{E_t}\right) \\
                       & \stackrel{(a)}{\leq} \norm{E_t\left(\bm{1}-\frac{\eta E_t}{\norm{E_t}}\right)}+4\eta\delta \norm{E_t}                                                                                                                      \\
                       & \leq \left(1-\frac{\eta}{\sqrt[]{d}}\right)\norm{E_t}+4\eta\delta \norm{E_t}                                                                                                                                               \\
                       & \stackrel{(b)}{\leq} \left(1-0.5\frac{\eta}{\sqrt[]{d}}\right)\norm{E_t}.
    \end{aligned}
\end{equation}
Here in (a) we used the balanced property, which results in $u_{k+1:,t}^2+v_{k+1:,t}^2\asymp 2E_t$. Moreover, (b) is implied by the fact that $m\gtrsim \frac{d}{(1-p)^2}$, which in turn implies $\delta\lesssim \frac{1}{\sqrt[]{d}}$. Hence, we have
\begin{equation}
    \norm{u_{t+1}\odot v_{t+1}-\theta^{\star}}\leq \left(1-\Omega\left(\frac{\eta}{\sqrt[]{d}}\right)\right)\norm{u_{t}\odot v_{t}-\theta^{\star}}.
\end{equation}
Then, for $t\geq \bar T$, we have
\begin{equation}
    \norm{u_{t}\odot v_{t}-\theta^{\star}}\lesssim \sqrt{d^2m}\alpha^{1-\tilde\Theta\left(\frac{k^2}{\sqrt{m(1-p)^2}}\right)}\left(1-\Omega\left(\frac{\eta}{\sqrt[]{d}}\right)\right)^{t-\bar T}\vee\eta\theta^{\star}_1.
\end{equation}

\subsection{Proof of Theorem~\ref{thm::N-layer}}
\label{sec::proof-convergence-higher-layer}
The proof of $N$-layer model is similar to that of 2-layer model. First, we study the signal dynamics, showing that the first $k$ components $\prod w_{i,j}^{(t)}$ converge to $\theta^{\star}_j$ sequentially for $1\leq j\leq k$. We also prove that the residual term remains small along the optimization trajectory.
Based on the SubGM update rule, one can write
\begin{equation}
    w_{i}^{(t+1)}=w_{i}^{(t)}+\eta\frac{\theta^{\star}-\prod w_j^{(t)}}{\norm{\theta^{\star}-\prod w_j^{(t)}}}\prod_{j\neq i}w_j^{(t)}+\eta\bm{\delta}\prod_{j\neq i}w_j^{(t)}, \text{ and } \norm{\bm{\delta}}_{\infty}\leq \delta, \forall i\in [N].
\end{equation}

\subsubsection*{Signal Dynamics}
\paragraph{Stage 1:} In this stage, we show that $\prod w_{i,1}$ will converge to $\theta^{\star}$ within $T_1=\Theta\left(\frac{\norm{\theta^{\star}}}{N\eta\theta^{\star}_1}\alpha^{-\frac{N-2}{N}}\right)$ iterations. We first prove the upper bound. According to the update rule, we have
\begin{equation}
    \begin{aligned}
        \prod w_{i,1}^{(t+1)} & = \prod w_{i,1}^{(t)}+\sum_{i=1}^{N}\eta\left(\frac{\theta^{\star}_1-\prod w_{j,1}^{(t)}}{\norm{\theta^{\star}-\prod w_j^{(t)}}}+\delta_{1}\right)\left(\prod_{j\neq i}w_{j,1}^{(t)}\right)^2+\text{higher order terms of } \eta \\
                              & \geq \prod w_{i,1}^{(t)}+\sum_{i=1}^{N}\eta\left(\frac{\theta^{\star}_1-\prod w_{j,1}^{(t)}}{\norm{\theta^{\star}-\prod w_j^{(t)}}}+\delta_{1}\right)\left(\prod_{j\neq i}w_{j,1}^{(t)}\right)^2                                 \\
                              & \geq \prod w_{i,1}^{(t)}+N\eta \left(\frac{\theta^{\star}_1-\prod w_{j,1}^{(t)}}{\norm{\theta^{\star}-\prod w_j^{(t)}}}-\delta\right)\left(\prod w_{i,1}^{(t)}\right)^{\frac{2(N-1)}{N}}.
    \end{aligned}
\end{equation}
Here we use the fact that $\prod w_{i,1}^{(t)}\leq \theta^{\star}_1$ and $\eta\lesssim \frac{1}{N}\frac{1}{\kappa}^{\frac{N-2}{N}}$ so that we can drop the higher order terms. For brevity, we only show how we can drop the $2$-th order term of $\eta$. The proof of the higher order terms is similar. One can write
\begin{equation}
    \begin{aligned}
        \eta^2\sum_{i\neq j}\left(\prod_{k\neq i}w_{k,1}^{(t)}\prod_{k\neq j}w_{k,1}^{(t)}\prod_{k\neq i,j}w_{k,1}^{(t)}\right) & \stackrel{(a)}{\leq} \eta^2\left(\theta_1^{\star}\right)^{\frac{N-2}{N}}\sum_{i\neq j}\left(\prod_{k\neq i}w_{k,1}^{(t)}\prod_{k\neq j}w_{k,1}^{(t)}\right) \\
                                                                                                                                & \stackrel{(b)}{\leq}(N-1)\eta\left(\theta_1^{\star}\right)^{\frac{N-2}{N}}\eta\sum_{i=1}^{m}\left(\prod_{j\neq i}w_{j,1}^{(t)}\right)^2                     \\
                                                                                                                                & \stackrel{(c)}{\lesssim} \eta\sum_{i=1}^{m}\left(\prod_{j\neq i}w_{j,1}^{(t)}\right)^2.
    \end{aligned}
\end{equation}
Here in (a) we use the balanced property and the fact that $\prod w_{i,1}^{(t)}\leq \theta^{\star}_1$. Moreover, in (b), we use the rearrangement Inequality. Finally in (c) we use the assumption that $\eta\lesssim {N}^{-1}{\kappa}^{-\frac{N-2}{N}}$.
For simplicity, we denote $x_t=\prod w_{i,1}^{(t)}$. Note that $\norm{\theta^{\star}-\prod w_j^{(t)}}\leq \norm{\theta^{\star}}$. Hence, the dynamic can be simplified as
\begin{equation}
    x_{t+1}\geq x_t+\frac{N\eta}{\norm{\theta^{\star}}}\left(\theta_1^{\star}-\delta\norm{\theta^{\star}}-x_t\right)x_t^{\frac{2(N-1)}{N}}.
\end{equation}
We next show that $x_T\geq \theta^{\star}_1-2\delta\norm{\theta^{\star}}$ within $T_1=\Theta\left(\frac{\norm{\theta^{\star}}}{N\eta\theta^{\star}_1}\alpha^{-\frac{N-2}{N}}\right)$ iterations provided that $x_0=\Theta(\alpha)$. To this goal, we divide our analysis into two substages.
\begin{itemize}
    \item {$x_t\leq \frac{\theta^{\star}_1}{2}$:} In this substage, we assume that $x_t\leq \frac{\theta^{\star}_1}{2}$. Hence, we can further simplify the dynamic as
          \begin{equation}
              x_{t+1}\geq x_t+0.5N\eta\frac{\theta^{\star}_1}{\norm{\theta^{\star}}}x_t^{\frac{2(N-1)}{N}}.
          \end{equation}
          Without loss of generality, we assume that $x_0=\alpha$. Now we divide the interval $[\alpha,0.5\theta^{\star}_1]$ into a series of sub-intervals $\{\cI_k\}$, where $\cI_k=[2^{k}\alpha,2^{k+1}\alpha)$. In each $\cI_k$, the dynamic can be further simplified as
          \begin{equation}
              x_{t+1}\geq \left(1+0.5N\eta\frac{\theta^{\star}_1}{\norm{\theta^{\star}}}\left(2^k\alpha\right)^{\frac{N-2}{N}}\right)x_t.
          \end{equation}
          Therefore, the number of iterations that $x_t$ spends in each interval $\cI_k$ is $\cO\left(\frac{\norm{\theta^{\star}}}{N\eta\theta^{\star}_1}\left(2^k\alpha\right)^{-\frac{N-2}{N}}\right)$. Hence, the total number of iterations is upper bounded by $\cO\left(\sum_{k=0}^{\infty}\frac{\norm{\theta^{\star}}}{N\eta\theta^{\star}_1}\left(2^k\alpha\right)^{-\frac{N-2}{N}}\right)=\cO\left(\frac{\norm{\theta^{\star}}}{N\eta\theta^{\star}_1}\alpha^{-\frac{N-2}{N}}\right)$.
    \item $x_t\geq \frac{\theta^{\star}_1}{2}$: In this substage, we define $y_t=\theta_1^{\star}-\delta\norm{\theta^{\star}}-x_t$. Via a similar trick, we can show that within additional $\cO\left(\frac{\norm{\theta^{\star}}}{N\eta}(\theta_1^{\star})^{-\frac{2N-2}{N}}\right)$ iterations, we have $x_t\geq \theta_1^{\star}-2\delta\norm{\theta^{\star}}$. Overall, after $T_1=\Theta\left(\frac{\norm{\theta^{\star}}}{N\eta\theta^{\star}_1}\alpha^{-\frac{N-2}{N}}\right)$ iterations, we have $\theta^{\star}_1-2\delta\norm{\theta^{\star}}\leq \prod w_{i,1}^{(T_1)}\leq \theta^{\star}_1$.
\end{itemize}

\paragraph{Stages 2 to $k$:} Similarly, for component $\prod w_{j,i}^{(t)}$, it takes $\cO\left(\frac{\norm{\theta^{\star}_{-(i-1)}}}{N\eta\theta_i^{\star}}\alpha^{-\frac{N-2}{N}}\right)$ iterations to attain $\theta_i^{\star}-2\delta\norm{\theta^{\star}}$. Overall, Stages 2 to $k$ take $\Theta\left(\frac{k^{\frac{3}{2}}}{N\eta}\alpha^{-\frac{N-2}{N}}\right)$ iterations to terminate.

\paragraph{Stage $k+1$:} In this stage, we take $S_t=\prod w_{j,:k}^{(t)}$. Hence, we have
\begin{equation}
    \begin{aligned}
        \norm{\theta^{\star}_{:k}-S_{t+1}} & \leq \norm{\left(\theta^{\star}_{:k}-S_{t}\right)\left(\bm{1}-\eta\frac{\sum_{i=1}^{N}\left(\prod_{j\neq i}w_{j,1}^{(t)}\right)^2}{\norm{\theta^{\star}-\prod w_j^{(t)}}}\right)}+4N\eta\delta \sqrt{k}(\theta_1^{\star})^{\frac{2(N-1)}{N}} \\
                                           & \leq \norm{\theta^{\star}_{:k}-S_{t}}\left(1-N\eta\frac{(\theta_k^{\star})^{\frac{2(N-1)}{N}}}{\norm{\theta^{\star}_{:k}-S_{t}}+\norm{E_t}}\right)+4N\eta\delta \sqrt{k}(\theta_1^{\star})^{\frac{2(N-1)}{N}}                                \\
                                           & \leq \norm{\theta^{\star}_{:k}-S_{t}}-0.5N\eta (\theta_k^{\star})^{\frac{2(N-1)}{N}}+4N\eta\delta \sqrt{k}(\theta_1^{\star})^{\frac{2(N-1)}{N}}                                                                                              \\
                                           & \leq \norm{\theta^{\star}_{:k}-S_{t}}-0.1N\eta (\theta_k^{\star})^{\frac{2(N-1)}{N}}.
    \end{aligned}
\end{equation}
Here we used the fact that $\norm{\theta^{\star}-\prod w_j^{(t)}}\leq \norm{\theta^{\star}_{:k}-S_t}+\norm{E_t}\leq 2\norm{\theta^{\star}_{:k}-S_t}$, and the assumption that $\delta\lesssim \frac{1}{N}\frac{1}{\kappa}^{\frac{2N-2}{N}}$.
On the other hand, we have
\begin{equation}
    \begin{aligned}
        \norm{S_{t+1}-S_t} & \leq \norm{\sum_{i=1}^N \eta \left(\frac{\theta^{\star}_{:k}-S_t}{\norm{\theta^{\star}-\prod w_j^{(t)}}}+\bm{\delta}_{i, :k}\right)\left(\prod_{j\neq i} w_{j,:k}^{(t)}\right)^2} \\
                           & \leq 2 N\eta\sqrt{k}(\theta^{\star}_1)^{\frac{2(N-1)}{N}}.
    \end{aligned}
\end{equation}
Hence, we conclude that within $\cO\left(\frac{\sqrt{k}\delta}{N\eta}\frac{1}{\kappa}^{\frac{2(N-1)}{N}}\right)$ iterations, we have $\norm{\theta^{\star}_{:k}-S_{t}}\lesssim \sqrt{d^2m}\alpha\vee N\eta(\theta^{\star}_1)^{\frac{2(N-1)}{N}}$. Overall, the total iteration complexity is bounded by $\cO\left(\frac{k^{\frac{3}{2}}}{N\eta}\alpha^{-\frac{N-2}{N}}\right)$.

\subsubsection*{Residual Dynamics}
Similar to the 2-layer model, here we study the surrogate of the residual term $\sum_{i=1}^{N}\left(w_{i,l}^{(t)}\right)^2$ for $l\geq k+1$. To this goal, we first notice that
\begin{equation}
    \begin{aligned}
        \sum_{i=1}^{N}\left(w_{i,l}^{(t+1)}\right)^2 & =\sum_{i=1}^{N}\left(w_{i,l}^{(t)}\right)^2+2N\eta\left(\frac{-\prod w_{j,l}^{(t)}}{\norm{\theta^{\star}-\prod w_j^{(t)}}}+\delta_{i,l}\right)\prod w_{j,l}^{(t)}    \\
                                                     & \quad+\eta^2\sum_{i=1}^{N}\left(\frac{-\prod w_{j,l}^{(t)}}{\norm{\theta^{\star}-\prod w_j^{(t)}}}+\delta_{i,l}\right)^2\left(\prod_{j\neq i} w_{j,l}^{(t)}\right)^2 \\
                                                     & \leq \sum_{i=1}^{N}\left(w_{i,l}^{(t)}\right)^2+4N\eta\delta\prod w_{j,l}^{(t)}                                                                                      \\
                                                     & \leq \sum_{i=1}^{N}\left(w_{i,l}^{(t)}\right)^2+4N\eta\delta\left(\frac{\sum_{i=1}^{N}\left(w_{i,l}^{(t)}\right)^2}{N}\right)^{\frac{N}{2}}.
    \end{aligned}
\end{equation}
Hence, once we set $z_t=\sum_{i=1}^{N}\left(w_{i,l}^{(t)}\right)^2$, we have the following simplified dynamic
\begin{equation}
    z_{t+1}\leq z_t+4N\eta\delta \left(\frac{z_t}{N}\right)^{\frac{N}{2}},
\end{equation}
with $z_0=\Theta\left(N\alpha^{\frac{2}{N}}\right)$. We claim that within $\cO\left(\frac{1}{N\eta\alpha}\right)$ iterations, we still have $z_t=\Theta\left(N\alpha^{\frac{2}{N}}\right)$. To show this, we suppose without loss of generality that $z_0=N\alpha^{\frac{2}{N}}$, and define $T$ as the first time that $z_T\geq 2N\alpha^{\frac{2}{N}}$. For any $0\leq t\leq T-1$, we have
\begin{equation}
    z_{t+1}\leq z_t + 4N\eta\delta 2^{\frac{N}{2}}\alpha.
\end{equation}
We conclude that
\begin{equation}
    T\geq \frac{N\alpha^{\frac{2}{N}}}{4N\eta\delta 2^{\frac{N}{2}}\alpha}=\frac{1}{4\delta 2^{\frac{N}{2}}}\frac{1}{\eta}\alpha^{-\frac{N-2}{N}}\gtrsim \frac{1}{N\eta}\alpha^{-\frac{N-2}{N}}.
\end{equation}
Therefore, via a basic inequality, we have
\begin{equation}
    \prod w_{i,l}^{(t)}\leq \left(\frac{\sum_{i=1}^{N}\left(w_{i,l}^{(t)}\right)^2}{N}\right)^{\frac{N}{2}}\lesssim \alpha.
\end{equation}
Combining the analysis of both signal and residual terms, we conclude that within $\Theta\left(\frac{1}{N\eta}\alpha^{-\frac{N-2}{N}}\right)$ iterations, we have
\begin{equation}
    \norm{\prod w_i^{(t)}-\theta^{\star}}\lesssim \sqrt{d^2m}\alpha\vee (\eta\theta^{\star}_1)^{\frac{2(N-1)}{N}}.
\end{equation}
\subsubsection*{Long Time Guarantee}
Similar to the proof of the 2-layer model, one can show that the residual term becomes the dominant term in the generalization error, and it stays in the order of $\alpha$ within $\Omega\left(\frac{1}{N\eta\delta}\alpha^{-\frac{N-2}{N}}\right)$ iterations. The details are omitted for brevity.
\subsubsection*{Balanced Property}
To prove the balanced property, we first study the dynamic of $w_{i,l}^{(t)}-w_{j,l}^{(t)}$, $\forall l\in [k], i,j\in [N]$. To this goal, we have
\begin{equation}
    w_{i,l}^{(t+1)}-w_{j,l}^{(t+1)}=\left(w_{i,l}^{(t)}-w_{j,l}^{(t)}\right)\left(1-\eta\left(\frac{\theta^{\star}_l-\prod w_{j,l}^{(t)}}{\norm{\theta^{\star}-\prod w_j^{(t)}}}+\delta_l\right)\prod_{f\neq i,j}w_{f,l}^{(t)}\right),
\end{equation}
which in turn implies
\begin{equation}
    \begin{aligned}
        \left|w_{i,l}^{(t+1)}-w_{j,l}^{(t+1)}\right| & \leq \left|w_{i,l}^{(t)}-w_{j,l}^{(t)}\right|\left(1-\eta\left(\frac{\theta^{\star}-\prod w_j^{(t)}}{\norm{\theta^{\star}-\prod w_j^{(t)}}}+\delta_l\right)\prod_{f\neq i,j}w_f^{(t)}\right).
    \end{aligned}
\end{equation}
If $\prod w^{(t)}_{j,l}\leq \theta^{\star}_l-\delta \norm{\theta^{\star}}$, the above inequality can be simplified as
\begin{equation}
    \left|w_{i,l}^{(t)}-w_{j,l}^{(t)}\right|\leq \left|w_{i,l}^{(0)}-w_{j,l}^{(0)}\right|\lesssim \alpha^{\frac{1}{N}}, \forall i, j\in [N], l\in [k].
\end{equation}
Once $\prod w^{(t)}_{j,l}\geq \theta^{\star}_l-\delta \norm{\theta^{\star}}$, we immediately have $w^{(t)}_{j,l}=\sqrt[N]{\theta^{\star}_l}\pm \cO(\sqrt[N]{\alpha})$. Then, we show that $w^{(t)}_{j,l}$ will stay close to $\sqrt[N]{\theta^{\star}_l}$. To this goal, we first observe that $\left(w_{i,l}^{(t+1)}-w_{i,l}^{(t)}\right)w_{i,l}^{(t)}\equiv \left(w_{j,l}^{(t+1)}-w_{j,l}^{(t)}\right)w_{j,l}^{(t)}$, $\forall i,j\in [N]$, which indicates that $w_{i,l}^{(t)}$ increases or decreases simultaneously. Hence, we conclude that $\left|w_{i,l}^{(t)}-w_{j,l}^{(t)}\right|\lesssim \delta \sqrt[N]{\theta^{\star}_l}$.

For the residual term, we can derive a tighter bound. First, we have
\begin{equation}
    \begin{aligned}
        \left(w_{i,l}^{(t+1)}\right)^2=\left(w_{i,l}^{(t)}\right)^2+2\eta\left(\frac{-\prod w_{j,l}^{(t)}}{\norm{\theta^{\star}-\prod w_j^{(t)}}}+\delta_l\right)\prod w^{(t)}_{j,l}+\eta^2\left(\frac{-\prod w_{j,l}^{(t)}}{\norm{\theta^{\star}-\prod w_j^{(t)}}}+\delta_l\right)^2\left(\prod_{j\neq i}w_{j,l}^{(t)}\right)^2.
    \end{aligned}
\end{equation}
Since we have already shown that $\prod w_{i,l}^{(t)}\lesssim\alpha$, we further have
\begin{equation}
    \left(w_{i,l}^{(t+1)}\right)^2\leq\left(w_{i,l}^{(t)}\right)^2+4\eta\delta\alpha.
\end{equation}
Therefore, one can write $\left(w_{i,l}^{(t)}\right)^2\lesssim \left(w_{i,l}^{(0)}\right)^2+4\eta\delta\alpha \frac{1}{\eta}\alpha^{-\frac{N-2}{N}}\lesssim \alpha^{\frac{2}{N}}$, which in turn implies $\left|w_{i,l}^{(t)}-w_{j,l}^{(t)}\right|\leq \left|w_{i,l}^{(t)}\right|+\left|w_{j,l}^{(t)}\right| \lesssim\alpha^{1/N}$.

\subsubsection*{Convergence in Under-parameterized Regime}
Similar to the 2-layer model, we consider the dynamic of $E_t=\prod w_{i,k+1:}^{(t)}$, which is characterized as follows
\begin{equation}
    \begin{aligned}
        \norm{E_{t+1}} & \leq \norm{E_t+\sum_{i=1}^{N}\eta\left(\frac{-E_t}{\norm{\theta^{\star}-\prod w_j^{(t)}}}+\delta_{k+1:}\right)\left(\prod_{j\neq i}w_{j,k+1:}^{(t)}\right)^2} \\
                       & \leq \norm{E_t+N\eta\left(\frac{-E_t}{\norm{\theta^{\star}-\prod w_j^{(t)}}}+\delta\right)E_t^{\frac{2(N-1)}{N}}}                                             \\
                       & \leq \norm{E_t+N\eta\left(\frac{-E_t}{\norm{E_t}}+\delta\right)E_t^{\frac{2(N-1)}{N}}}                                                                        \\
                       & \leq \norm{E_t}-N\eta d^{-\frac{N-1}{N}}\norm{E_t}^{\frac{2N-2}{N}}+N\eta\delta \norm{E_t^{\frac{2(N-1)}{N}}}                                                 \\
                       & \leq \norm{E_t}-N\eta d^{-\frac{N-1}{N}}\norm{E_t}^{\frac{2N-2}{N}}+N\eta\delta \norm{E_t}^{\frac{2N-2}{N}}                                                   \\
                       & \leq \norm{E_t}-N\eta d^{-\frac{N-1}{N}}\norm{E_t}^{\frac{2N-2}{N}}.
    \end{aligned}
\end{equation}
The last inequality comes from the fact that $\delta\lesssim d^{-\frac{N-1}{N}}$ since we assume $m\gtrsim\frac{d^{\frac{2N-2}{N}}}{(1-p)^2}$. Hence, we have
\begin{equation}
    \norm{E_t}\lesssim \left(\frac{1}{N\eta d^{-(N-1)/N} (t-\bar T)+1/\norm{E_{\bar T}}}\right)^{\frac{N}{N-2}}.
\end{equation}
Since the residual term is the dominant term in the generalization error, we have
\begin{equation}
    \norm{\prod w_i^{(t)}-\theta^{\star}}\lesssim \left(\frac{\norm{\prod w_i^{(\bar T)}-\theta^{\star}}}{\norm{\prod w_i^{(\bar T)}-\theta^{\star}}N\eta d^{-(N-1)/N} (t-\bar T)+1}\right)^{\frac{N}{N-2}},
\end{equation}
which completes the proof.

\section{Proof of Proposition~\ref{prop:sign-RIP}}
First, we provide an upper bound of the covering number for the $(k,\vartheta)$-approximate sparse unit ball. We defer a preliminary discussion on covering number to Appendix~\ref{app_prelim}
.\begin{lemma}
    Let $\cT_{k,\vartheta}:=\{u\in \R^d:u \text{ is }(k,\vartheta)\text{-approximate sparse}, \norm{u}\leq 1\}$. Then its covering number $N(\cT_{k,\vartheta},\err,\norm{\cdot})$ is upper bounded by
    \begin{equation}
        N(\cT_{k,\vartheta},\err,\norm{\cdot})\leq \left(\frac{ed}{k}\right)^k\left(1+\frac{4}{\err}\right)^k,
    \end{equation}
    provided that $\err\geq \vartheta$.
\end{lemma}

The next lemma will play a crucial role in proving Proposition~\ref{prop:sign-RIP}.
\begin{lemma}
    \label{lem::scaling}
    Suppose $x\in \R^d$ is a standard Gaussian vector, i.e., $x_{i}\overset{i.i.d.}{\sim}\cN(0,1)$, and the noise $\err$ satisfies Assumption~\ref{assumption::general-noise}, then we have
    $$\varphi(u)=\frac{\bE\left[\sign\left(\inner{x}{u}+\err\right)\inner{x}{v}\right]}{\inner{\frac{u}{\norm{u}}}{v}}=\sqrt{\frac{2}{\pi}}(1-p)+\sqrt{\frac{2}{\pi}} p \mathbb{E}\left[e^{-\err^{2} /\left(2\left\|u\right\|^2\right)}\right].$$
\end{lemma}

The proof of this lemma can be found in Appendix~\ref{sec::proof-scaling}. Now, we are ready to prove Proposition~\ref{prop:sign-RIP}. Our goal is to show that for arbitrary $u\in \cA$, the following inequality holds
\begin{equation}
    \norm{\frac{1}{m}\sum_{i=1}^{m}\sign\left(\inner{x_i}{u}+\err_i\right)x_i-\varphi(u)\frac{u}{\norm{u}}}_{\infty}\leq \delta
\end{equation}
with probability at least $1-Ce^{-cm\delta^2}$ provided that $m\gtrsim \frac{k\log(d)\log(R)\log(\frac{1}{\vartheta})}{(1-p)^2}$ . Here we define $\cA:=\left\{u:r\leq\norm{u}\leq R, u \text{ is }(k,\vartheta)\text{-approximate sparse}\right\}$, where $r\gtrsim \sqrt{{dm}/{k}}\vartheta\log\left({1}/{\vartheta}\right)$. Moreover, we define $\cB:=\left\{u:r\leq\norm{u}\leq R, \norm{u}_0\leq k\right\}$ and $\cC:=\{(u,u'):u\in \cA, v\in \cB_\zeta, \norm{u-u'}\leq \zeta\}$. Here $\cB_\zeta$ is the $\zeta$-net of $\cB$ with $\zeta\gtrsim r$. Finally, we define $\cD:=\{\pm \mathbf{e}_j\}_{j\in [d]}$, where $\mathbf{e}_j$ forms the standard basis of $\R^d$.
Based on these definitions, we have
\begin{equation}
    \begin{aligned}
         & \sup_{u\in \cA}\norm{\frac{1}{m}\sum_{i=1}^{m}\sign\left(\inner{x_i}{u}+\err_i\right)x_i-\varphi(u)\frac{u}{\norm{u}}}_{\infty}                                       \\
         & =\sup_{u\in \cA,v\in \cD}\frac{1}{m}\sum_{i=1}^{m}\sign\left(\inner{x_i}{u}+\err_i\right)\inner{x_i}{v}-\frac{\varphi(u)}{\norm{u}}\inner{u}{v}                       \\
         & =\sup_{v\in \cD}\left\{\sup_{u\in \cA}\frac{1}{m}\sum_{i=1}^{m}\sign\left(\inner{x_i}{u}+\err_i\right)\inner{x_i}{v}-\frac{\varphi(u)}{\norm{u}}\inner{u}{v}\right\}.
    \end{aligned}
\end{equation}
We then show that for each element $y\in \cD$, $\sup_{u\in\cA}\frac{1}{m}\sum_{i=1}^{m}\sign\left(\inner{x_i}{u}+\err_i\right)\inner{x_i}{y}-\varphi(u)\frac{\inner{u}{y}}{\norm{u}}$, $j\in [d]$ is $\cO\left(\frac{1}{m}\right)$-sub-Gaussian random variable. To see this, note that
\begin{equation}
    \begin{aligned}
        \norm{\sign\left(\inner{x_i}{u}+\err_i\right)x_{i,j}-\varphi(u)\frac{u_j}{\norm{u}}}_{\psi_2} & \leq \norm{\sign\left(\inner{x_i}{u}+\err_i\right)x_{i,j}}_{\psi_2}+\sqrt{\frac{2}{\pi}} \\
                                                                                                      & \leq \norm{x_{i,j}}_{\psi_2}+\sqrt{\frac{2}{\pi}}=\cO(1).
    \end{aligned}
\end{equation}
Here we use the property of sub-Gaussian norm. This implies that $\frac{1}{m}\sum_{i=1}^{m}\sign\left(\inner{x_i}{u}+\err_i\right)x_{i,j}-\varphi(u)\frac{u_j}{\norm{u}}$ is $\cO\left(\frac{1}{m}\right)$-sub-Gaussian random variable, since it is the sample average of $\sign\left(\inner{x_i}{u}+\err_i\right)x_{i,j}-\varphi(u)\frac{u_j}{\norm{u}}$.

Hence, via maximal inequality, we have that for $\forall t>0$,
\begin{equation}
    \begin{aligned}
         & \bP\left(\sup_{u\in\cA}\norm{\frac{1}{m}\sum_{i=1}^{m}\sign\left(\inner{x_i}{u}+\err_i\right)x_i-\varphi(u)\frac{u}{\norm{u}}}_{\infty}\right.                                                     \\
         & \quad\quad\left.\geq \sup_{y\in \cD}\bE\left[\sup_{u\in\cA} \frac{1}{m}\sum_{i=1}^{m}\sign\left(\inner{x_i}{u}+\err_i\right)\inner{x_i}{y}-\varphi(u)\frac{\inner{u}{y}}{\norm{u}}\right]+t\right) \\
         & \leq 2de^{-cmt^2}.
    \end{aligned}
\end{equation}
Hence, it suffices to study $\bE\left[\sup_{u\in\cA} \frac{1}{m}\sum_{i=1}^{m}\sign\left(\inner{x_i}{u}+\err_i\right)x_{i,1}-\varphi(u)\frac{u_1}{\norm{u}}\right]$.
To this goal, we decompose it into two terms via triangle inequality.
\begin{equation}
    \bE\left[\sup_{u\in\cA} \frac{1}{m}\sum_{i=1}^{m}\sign\left(\inner{x_i}{u}+\err_i\right)x_{i,1}-\varphi(u)\frac{u_1}{\norm{u}}\right]\leq {\rm (A)+ (B)},
\end{equation}
where
\begin{equation}
    {\rm (A)}:=\bE\left[\sup_{u\in\cB_\zeta} \frac{1}{m}\sum_{i=1}^{m}\sign\left(\inner{x_i}{u}+\err_i\right)x_{i,1}-\varphi(u)\frac{u_1}{\norm{u}}\right],
\end{equation}
and
\begin{equation}
    {\rm (B)}:=\bE\left[\sup_{(u,u')\in \cC}\frac{1}{m}\sum_{i=1}^{m}\left(\sign\left(\inner{x_i}{u}+\err_i\right)-\sign\left(\inner{x_i}{u'}+\err_i\right)\right)x_{i,1}-\varphi(u)\frac{u_1}{\norm{u}}+\varphi(u')\frac{u'_1}{\norm{u'}}\right].
\end{equation}

\noindent We first control ${\rm (A)}$. To this goal, we apply the union bound. Note that $\frac{1}{m}\sum_{i=1}^{m}\sign\left(\inner{x_i}{u}+\err_i\right)x_{i,1}-\varphi(u)\frac{u_1}{\norm{u}}$ is $\cO(\frac{1}{m})$-sub-Gaussian and $\left|\cB_\zeta\right|\leq \left(\frac{R}{\zeta}\right)^{Ck\log(d)}$. We then have
\begin{equation}
    {\rm (A)}\lesssim \sqrt{\frac{k\log(d)\log\left(\frac{R}{\zeta}\right)}{m}}.
\end{equation}

Now we control ${\rm (B)}$.
Via triangle inequality, we first obtain
\begin{equation}
    \begin{aligned}
        {\rm (B)} & \leq \underbrace{\bE\left[\sup_{(u,u')\in \cC}\frac{1}{m}\sum_{i=1}^{m}\left(\sign\left(\inner{x_i}{u}+\err_i\right)-\sign\left(\inner{x_i}{u'}+\err_i\right)\right)x_{i,1}\right]}_{\rm (B_1)} \\
                  & \quad+\underbrace{\sup_{(u,u')\in \cC}\left\{-\varphi(u)\frac{u_1}{\norm{u}}+\varphi(u')\frac{u'_1}{\norm{u'}}\right\}}_{\rm (B_2)}.
    \end{aligned}
\end{equation}
For the first part, applying Hölder's inequality leads to
\begin{equation}
    \begin{aligned}
        {\rm (B_1)} & \leq \bE\left[\sup_{(u,u')\in \cC}\left(\frac{1}{m}\sum_{i=1}^m|\sign\left(\inner{x_i}{u}+\err_i\right)-\sign\left(\inner{x_i}{u'}+\err_i\right)|\right)\max_{1\leq i\leq m}\left|x_{i,1}\right|\right]         \\
                    & \leq \bE\left[\sup_{(u,u')\in \cC}\left(\frac{1}{m}\sum_{i=1}^m\mathbbm{1}\left(|\inner{x_i}{u-u'}|\geq \left|\inner{x_i}{u}+\err_i\right|\right)\right)\max_{1\leq i\leq m}\left|x_{i,1}\right|\right]         \\
                    & \leq \underbrace{\bE\left[\sup_{\norm{\Delta u}\leq \zeta}\left(\frac{1}{m}\sum_{i=1}^m\mathbbm{1}\left(|\inner{x_i}{\Delta u}|\geq t\right)\right)\max_{1\leq i\leq m}\left|x_{i,1}\right|\right]}_{\rm (B_3)} \\
                    & +\underbrace{\bE\left[\sup_{u\in \cB_\zeta}\left(\frac{1}{m}\sum_{i=1}^m\mathbbm{1}\left(|\inner{x_i}{u}+\err_i|\leq t\right)\right)\max_{1\leq i\leq m}\left|x_{i,1}\right|\right]}_{\rm (B_4)},
    \end{aligned}
\end{equation}
\begin{sloppypar}
    \noindent where $t>0$ is a constant to be determined later. Here, we used the fact that $\mathbbm{1}\left(|\inner{x_i}{u-u'}|\geq \left|\inner{x_i}{u}+\err_i\right|\right)\leq \mathbbm{1}\left(|\inner{x_i}{\Delta u}|\geq t\right)+\mathbbm{1}\left(|\inner{x_i}{u}+\err_i|\leq t\right)$ in the last inequality. We first bound $\rm (B_3)$
\end{sloppypar}
\begin{equation}
    \begin{aligned}
        {\rm (B_3)} & \leq \bE\left[\left(\frac{1}{m}\sum_{i=1}^{m}\mathbbm{1}\left(\zeta\norm{x_i}\geq t\right)\right)\max_{1\leq i\leq m}\left|x_{i,1}\right|\right]                                                       \\
                    & \leq \bE\left[\mathbbm{1}\left(\zeta\norm{x_i}\geq t\right)\right]\bE\left[\max_{j\neq i}\left|x_{j,1}\right|\right]+\bE\left[\mathbbm{1}\left(\zeta\norm{x_i}\geq t\right)\left|x_{i,1}\right|\right] \\
                    & \lesssim e^{-C\frac{t^2}{\zeta^2}}\sqrt{\log\left(m\right)}+\bE\left[\mathbbm{1}\left(\zeta\norm{x_i}\geq t\right)\left|x_{i,1}\right|\right],
    \end{aligned}
\end{equation}
provided that $\frac{t}{\zeta}\gtrsim \sqrt{d}$. Applying Cauchy-Schwarz inequality, we have
\begin{equation}
    \bE\left[\mathbbm{1}\left(\zeta\norm{x_i}\geq t\right)\left|x_{i,1}\right|\right]\leq\sqrt{\bP\left(\zeta\norm{x_i}\geq t\right)}\sqrt{\bE\left[x_{i,1}^2\right]} \leq e^{-C\frac{t^2}{\zeta^2}}.
\end{equation}
Hence, we conclude that ${\rm (B_3)}\lesssim e^{-C\frac{t^2}{\zeta^2}}\sqrt{\log\left(m\right)}$.
Next we control $\rm (B_4)$. Note that $\max_{i}\left|x_{i,1}\right|$ is $\cO\left(\log(m)\right)$-sub-Gaussian. Via union bound, we have
\begin{equation}
    \begin{aligned}
        {\rm (B_4)} & \leq \sup_{u\in \cB}\bE\left[\mathbbm{1}\left(|\inner{x_i}{u}+\err_i|\leq t\right)\max_{1\leq i\leq m}\left|x_{i,1}\right|\right]+C\sqrt{\frac{k\log(m)\log(d)\log(\frac{R}{\zeta})}{m}}.
    \end{aligned}
\end{equation}
For the first part, applying the similar decomposition method, we have
\begin{equation}
    \begin{aligned}
        \bE\left[\mathbbm{1}\left(|\inner{x_i}{u}+\err_i|\leq t\right)\max_{1\leq i\leq m}\left|x_{i,1}\right|\right] & \leq \bE\left[\mathbbm{1}\left(|\inner{x_i}{u}+\err_i|\leq t\right)\max_{j\neq i}\left|x_{i,1}\right|\right] \\
                                                                                                                      & \quad+ \bE\left[\mathbbm{1}\left(|\inner{x_i}{u}+\err_i|\leq t\right)\left|x_{i,1}\right|\right]             \\
                                                                                                                      & \lesssim \sqrt{\log(m)}\frac{t}{r}.
    \end{aligned}
\end{equation}
Hence, we conclude that $ {\rm (B_4)}\lesssim \sqrt{\log(m)}\frac{t}{r}+\sqrt{\frac{k\log(m)\log(d)\log(\frac{R}{\zeta})}{m}}$.
For ${\rm (B_2)}$, we first have
\begin{equation}
    \begin{aligned}
        \left|-\varphi(u)\frac{u_1}{\norm{u}}+\varphi(u')\frac{u'_1}{\norm{u'}}\right| & =\left|\varphi(u')-\varphi(u)\right|\frac{|u_1|}{\norm{u}}+\varphi(u')\left|\frac{u'_1}{\norm{u'}}-\frac{u_1}{\norm{u}}\right| \\
                                                                                       & \lesssim \left|\varphi(u')-\varphi(u)\right|+\zeta.
    \end{aligned}
\end{equation}
For the first part, we use Mean Value Theorem to write
\begin{equation}
    |\varphi(u')-\varphi(u)|\leq \norm{\nabla \varphi(v)}\norm{u'-u}\leq \norm{\nabla \varphi(v)}\zeta,
\end{equation}
where $v$ is a point between $u$ and $u'$.
Note that $\nabla \varphi(v)=\sqrt{\frac{2}{\pi}}p\bE\left[\frac{\err^2v}{\norm{v}^4}e^{-\frac{\err^2}{2\norm{v}^2}}\right]$. Hence, we have
\begin{equation}
    \begin{aligned}
        \sup_{\norm{v}\geq r}\norm{\nabla \varphi(v)} & \lesssim \sup_{\norm{v}\geq r}\bE\left[\frac{\err^2}{\norm{v}^3}e^{-\frac{\err^2}{2\norm{v}^2}}\right]\leq \frac{1}{r}\sup_{\norm{v}\geq r}\bE\left[\frac{\err^2}{\norm{v}^2}e^{-\frac{\err^2}{2\norm{v}^2}}\right]\lesssim \frac{1}{r}.
    \end{aligned}
\end{equation}
Overall, we have ${\rm (B_2)}\lesssim \frac{\zeta}{r}$, which results in
\begin{equation}
    \begin{aligned}
         & \bE\left[\sup_{u\in\cA} \frac{1}{m}\sum_{i=1}^{m}\sign\left(\inner{x_i}{u}+\err_i\right)x_{i,1}-\varphi(u)\frac{u_1}{\norm{u}}\right]                        \\
         & \lesssim \frac{\zeta}{r}+e^{-C\frac{t^2}{\zeta^2}}\sqrt{\log\left(m\right)}+\sqrt{\log(m)}\frac{t}{r}+\sqrt{\frac{k\log(m)\log(d)\log(\frac{R}{\zeta})}{m}}.
    \end{aligned}
\end{equation}
Hence, once we set $\zeta\asymp \vartheta$, and $t\asymp \sqrt{d}\vartheta\log(m)$, together with the assumption that $r\gtrsim\sqrt{\frac{dm}{k}}\vartheta\log\left(\frac{1}{\vartheta}\right)$, we conclude that
\begin{equation}
    \bE\left[\sup_{u\in\cA} \frac{1}{m}\sum_{i=1}^{m}\sign\left(\inner{x_i}{u}+\err_i\right)x_{i,1}-\varphi(u)\frac{u_1}{\norm{u}}\right]\lesssim \sqrt{\frac{k\log^2(m)\log(d)\log(\frac{R}{\vartheta})}{m}}.
\end{equation}
This leads to
\begin{equation}
    \begin{aligned}
         & \bP\left(\sup_{u\in\cA}\norm{\frac{1}{m}\sum_{i=1}^{m}\sign\left(\inner{x_i}{u}+\err_i\right)x_i-\varphi(u)\frac{u}{\norm{u}}}_{\infty}\geq C\sqrt{\frac{k\log^2(m)\log(d)\log(\frac{R}{\vartheta})}{m}}+\delta\right) \\
         & \leq 2de^{-cm\delta^2}.
    \end{aligned}
\end{equation}
Therefore, the following inequality holds, provided that $m\gtrsim \frac{k\log^2(m)\log(d)\log(\frac{R}{\vartheta})}{(1-p)^2\delta^2}$
\begin{equation}
    \bP\left(\sup_{u\in\cA}\norm{\frac{1}{\varphi(u)}\frac{1}{m}\sum_{i=1}^{m}\sign\left(\inner{x_i}{u}+\err_i\right)x_i-\frac{u}{\norm{u}}}_{\infty}\geq \delta\right)\leq e^{-cm\delta^2}.
\end{equation}

Now we turn to the case $m\gtrsim \frac{d}{(1-p)^2}$. Following the same technique, it suffices to bound $\bE\left[\sup_{u\in \R^d} \frac{1}{m}\sum_{i=1}^{m}\sign\left(\inner{x_i}{u}+\err_i\right)x_{i,1}-\varphi(u)\frac{u_1}{\norm{u}}\right]$. To this goal, we first notice that
\begin{equation}
    \begin{aligned}
         & \bE\left[\sup_{u\in \R^d} \frac{1}{m}\sum_{i=1}^{m}\sign\left(\inner{x_i}{u}+\err_i\right)x_{i,1}-\varphi(u)\frac{u_1}{\norm{u}}\right]                           \\
         & =\underbrace{\bE\left[\sup_{\norm{u}=1,\lambda\in \R} \frac{1}{m}\sum_{i=1}^{m}\sign\left(\inner{x_i}{u}+\lambda\err_i\right)x_{i,1}-\varphi(u)u_1\right]}_{(A)}.
    \end{aligned}
\end{equation}
Similarly, applying one-step discretization, we have
\begin{equation}
    \begin{aligned}
        (A) & \leq \underbrace{\bE\left[\sup_{u\in \bS_{\err},\lambda\in \R} \frac{1}{m}\sum_{i=1}^{m}\sign\left(\inner{x_i}{u}+\lambda\err_i\right)x_{i,1}-\phi(\lambda)u_1\right]}_{(B)}                                                                                  \\
            & +\underbrace{\bE\left[\sup_{\norm{u-u'}\leq \err,\lambda\in \R} \frac{1}{m}\sum_{i=1}^{m}\left(\sign\left(\inner{x_i}{u}+\lambda\err_i\right)-\sign\left(\inner{x_i}{u'}+\lambda\err_i\right)\right)x_{i,1}+\phi(\lambda)\left(u'_1-u_1\right)\right]}_{(C)}.
    \end{aligned}
\end{equation}
Here $\phi(\lambda)=\sqrt{\frac{2}{\pi}}(1-p)+\sqrt{\frac{2}{\pi}} p \mathbb{E}\left[e^{-\lambda^2\err^{2}/2}\right]$ is the same as before.
We first control $(B)$. To this goal, we show that $\sup_{\lambda\in \R} \frac{1}{m}\sum_{i=1}^{m}\sign\left(\inner{x_i}{u}+\lambda\err_i\right)x_{i,1}-\phi(\lambda)u_1$ is $\cO(1/m)$-sub-Gaussian. We prove it via checking the sub-Gaussian norm
\begin{equation}
    \norm{\sup_{\lambda\in \R}\sign\left(\inner{x_i}{u}+\lambda\err_i\right)x_{i,1}-\phi(\lambda)u_1}_{\psi_2}\leq \norm{|x_{i,1}|}_{\psi_2}+\sqrt[]{\frac{2}{\pi}}=\cO(1).
\end{equation}
Hence, via maximum inequality, we have
\begin{equation}
    (B)\leq \underbrace{\bE\left[\sup_{\lambda\in \R} \frac{1}{m}\sum_{i=1}^{m}\sign\left(\inner{x_i}{u}+\lambda\err_i\right)x_{i,1}-\phi(\lambda)u_1\right]}_{(D)}+\cO\left(\sqrt[]{\frac{d\log\left(\frac{1}{\err}\right)}{m}}\right).
\end{equation}
To control $(D)$, we further decompose it into two parts,
\begin{equation}
    \begin{aligned}
        (D) & \leq \underbrace{\bE\left[\sup_{\nu\in [0,1]} \frac{1}{m}\sum_{i=1}^{m}\sign\left(\nu\inner{x_i}{u}+\err_i\right)x_{i,1}-\phi\left(\frac{1}{\nu}\right)u_1\right]}_{(D_1)} \\
            & \quad+\underbrace{\bE\left[\sup_{\lambda\in [0,1]} \frac{1}{m}\sum_{i=1}^{m}\sign\left(\inner{x_i}{u}+\lambda\err_i\right)x_{i,1}-\phi(\lambda)u_1\right]}_{(D_2)}.
    \end{aligned}
\end{equation}
To control $(D_1)$ and $(D_2)$ we use arguments based on bracketing maximal inequality. We defer a preliminary discussion on bracketing maximal inequality to Appendix~\ref{app_prelim}.
We first control $(D_1)$. Let $\bT_\xi$ be defined as the $\xi$-net of the interval $[0,1]$. We show that for any $\nu,\nu'\in [0,1]$ such that $|\nu-\nu'|\leq \xi$, we can control $\norm{\left(\sign(\nu\inner{x_i}{u}+\err_i)-\sign(\nu'\inner{x_i}{u}+\err_i)\right)x_{i,1}}_{L_2(\bP)}$. To this goal, we first have
\begin{equation}
    \begin{aligned}
         & \bE\left[\left(\sign(\nu\inner{x_i}{u}+\err_i)-\sign(\nu'\inner{x_i}{u}+\err_i)\right)^2x_{i,1}^2\right]                                          \\
         & \lesssim \bE\left[|\sign(\nu\inner{x_i}{u}+\err_i)-\sign(\nu'\inner{x_i}{u}+\err_i)|\right]                                                       \\
         & \leq \bE\left[\mathbbm{1}\left(|\left(\nu-\nu'\right)\inner{x_i}{u}|\geq t\right)+\mathbbm{1}\left(|\nu\inner{x_i}{u}+\err_i|\leq t\right)\right] \\
         & \lesssim e^{-C\frac{t^2}{\xi^2}}+t.
    \end{aligned}
\end{equation}
Upon picking $t\asymp \xi\log\left(\frac{1}{\xi}\right)$, we have
\begin{equation}
    \norm{\left(\sign(\nu\inner{x_i}{u}+\err_i)-\sign(\nu'\inner{x_i}{u}+\err_i)\right)x_{i,1}}_{L_2(\bP)}\lesssim \sqrt[]{\xi\log\left(\frac{1}{\xi}\right)}.
\end{equation}
Therefore, the bracketing number is bounded by $N_{[]}(\varepsilon\|F\|, \mathcal{F},\|\cdot\|)\lesssim C\frac{1}{\sqrt[]{\err}}$, which in turn leads to an upper bound on the bracketing entropy $J_{[]}(1, \mathcal{F},L_2(\bP))\lesssim 1$. Applying Theorem~\ref{thm::empirical-process} leads to
\begin{equation}
    (D_1)\lesssim \sqrt[]{\frac{1}{m}}.
\end{equation}
Similarly, we can show that $(D_2)\lesssim \sqrt[]{\frac{1}{m}}$. Therefore, we conclude that
\begin{equation}
    (B)\lesssim \sqrt[]{\frac{d\log\left(\frac{1}{\err}\right)}{m}}.
\end{equation}
For $(C)$, we can use the similar technique in the overparameterized setting ($m\ll d$), which leads to
\begin{equation}
    (C)\lesssim \sqrt[]{\log(m)}\err+\sqrt[]{\frac{d\log(m)}{m}}.
\end{equation}
Therefore, once we set $\err\asymp \sqrt[]{\frac{d}{m}}$, we immediately obtain
\begin{equation}
    (A)\lesssim \sqrt[]{\frac{d\log(m)}{m}}.
\end{equation}
Combining the derived bounds results in
\begin{equation}
    \bP\left(\sup_{u\in\R^d}\norm{\frac{1}{m}\sum_{i=1}^{m}\sign\left(\inner{x_i}{u}+\err_i\right)x_i-\varphi(u)\frac{u}{\norm{u}}}_{\infty}\geq C\sqrt{\frac{d\log(m)}{m}}+\delta\right)\leq e^{c_1\log(d)-c_2m\delta^2}.
\end{equation}
Assuming $m\gtrsim \frac{d\log(m)}{(1-p)^2}$, the above bound reduces to
\begin{equation}
    \bP\left(\sup_{u\in\R^d}\norm{\frac{1}{\varphi(u)}\frac{1}{m}\sum_{i=1}^{m}\sign\left(\inner{x_i}{u}+\err_i\right)x_i-\frac{u}{\norm{u}}}_{\infty}\geq \delta\right)\leq e^{-cm\delta^2}.
\end{equation}

\section{Auxiliary Lemmas}
\begin{lemma}
    \label{lem::uniform-concentration-of-absolute-value}
    Suppose $x_1,\cdots,x_m$ are i.i.d. standard Gaussian vectors with dimension $d$. Then, for arbitrary $\delta>0$ we have
    \begin{equation}
        \bP\left(\sup_{\norm{u}=1}\left|\frac{1}{m}\sum_{i=1}^{m}|\inner{x_i}{u}|-\sqrt{\frac{2}{\pi}}\right|\geq C\sqrt{\frac{d}{m}}+\delta\right)\leq e^{-cm\delta^2}.
    \end{equation}
    Here $C,c$ are universal constants.
\end{lemma}
\begin{proof}
    This lemma directly follows from the standard expectation and high probability bounds for sub-Gaussian process. See e.g.,~\cite[Lemma 4]{ma2022global} for a simple proof.
\end{proof}
\begin{lemma}
    \label{lem::hadamard-product}
    For two arbitrary vectors $a,b\in \bR^n$, we have
    \begin{equation}
        \norm{a\odot b}\leq \norm{a}_{\infty}\norm{b}.
    \end{equation}
\end{lemma}

\section{Deferred Proofs}
\subsection{Proof of Lemma~\ref{lem::scaling}}
\label{sec::proof-scaling}
\begin{proof}
    To prove this lemma, it suffices to show that, for any $u,v\in\R^{d}$, we have
    \begin{equation}
        \E\left[\sign\left(\err+\inner{x}{u}\right)\inner{x}{v}\right]=\sqrt{\frac{2}{\pi}}\E\left[e^{-\err^2/2\norm{u}^2}\right]\inner{\frac{u}{\norm{u}}}{v}.
    \end{equation}
    Without loss of generality, we assume that $\norm{u}=\norm{v}=1$. Let us denote $w:=\inner{x}{u},z:=\inner{x}{v},\rho:=\Cov(w,z)=\inner{u}{v}$. Then
    \begin{equation}
        \begin{aligned}
            \E\left[\sign\left(\err+\inner{x}{u}\right)\inner{x}{v}\right] & =\E\left[\sign\left(\err+w\right)z\right]                                                                                                      \\
                                                                           & \stackrel{(a)}{=}\rho \E\left[\sign(w+\err)w\right]                                                                                            \\
                                                                           & =\rho\E_{\err}\left[ \int_{-\err}^{\infty}t \frac{1}{\sqrt{2\pi}}e^{-t^2/2}dt-\int_{-\infty}^{-\err}t \frac{1}{\sqrt{2\pi}}e^{-t^2/2}dt\right] \\
                                                                           & =\rho\E_{\err}\left[ \int_{-\err}^{\infty}t \frac{1}{\sqrt{2\pi}}e^{-t^2/2}dt+\int_{\err}^{\infty}t \frac{1}{\sqrt{2\pi}}e^{-t^2/2}dt\right]   \\
                                                                           & =2\rho \E_{\err}\left[\int_{|\err|}^{\infty} t \frac{1}{\sqrt{2\pi}}e^{-t^2/2}dt\right]                                                        \\
                                                                           & =\sqrt{\frac{2}{\pi}}\inner{u}{v}\E_\err\left[\int_{|\err|}^{\infty}d\left(-e^{-t^2/2}\right)\right]                                           \\
                                                                           & =\sqrt{\frac{2}{\pi}}\inner{u}{v}\E_\err\left[e^{-\err^2/2}\right].
        \end{aligned}
    \end{equation}
    Here in (a) we use the fact that $z|w, \err\sim \mathcal{N}(\rho w, 1-\rho^2)$ since $\err$ is independent of $w, z$. Hence, we have
    \begin{equation}
        \E\left[\sign\left(\err+\inner{x}{u}\right)\inner{x}{v}\right]=\sqrt{\frac{2}{\pi}}\E\left[e^{-\err^2/2\norm{u}^2}\right]\inner{\frac{u}{\norm{u}}}{v}
    \end{equation}
    for any $u,v\in\R^{d}$. On the other hand, it is easy to verify that $\E\left[\sign\left(\inner{x}{u}\right)\inner{x}{v}\right]=\sqrt{\frac{2}{\pi}}\inner{\frac{u}{\norm{u}}}{v}$. The proof is completed by noting that the corruption probability is $p$.
\end{proof}

\section{Preliminaries on the Uniform Concentration Bounds}\label{app_prelim}
In this section, we provide the preliminary probability tools for proving Proposition~\ref{prop:sign-RIP}.

\begin{definition}[Sub-Gaussian random variable]
    \label{def-sub-gaussian}
    We say a random variable $X\in \R$ with expectation $\E[X]=\mu$ is $\sigma^2$-sub-Gaussian if for all $\lambda\in \R$, we have $\E\left[e^{\lambda (X-\mu)}\right]\leq e^{\frac{\lambda^2\sigma^2}{2}}$.
    Moreover, the sub-Gaussian norm of $X$ is defined as $\norm{X}_{\psi_2}:=\sup_{p\geq 1} \left\{p^{-1/2} (\bE[|X|^p])^{1/p}\right\}$.
\end{definition}
According to~\cite{wainwright2019high}, the following statements are equivalent:
\begin{itemize}
    \item $X$ is $\sigma^2$-sub-Gaussian.
    \item (Tail bound) For any $t>0$, we have $\mathbb{P}(|X-\mu|\geq t)\leq 2e^{-\frac{t^2}{2\sigma^2}}$.
    \item (Moment bound) We have $\norm{X}_{\psi_2}\lesssim \sigma$.
\end{itemize}
Next, we provide the definitions of the sub-Gaussian process, $\err$-net, and covering number.
\begin{definition}[Sub-Gaussian process]
    \label{def::sub-Gaussian-process}
    A zero mean stochastic process $\{\cX_{\theta}, \theta \in \bT\}$ is a $\sigma^2$-sub-Gaussian process with respect to a metric $d$ on a set $\bT$, if for every $\theta,\theta'\in \bT$, the random variable $\cX_{\theta}-\cX_{\theta'}$ is $\left(\sigma d(\theta,\theta')\right)^2$-sub-Gaussian.
\end{definition}
\begin{definition}[$\err$-net and covering number]
    A set $\cN$ is called an $\err$-net for $(\bT,d)$ if for every $t\in \bT$, there exists $\pi(t) \in \cN$ such that $d(t,\pi(t)) \leq \err$. The covering number $N(\bT, d, \err)$ is defined as the smallest cardinality of an $\err$-net for $(\bT,d)$:
    \begin{equation}\nonumber
        N(\bT, d, \err) := \inf\{|\cN| : \cN \text{ is an } \err\text{-net for } (\bT, d)\}.
    \end{equation}
\end{definition}

\begin{definition}[Bracketing number, Definition 2.1.6 in \cite{van1996weak}]
    Given two functions $l$ and $u$, the bracket $[l, u]$ is the set of all functions $f$ with $l \leq f \leq u$. An $\varepsilon$-bracket is a bracket $[l, u]$ with $\|u-l\|<\varepsilon$. The bracketing number $N_{[]}(\varepsilon, \mathcal{F},\|\cdot\|)$ is the minimum number of $\varepsilon$-brackets needed to cover $\mathcal{F}$. The bracketing entropy is the logarithm of the bracketing number. In the definition of the bracketing number, the upper and lower bounds $u$ and $l$ of the brackets need not belong to $\mathcal{F}$ themselves but are assumed to have finite norms.
\end{definition}
Bracketing number can be regarded as an analog of covering number, describing the geometric complexity of the underlining function space. Although bracketing number of a general function class is difficult to characterize, for some specific function classes, we can easily derive upper bounds for their bracketing number. In particular, we have the following result for Lipschitz functions.
\begin{theorem}[Adapted from Theorem 2.7.11 in \cite{van1996weak}]
    Let $\cF = \{f_t: t \in T\}$ be a class of functions. Suppose that for arbitrary $s,t\in T$, we have
    \begin{equation}
        \left|f_{s}(x)-f_{t}(x)\right| \leq d(s, t) F(x),
    \end{equation}
    for some metric $d$ on the index set, function $F$ on the sample space, and every $x$. Then, for any norm $\norm{\cdot}$,
    \begin{equation}
        N_{[]}(2 \varepsilon\|F\|, \mathcal{F},\|\cdot\|) \leq N(\varepsilon, T, d) .
    \end{equation}
\end{theorem}
\begin{theorem}[Adapted from Theorem 2.14.2 in \cite{van1996weak}]
    \label{thm::empirical-process}
    For a given norm $\norm{\cdot}$, define a bracketing integral of a class of functions $\cF$ as
    \begin{equation}
        J_{[]}(\delta, \mathcal{F},\|\cdot\|)=\int_{0}^{\delta} \sqrt{1+\log N_{[]}(\varepsilon\|F\|, \mathcal{F},\|\cdot\|)} d \varepsilon.
    \end{equation}
    Let $\cF$ be a class of measurable functions with measurable envelope function $F$, we have
    \begin{equation}
        \bE\left[\sup_{f\in \cF}\frac{1}{n}\sum_{i=1}^{n}f(X_i)-\bE\left[f(X)\right]\right]\lesssim J_{[]}(1, \mathcal{F},L_2(\bP))\frac{\norm{F}_{L_2(\bP)}}{\sqrt{n}},
    \end{equation}
    where $\bP$ is the distribution of $X$, and the $L_2(\bP)$-norm is defined as $\norm{f}_{L_2(\bP)}:=\left(\int_{\Omega} f^2(\omega)d\bP(\omega)\right)^{1/2}$.
\end{theorem}

\end{document}